\documentclass{article}
\pdfoutput=1
\PassOptionsToPackage{numbers, compress}{natbib}
\usepackage[final]{neurips_2022}

\usepackage[utf8]{inputenc} %
\usepackage[T1]{fontenc}    %
\usepackage{hyperref}       %
\usepackage{url}            %
\usepackage{booktabs}       %
\usepackage{amsfonts}       %
\usepackage{nicefrac}       %
\usepackage{microtype}      %
\usepackage{xcolor}         %
\usepackage{wrapfig}

\usepackage{nccmath}
\usepackage{amsmath}
\usepackage{amssymb}
\usepackage{amsthm}
\usepackage{bm}
\usepackage{tikz}
\usetikzlibrary{bayesnet,arrows,fit,positioning,shapes,calc,decorations.pathmorphing,matrix,shapes.geometric,automata,decorations.text} %
\usepackage{stackengine}
\usepackage{subcaption}
\usepackage{paralist}
\usepackage{tabularx}
\usepackage{scalerel,graphicx,xparse,textcomp}
\usepackage{mathtools}
\usepackage{etoolbox}
\usepackage[capitalize,noabbrev]{cleveref}
\usepackage{enumitem}
\usepackage{thmtools, thm-restate}
\usepackage{multirow}
\usepackage{changepage} %
\usepackage{adjustbox} %
\usepackage{makecell}
\usepackage{lipsum}
\usepackage{cancel}
\usepackage{algorithm}
\usepackage{algorithmic}
\usepackage{multicol}
\usepackage{pifont}
\usepackage{stmaryrd}
\usepackage{trimclip}
\usepackage{tikzducks}
\usepackage[normalem]{ulem}
\usepackage{titlesec}
\newcommand\stackrqarrow[1]{%
    \mathrel{\stackon[0pt]{$\rightsquigarrow$}{$\scriptscriptstyle#1$}}}
\newcommand\stackrarrow[1]{%
\mathrel{\stackon[0pt]{$\rightarrow$}{$\scriptscriptstyle#1$}}}

\newcommand\indep{\protect\mathpalette{\protect\independenT}{\perp}}
\def\independenT#1#2{\mathrel{\rlap{$#1#2$}\mkern2mu{#1#2}}}

\makeatletter
\newcommand*\bigcdot{\mathpalette\bigcdot@{.5}}
\newcommand*\bigcdot@[2]{\mathbin{\vcenter{\hbox{\scalebox{#2}{$\m@th#1\bullet$}}}}}
\makeatother



\newcommand{\veryshortarrow}[1][3pt]{\mathrel{%
   \hbox{\rule[\dimexpr\fontdimen22\textfont2-.2pt\relax]{#1}{.4pt}}%
   \mkern-4mu\hbox{\usefont{U}{lasy}{m}{n}\symbol{41}}}}

\makeatletter

\setbox0\hbox{$\xdef\scriptratio{\strip@pt\dimexpr
    \numexpr(\sf@size*65536)/\f@size sp}$}

\newcommand{\scriptveryshortarrow}[1][3pt]{{%
    \hbox{\rule[\scriptratio\dimexpr\fontdimen22\textfont2-.2pt\relax]
               {\scriptratio\dimexpr#1\relax}{\scriptratio\dimexpr.4pt\relax}}%
   \mkern-4mu\hbox{\let\f@size\sf@size\usefont{U}{lasy}{m}{n}\symbol{41}}}}

\makeatother

\DeclareRobustCommand
  \cdotsshort{\mathinner{\cdotp\mkern-2mu\cdotp\mkern-2mu\cdotp}}

\newdimen\arrowsize 
\pgfarrowsdeclare{arcsq}{arcsq}
{
  \arrowsize=0.2pt
  \advance\arrowsize by .5\pgflinewidth
  \pgfarrowsleftextend{-4\arrowsize-.5\pgflinewidth}
  \pgfarrowsrightextend{.5\pgflinewidth}
}
{
  \arrowsize=1.5pt
  \advance\arrowsize by .5\pgflinewidth
  \pgfsetdash{}{0pt} %
  \pgfsetroundjoin   %
  \pgfsetroundcap    %
  \pgfpathmoveto{\pgfpoint{0\arrowsize}{0\arrowsize}}
  \pgfpatharc{-90}{-140}{4\arrowsize}
  \pgfusepathqstroke
  \pgfpathmoveto{\pgfpointorigin}
  \pgfpatharc{90}{140}{4\arrowsize}
  \pgfusepathqstroke
}

\newtheorem{theorem}{Theorem}

\newtheorem{lemma}{Lemma}

\newtheorem{remark}{Remark}

\theoremstyle{definition}
\newtheorem{definition}{Definition}
\newtheorem{assumption}{Assumption}
\newtheorem{example}{Example}

\titlespacing{\subsection}{0pt}{-.1\parskip}{-.1\parskip}
\titlespacing{\subsubsection}{0pt}{-.1\parskip}{-.1\parskip}

\crefformat{section}{\S#2#1#3} %
\crefformat{subsection}{\S#2#1#3}
\crefformat{subsubsection}{\S#2#1#3}

\DeclarePairedDelimiterX\braket[2]{\langle}{\rangle}{#1 \delimsize\vert #2}

\let\emptyset\varnothing

\newcommand{\A}{\mathbf{A}}
\newcommand{\B}{\mathbf{B}}
\newcommand{\I}{\mathbf{I}}
\newcommand{\W}{\mathbf{W}}
\newcommand{\X}{\mathbf{X}}
\newcommand{\Y}{\mathbf{Y}}
\newcommand{\Z}{\mathbf{Z}}
\newcommand{\E}{\mathbf{E}}
\newcommand{\SH}{\mathbf{S}}

\newcommand{\Real}{\mathbb{R}}
\newcommand{\Scal}{\mathcal{S}}

\newcommand{\Xtilde}{\smash{\Tilde{\X}}}
\newcommand{\Ztilde}{\smash{\Tilde{\Z}}}
\newcommand{\Ytilde}{\smash{\Tilde{\Y}}}
\newcommand{\Etilde}{\smash{\Tilde{\E}}}

\newcommand{\Xt}{\smash{\Tilde{X}}}
\newcommand{\Xitilde}{\smash{\Tilde{X}_i}}

\newcommand{\Gtilde}{\smash{\Tilde{G}}}

\newcommand{\darmois}{\text{Darmois–Skitovich theorem}}
\newcommand{\T}{\intercal}
\newcommand{\w}{\omega}
\newcommand{\rank}{\operatorname{rank}}
\newcommand{\Anc}{\operatorname{Anc}}
\newcommand{\Ancout}[1]{\operatorname{Anc}_{\operatorname{out}(#1)}}
\newcommand{\cov}{\operatorname{cov}}
\newcommand{\var}{\operatorname{var}}

\newcommand{\nullspace}{\operatorname{null}}
\newcommand{\GIN}{\operatorname{GIN}}
\newcommand{\IN}{\operatorname{IN}}
\newcommand{\EGIN}{\operatorname{TIN}}
\newcommand{\TIN}{\operatorname{TIN}}
\newcommand{\LC}{\operatorname{LC}}
\newcommand{\OmegaZY}{\Omega_{\Z;\Y}}

\newcommand{\wTY}{\w^\T \Y}

\newcommand{\true}{\textsc{true}}
\newcommand{\false}{\textsc{false}}

\newcommand{\xmark}{\ding{55}}%

\newcommand{\red}[1]{\textcolor{red}{#1}}
\newcommand{\newcontent}[1]{{\color{black}#1}} %

\def\factory{\scalerel*{\includegraphics{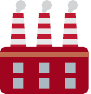}}{\textrm{\textbigcircle}}}
\def\air{\scalerel*{\includegraphics{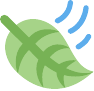}}{\textrm{\textbigcircle}}}
\def\lungs{\scalerel*{\includegraphics{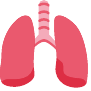}}{\textrm{\textbigcircle}}}
\def\report{\scalerel*{\includegraphics{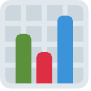}}{\textrm{\textbigcircle}}}
\def\hospital{\scalerel*{\includegraphics{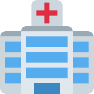}}{\textrm{\textbigcircle}}}

\vspace{-0.5em}
\title{Independence Testing-Based Approach to Causal Discovery under Measurement Error and\\Linear Non-Gaussian Models}
\author{Haoyue Dai$^{1,2}$
\quad
Peter Spirtes$^1$
\quad
Kun Zhang$^{1,2}$\\
$^1$Department of Philosophy, Carnegie Mellon University\\
$^2$Machine Learning Department, Mohamed bin Zayed University of Artificial Intelligence\\
\texttt{hyda@cmu.edu \quad ps7z@andrew.cmu.edu \quad kunz1@cmu.edu}
}

\begin{document}
\maketitle
\vspace{-0.5em}

\begin{abstract}
  \vspace{-1.2em}
Causal discovery aims to recover causal structures generating the observational data. Despite its success in certain problems, in many real-world scenarios the observed variables are not the target variables of interest, but the imperfect measures of the target variables. Causal discovery under measurement error aims to recover the causal graph among unobserved target variables from observations made with measurement error. We consider a specific formulation of the problem, where the unobserved target variables follow a linear non-Gaussian acyclic model, and the measurement process follows the random measurement error model. Existing methods on this formulation rely on non-scalable over-complete independent component analysis (OICA). In this work, we propose the Transformed Independent Noise ($\TIN$) condition, which checks for independence between a specific linear transformation of some measured variables and certain other measured variables. By leveraging the non-Gaussianity and higher-order statistics of data, $\TIN$ is informative about the graph structure among the unobserved target variables. By utilizing $\EGIN$, the ordered group decomposition of the causal model is identifiable. In other words, we could achieve what once required OICA to achieve by only conducting independence tests. Experimental results on both synthetic and real-world data demonstrate the effectiveness and reliability of our method.\footnote{An online demo and codes are available at \url{https://cmu.edu/dietrich/causality/tin}.}

\end{abstract}

\vspace{-0.5em}
\section{Introduction}
\label{sec:introduction}
\begin{wrapfigure}{r}{0.25\textwidth}
\centering
\vspace{0.7em}
\begin{tikzpicture}[->,>=stealth,shorten >=1pt,auto,node distance=1.25cm,
semithick,square/.style={regular polygon,regular polygon sides=4}]
\node[state, fill=lightgray!50, inner sep=0.1pt, minimum size=0.9cm] (F) [] {\LARGE\factory};
\node[state, fill=lightgray!50, inner sep=0.1pt, minimum size=0.9cm] (A) [right of=F] {\LARGE\air};
\node[state, fill=lightgray!50, inner sep=0.1pt, minimum size=0.9cm] (H) [right of=A] {\LARGE\lungs};
\node[state, inner sep=0.1pt, minimum size=0.9cm] (R) [below of=F] {\LARGE\report};
\node[state, inner sep=0.1pt, minimum size=0.9cm] (P) [below of=A] {\footnotesize PM$_{2.5}$};
\node[state, inner sep=0.1pt, minimum size=0.9cm] (L) [below of=H] {\LARGE\hospital};
\path (F) edge node {} (A)
(A) edge node {} (H)
(F) edge node {} (R)
(A) edge node {} (P)
(H) edge node {} (L);
\end{tikzpicture}
\caption{Example of measurement error. Gray nodes are latent underlying variables and white nodes are observed ones.}
\label{fig:illu_example}
\vspace{-0.7em}
\end{wrapfigure}
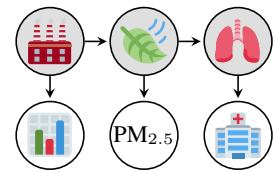Discovery of causal relations is a fundamental goal of science. To identify causal relations from observational data, known as causal discovery, has thus drawn much attention in various scientific fields, including economics, biology, and social science~\citep{spirtes2010automated, pearl2009causality, friedman2000using}. Methods for causal discovery can be roughly categorized into constraint-based ones (e.g., PC~\citep{spirtes1991algorithm}), score-based ones (e.g., Greedy Equivalence Search (GES)~\citep{chickering2002optimal}), and ones based on structural equation models (SEM)~\citep{shimizu2006linear, hoyer2008nonlinear, zhang2012identifiability}. Almost all these methods assume that the recorded values are values of the variables of interest, which however, is usually not the case in real-world scenarios. Some variables may be impossible to observe or quantify, so recorded values are actually a proxy of them (e.g., measure one's mental status by survey questionnaire), and some variables, though quantifiable, may subject to error introduced by instruments (e.g., measure brain signals using functional magnetic resonance (fMRI)). The difference between quantities of interest and their measured value is termed as \textit{measurement error}~\citep{fuller2009measurement}.

Measurement error adversely impairs causal discovery~\citep{pearl2012measurement, kuroki2014measurement, scheines2016measurement}. The measuring process can be viewed as directed edges from underlying variables of interest (unobservable) to measured values (observable), and the d-separation patterns on underlying variables typically do no hold on measured ones. Consider the causal effects from factory emissions $\factory$ to air quality $\air$ then to residents' lung health $\lungs$, as shown in~\cref{fig:illu_example}, while we only have corresponding measured quantities: chimney statistics $\report$, PM$_{2.5}$, and hospital reported cases $\hospital$. Though $\factory$ and $\lungs$ are independent given $\air$, $\report$ and $\hospital$ are however dependent given PM$_{2.5}$. If measurement error is severe, $\report$ and $\hospital$ even tend to be marginally independent~\citep{zhang2017causal, scheines2016measurement}, which makes PM$_{2.5}$ look like a collider (common child). One might thus incorrectly infer that, lung cancer causes air pollution. In fact, such measurement error is always a serious threat in environmental epidemiologic studies~\citep{richardson1993bayesian, edwards2017measurement}.

Denote by $\Xtilde=\{\Xitilde\}_{i=1}^n$ the latent measurement-error-free variables and $\X=\{X_i\}_{i=1}^n$ observed ones. While there are different models for measuring process~\citep{fuller2009measurement,carroll2006measurement,wansbeek2001measurement}, in this paper, we consider the random measurement error model~\citep{scheines2016measurement}, where the observed variables are generated from the latent measurement-error-free variables $\Xitilde$ with additive random measurement errors $\E=\{E_i\}_{i=1}^n$: 
\begin{equation} \label{Eq:def_of_random_measurement_error_model}
X_i = \tilde{X}_i + E_i.
\end{equation}
Measurement errors $\E$ are assumed to be mutually independent and independent of $\Xtilde$. We assume causal sufficiency relative to $\Xtilde$ (i.e., no confounder of $\Xtilde$ who does not have a respective measurement), and focus on the case where $\Xtilde$ is generated by a linear, non-Gaussian, acyclic model (LiNGAM~\citep{shimizu2006linear}, see~\cref{sec:basic_notations}). Note that here w.l.o.g., the linear weights of $\{\Xitilde\rightarrow X_i\}_{i=1}^n$ are assumed to be one (since we do not care about scaling). Generally, if observations are measured by $X_i = c_i\tilde{X}_i + E_i$ with weights $\{c_i\}_{i=1}^n$ not necessarily being one, all results in this paper still hold. \looseness=-1 

The objective of causal discovery under measurement error is to recover causal structure among latent variables $\Xtilde$, denoted as $\Gtilde$, a directed acyclic graph (DAG), from contaminated observations $\X$. As illustrated by~\cref{fig:illu_example}, causal discovery methods that utilize (conditional) independence produce biased estimation (see~\cref{prop:rare_dsep} for details). SEM-based methods also typically fail to find correct directions, since the SEM for $\Xtilde$ usually do not hold on $\X$. Unobserved $\Xtilde$ are actually confounders of $\X$, and there exists approaches to deal with confounders, such as Fast Causal Inference (FCI~\citep{spirtes2000causation}). However, they focus on structure among observed variables instead of the unobserved ones, which is what we aim to recover here. With the interest for the latter, another line of research called \textit{causal discovery with latent variables} is developed, which this paper is also categorized to. However, existing methods~\citep{silva2005generalized, silva2006learning, sullivant2010trek, spirtes2013calculation, anandkumar2013learning, kummerfeld2016causal, xie2020generalized, salehkaleybar2020learning} cannot be adopted either, since they typically require at least two measurements (indicators) for each latent variable, while we only have one for each here (and is thus a more difficult task). Specifically on the measurement error problem, \citep{halpern2015anchored} proposes anchored causal inference in the binary setting. In the linear Gaussian setting, \citep{zhang2017causal} presents identifiability conditions by factor analysis. A main difficulty here is the unknown variances of the measurement errors $\E$, otherwise the covariance matrix of $\Xtilde$ can be obtained and readily used. To this end, \citep{blom2018upper} provides an upper-bound of $\E$ and~\citep{saeed2020anchored} develops a consistent partial correlations estimator. In linear non-Gaussian settings (i.e., the setting of this paper), \citep{zhang2018causal} shows that the \textit{ordered group decomposition} of $\Gtilde$, which contains major causal information, is identifiable. However, the corresponding method relies on over-complete independent component analysis (OICA~\citep{hyvarinen2000independent}), which is notorious for suffering from local optimal and high computational complexity~\citep{shimizu2009estimation, hoyer2008estimation}. Hence, the identifiability results in~\citep{zhang2018causal}, despite the theoretical correctness, is far from practical achievability.

The main contributions of this paper are as follows: \textbf{1)} We define the \textbf{T}ransformed \textbf{I}ndependent \textbf{N}oise ($\TIN$) condition, which finds and checks for independence between a specific linear transformation (combination) of some variables and others. The existing Independent Noise ($\IN$~\citep{shimizu2011directlingam}) and Generalized Independent Noise ($\GIN$~\citep{xie2020generalized}) conditions are special cases of $\TIN$. \textbf{2)} We provide graphical criteria of $\TIN$, which might further improve identifiability of causal discovery with latent variables. \textbf{3)} We exploit $\TIN$ on a specific task, causal discovery under measurement error and LiNGAM, and identify the \textit{ordered group decomposition}. This identifiability result once required computationally and statistically ineffective OICA to achieve, while we achieve it merely by conducting independence tests. Evaluation on both synthetic and real-world data demonstrate the effectiveness of our method.

\vspace{-0.6em}
\section{Motivation: Independence Condition and Structural Information}
\label{sec:motivation}
The example in~\cref{fig:illu_example} illustrates how the (conditional) (in)dependence relations differ between observed $\X$ and latent $\Xtilde$, and thus lead to biased discovery results. To put it generally, we have,
\begin{restatable}[\textit{rare} d-separation]{proposition}{RAREDSEP}\label{prop:rare_dsep}Suppose variables follow random measurement error model defined in~\cref{Eq:def_of_random_measurement_error_model}. For disjoint sets of observed variables $\Z, \Y, \SH$ and their respective latent ones $\Ztilde, \Ytilde, \Tilde{\SH}$, d-separation $\Z \indep_{\mkern-9.5mu d} \Y \vert \SH$ holds, only when marginally $\Ztilde \indep_{\mkern-9.5mu d} \Ytilde$, and $\Ztilde \indep_{\mkern-9.5mu d} \Ytilde\vert \Tilde{\SH}$ hold.%
\end{restatable}
By \textit{`rare'} we mean that the d-separation patterns among $\Xtilde$ usually do not hold among $\X$ (except for \textit{rare} marginal ones), since the observed variables are not causes of any other (though the latent variables they intend to measure might be). For example, consider the underlying $\Gtilde$ to be chain structure (\cref{fig:chain}) and fully connected DAG (\cref{fig:fully_connected}). There exists no (conditional) independence on either, and PC algorithm will output just a fully connected skeleton on both cases. Then, without (conditional) independence (which is non-parametric) to be directly used, can we \textit{create independence}, by leveraging the parametric assumption (LiNGAM) and benefit from non-Gaussianity? 

Naturally we recall the Independent Noise (IN) condition proposed in Direct-LiNGAM~\citep{shimizu2011directlingam}:
\begin{definition}[$\IN$ condition]\label{def:in_definition}
Let $Y_i$ be a single variable and $\Z$ be a set of variables. Suppose variables follow LiNGAM. We say ($\Z,Y_i$) satisfies $\IN$ condition, denoted by $\IN(\Z,Y_i)$, if and only if the residual of regressing $Y_i$ on $\Z$ is statistically independent to $\Z$. Mathematically, let $\Tilde{\w}$ be the vector of regression coefficients, i.e., $\Tilde{\w}\coloneqq \cov(Y_i,\Z)\cov(\Z,\Z)^{-1}$; $\IN(\Z,Y_i)$ holds iff $Y_i - \Tilde{\w}^\intercal \Z \indep \Z$.
\end{definition}\vspace{-0.3em}
Here ``$\cov$'' denotes the variance-covariance matrix. $\IN$ identifies exogenous (root) variables, based on which the causal ordering of variables can be determined (Lemma 1 in~\citep{shimizu2011directlingam}). However, $\IN$ cannot be applied to measurement error model. With hidden confounders ($\Xtilde$) behind observed $\X$, independence between regressor and residual typically does not exist on any regression among $\X$. In fact, $\X$ follows errors-in-variables models~\citep{griliches1970error,chesher1991effect}, for which the identifiability w.r.t. $\Gtilde$ is not clear.

However, we might still benefit from this idea to leverage non-Gaussianity of exogenous noises. Consider the~\cref{fig:illu_example} example and abstract it to $\smash{\Tilde{X}_1 \stackrarrow{a} \Tilde{X}_2 \stackrarrow{b} \Tilde{X}_3}$ with $\smash{\{\Xitilde\stackrarrow{1}X_i\}_{i=1}^3}$. Although $\IN$ does not hold on any of $\X$, interestingly, there exists a linear transformation of observations $bX_2 - X_3$, which contains only $\{\Tilde{E}_3, E_2, E_3\}$ ($\Tilde{E}_1$ and $\Tilde{E}_2$ are cancelled out) and shares no common non-Gaussian noise term with $X_1$. Hence, by the $\darmois$~\citep{kagan1973characterization}, $bX_2 - X_3$ is independent of $X_1$. This finding is echoed in Generalized Independent Noise ($\GIN$~\citep{xie2020generalized}) condition: %

\begin{definition}[$\GIN$ condition]\label{def:gin_definition}
Let $\mathbf{Z}$ and $\mathbf{Y}$ be two sets of random variables that follow LiNGAM. We say ($\Z,\Y$) satisfies the $\GIN$ condition, denoted by $\GIN(\Z,\Y)$, if and only if the following two conditions are satisfied: \textbf{1)} There exists nonzero solution vectors $\omega\in\Real^{\vert \Y \vert}$ to equation $\cov(\Z,\Y)\omega=\mathbf{0}$, and \textbf{2)} Any such solution $\omega$ makes the linear transformation $\wTY$ independent of $\Z$.
\end{definition}
\vspace{-0.3em}
Here $\vert \Y \vert$ denotes the dimensionality of $\Y$. The intuition of $\GIN$ is that, despite no independent residual by normal regression, it is possible to realize independent \textit{``pseudo-residuals''}~\citep{cai2019triad} by regressing with \textit{``reference variables''}. \citep{xie2020generalized} shows $\IN$ as a special case of $\GIN$ (Proposition 2), and further gives graphical criteria (Theorem 2), based on which a recursive learning algorithm is developed to solve the latent-variable problem. Each latent variable is required to have at least two observations. Interestingly we find that, in measurement error models, if each latent variable $\Xitilde$ has two measurements $\smash{X_{i_1}, X_{i_2}}$, then the $\GIN$ condition can be readily used to \textit{fully identify} the structure of $\Gtilde$, which is already a breakthrough over existing methods~\citep{silva2006learning, sullivant2010trek, spirtes2013calculation, kummerfeld2016causal}. See~\cref{app:gin_two_measurements} for the whole procedure. With regard to our more challenging task where each $\Xitilde$ has only one measurement $X_i$, a natural question is that, can $\GIN$ also help? Given the example in~\cref{fig:illu_example} illustrated above, the answer seems to be affirmative: $\GIN(\{X_i\}, \X\backslash \{X_i\})$ only holds for $i=1$, so the root can be identified. More generally: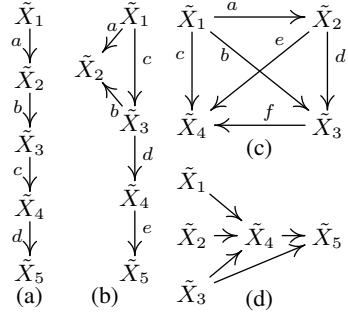
\begin{wrapfigure}[14]{r}{0.35\textwidth}
\vspace{-0.3em}
\begin{minipage}[b][0.30\textwidth][t]{0.05\textwidth}\centering
\begin{tikzpicture}[scale=0.6, line width=0.5pt, inner sep=0.2mm, shorten >=.1pt, shorten <=.1pt]
		\draw (0, 5.7) node(1)  {{\footnotesize\,$\Tilde{X}_1$\,}};
		\draw (0, 4.275) node(2)  {{\footnotesize\,$\Tilde{X}_2$\,}};
		\draw (0, 2.85) node(3)  {{\footnotesize\,$\Tilde{X}_3$\,}};
		\draw (0, 1.425) node(4)  {{\footnotesize\,$\Tilde{X}_4$\,}};
		\draw (0, 0) node(5)  {{\footnotesize\,$\Tilde{X}_5$\,}};
		\draw[-arcsq] (1) -- (2) node[pos=0.4, left=2pt] {\scriptsize{$a$}}; 
		\draw[-arcsq] (2) -- (3) node[pos=0.4, left=2pt] {\scriptsize{$b$}}; 
		\draw[-arcsq] (3) -- (4) node[pos=0.4, left=2pt] {\scriptsize{$c$}}; 
		\draw[-arcsq] (4) -- (5) node[pos=0.4, left=2pt] {\scriptsize{$d$}}; 
	\end{tikzpicture}\vspace{-0.6em}
    \subcaption{} \label{fig:chain} %
\end{minipage}%
\hfill
  \begin{minipage}[b][0.30\textwidth][t]{0.07\textwidth}
  \centering
    \begin{tikzpicture}[scale=0.6, line width=0.5pt, inner sep=0.2mm, shorten >=.1pt, shorten <=.1pt]
		\draw (0, 5.7) node(1)  {{\footnotesize\,$\Tilde{X}_1$\,}};
		\draw (-1, 4.5125) node(2)  {{\footnotesize\,$\Tilde{X}_2$\,}};
		\draw (0, 3.325) node(3)  {{\footnotesize\,$\Tilde{X}_3$\,}};
		\draw (0, 1.6625) node(4)  {{\footnotesize\,$\Tilde{X}_4$\,}};
		\draw (0, 0) node(5)  {{\footnotesize\,$\Tilde{X}_5$\,}};
		\draw[-arcsq] (1) -- (2) node[pos=0.5,above=3pt] {\scriptsize{$a$}}; 
		\draw[-arcsq] (1) -- (3) node[pos=0.4,right=2pt] {\scriptsize{$c$}}; 
		\draw[-arcsq] (3) -- (2) node[pos=0.4,below=2pt] {\scriptsize{$b$}}; 
		\draw[-arcsq] (3) -- (4) node[pos=0.4,right=2pt] {\scriptsize{$d$}}; 
		\draw[-arcsq] (4) -- (5) node[pos=0.4,right=2pt] {\scriptsize{$e$}}; 
		\end{tikzpicture}\vspace{-0.6em}
    \subcaption{} \label{fig:triangle_head_chain} %
  \end{minipage}%
\hfill
  \begin{minipage}[b][0.30\textwidth][t]{0.20\textwidth}
  \centering
    \begin{tikzpicture}[scale=0.6, line width=0.5pt, inner sep=0.2mm, shorten >=.1pt, shorten <=.1pt]
		\draw (0.5, 3.825) node(1)  {{\footnotesize\,$\Tilde{X}_1$\,}};
		\draw (3.5, 3.825) node(2)  {{\footnotesize\,$\Tilde{X}_2$\,}};
		\draw (3.5, 1.425) node(3)  {{\footnotesize\,$\Tilde{X}_3$\,}};
		\draw (0.5, 1.425) node(4)  {{\footnotesize\,$\Tilde{X}_4$\,}};
		\draw[-arcsq] (1) -- (2) node[pos=0.2,above=2pt] {\scriptsize{$a$}}; 
		\draw[-arcsq] (1) -- (3) node[pos=0.15,below=1pt] {\scriptsize{$b$}}; 
		\draw[-arcsq] (1) -- (4) node[pos=0.3, left=2pt] {\scriptsize{$c$}}; 
		\draw[-arcsq] (2) -- (3) node[pos=0.3,right=2pt] {\scriptsize{$d$}}; 
		\draw[-arcsq] (2) -- (4) node[pos=0.3,above=3pt] {\scriptsize{$e$}}; 
		\draw[-arcsq] (3) -- (4) node[pos=0.4,above=1pt] {\scriptsize{$f$}}; 
		\end{tikzpicture}\vspace{-0.7em}
    \subcaption{} \label{fig:fully_connected} %
    \begin{tikzpicture}[scale=0.6, line width=0.5pt, inner sep=0.2mm, shorten >=.1pt, shorten <=.1pt]
		\draw (0.5, 3.75) node(1)  {{\footnotesize\,$\Tilde{X}_1$\,}};
		\draw (0.5, 2.5) node(2)  {{\footnotesize\,$\Tilde{X}_2$\,}};
		\draw (0.5, 1.25) node(3)  {{\footnotesize\,$\Tilde{X}_3$\,}};
		\draw (2, 2.5) node(4)  {{\footnotesize\,$\Tilde{X}_4$\,}};
		\draw (3.5, 2.5) node(5)  {{\footnotesize\,$\Tilde{X}_5$\,}};
		\draw[-arcsq] (1) -- (4);
		\draw[-arcsq] (2) -- (4);
		\draw[-arcsq] (3) -- (4);
		\draw[-arcsq] (3) -- (5);
		\draw[-arcsq] (4) -- (5);
		\end{tikzpicture}\vspace{-1.3em}
    \subcaption{} \label{fig:example_d} %
    
  \end{minipage}
  \vspace{-0.5em}
\caption{Graph structure $\Gtilde$ examples. For simplicity, observed $\X$, measurement edges are omitted.} \label{fig:graph_examples}
\end{wrapfigure}
\begin{example}[$\GIN$ on chain structure]\label{eg:gin_on_chain} Consider cases where the underlying graph $\Gtilde$ is a chain structure with $n$ ($n\geq3$) vertices and directed edges $\smash{\{\Tilde{X}_i\rightarrow \Tilde{X}_{i+1}\}_{i=1}^{n-1}}$. \cref{fig:chain} is an example with $n=5$. We find that $\GIN(\Z=X_1,\Y=X_{2,3,4,5})$ holds (where $X_{2,3,4,5}$ denotes $\{X_2,X_3,X_4,X_5\}$; same below), with solution $\w=[-bx-bcy-bcdz,x,y,z]^\T$, $x,y,z\in\Real$. $\wTY$ cancels noise components in $\Y$ that are also shared by $\Z$, and thus $\wTY\indep\Z$. However, $\GIN(X_i,\X\backslash\{X_i\})$ is violated for any other $i\neq 1$ (see~\cref{app:eg_gin_on_chain_details} for a detailed derivation). With this asymmetry, the latent root $\smash{\Tilde{X}_1}$ can be identified. Furthermore, $\GIN(X_i, X_{i+1,\cdotsshort,n})$ holds for any $i=1,\cdotsshort,n-2$.%
\end{example}
\cref{eg:gin_on_chain} might give us an intuition that by recursively testing $\GIN$ (over the newly found subroot and the remaining variables), we could identify the causal ordering of first $n-2$ variables for any DAG. However, this is over-optimistic thanks to the sparsity of chain structure. Consider a denser structure:
\vspace{0.1em}
\begin{example}[$\GIN$ on fully connected DAG]\label{eg:gin_on_fully_connected} Consider cases where $\Gtilde$ is a fully connected DAG with $n$ ($n\geq3$) vertices and directed edges $\smash{\{\Tilde{X}_i\rightarrow \Tilde{X}_j\}}$ for every $i<j$. \cref{fig:fully_connected} is an example with $n=4$. 
We first find that $\GIN(X_1, X_{2,3,4})$ holds. However, in contrast to the chain structure, $\GIN(X_2,X_{3,4})$ does not hold - there is no way to cancel both $\tilde{E}_{1,2}$ from $X_{3,4}$. More generally we have: $\GIN(X_1, X_{2,\cdotsshort,n})$, $\GIN(X_{1,2}, X_{3,\cdotsshort,n})$, $\cdotsshort$, $\GIN(X_{1,\cdotsshort,k}, X_{k+1,\cdotsshort,n})$, $k=\lfloor (n-1)/2 \rfloor$.
\end{example}\vspace{-0.3em}
Since a fully connected DAG is the densest extreme,~\cref{eg:gin_on_fully_connected} might give us an intuition that $\GIN$ could identify at least the causal ordering of the first half of variables. Unfortunately, this is still over-optimistic, since we could not know beforehand the structure type of $\Gtilde$.

\begin{example}[$\GIN$ on chain structure with \textit{triangular head}]\label{eg:gin_on_triangle_head_chain} \cref{fig:triangle_head_chain} shows a variation of chain structure, with edges $\smash{\Tilde{X}_1 \rightarrow \Tilde{X}_2, \Tilde{X}_1\rightarrow\Tilde{X}_3, \Tilde{X}_3\rightarrow\Tilde{X}_2}$, and $\smash{\{\Tilde{X}_i\rightarrow \Tilde{X}_{i+1}\}_{i=3}^{n-1}}$. We name it ``chain structure with \textit{triangular head}''. Interestingly, $\GIN(X_2, X_{3,\cdotsshort,n})$, which is satisfied on~\cref{fig:chain}, is also satisfied here. E.g., in~\cref{fig:triangle_head_chain} ($n=5$), the solution $\w = [-dx-dey,x,y]^\T$, $x,y\in\Real$. $\wTY$ cancels $\Tilde{E}_{1,3}$ from $\Y$, and thus $\wTY\indep\Z$. Actually, the chain structure with \textit{triangular head} $\Tilde{G}_{\triangledown \mathbb{C}}$ is \textit{unidentifiable} with chain structure $\Tilde{G}_{\mathbb{C}}$ w.r.t. $\GIN$ conditions, i.e., for any two sets of observed variables $\Z,\Y\subseteq\X$, $\GIN(\Z,\Y)$ holds on $\Tilde{G}_{\triangledown \mathbb{C}}$ if and only if $\GIN(\Z,\Y)$ holds on $\Tilde{G}_{\mathbb{C}}$. Consequently, if directly using any recursive algorithm by $\GIN$ as in~\cref{eg:gin_on_chain}, the output causal ordering would still be $\Tilde{X}_1, \Tilde{X}_2, \Tilde{X}_3,\cdotsshort$, which is incorrect due to $\Tilde{X}_3\rightarrow \Tilde{X}_2$ in the triangular head.
\end{example}

The above examples show that it is not as simple as it seems to use $\GIN$ for one-measurement model: only part of the causal ordering can be identified, and worse yet, rather complicated error correction is needed to deal with possible incorrect orderings. However, after a closer look at the unidentifiable examples above, we find that actually more information can be uncovered beyond the $\GIN$ condition:\looseness=-1

\begin{example}[Asymmetry beyond $\GIN$]\label{eg:revisiiting_difficult_gin_examples} \textbf{1)} Consider~\cref{eg:gin_on_fully_connected} where only the root variable $\Tilde{X}_1$ can be identified by $\GIN$ and $\Tilde{X}_{2,3,4}$ are \textit{unidentifiable}, i.e., permutation of the labeling of $X_{2,3,4}$ will still preserve the $\GIN$ condition over any two subsets $\Z,\Y$. However, there actually exists an asymmetry between $X_2$ and $X_{3,4}$: we could construct linear transformation of $X_{1,3,4}$: $\frac{cd-be}{d}X_1 + \frac{df+e}{d}X_3 - X_4$ s.t. it is independent of $X_2$, while there exists no such linear transformation of $X_{1,2,4}$ to be independent of $X_3$, and also for $X_{1,2,3}$ to $X_4$. \textbf{2)} Consider~\cref{eg:gin_on_chain,eg:gin_on_triangle_head_chain} where $\Tilde{G}_{\mathbb{C}}$ and $\Tilde{G}_{\triangledown \mathbb{C}}$ are unidentifiable w.r.t. $\GIN$ conditions. Let $\Z\coloneqq X_4$ and $\Y\coloneqq X_{1,2,3}$, $\GIN(\Z,\Y)$ is violated on both graphs. However, an asymmetry actually exists: on $\Tilde{G}_{\triangledown \mathbb{C}}$, we could construct $aX_1-X_2+bX_3$ (which cancels $\Tilde{E}_{1,3}$) to be independent to $X_4$, while this is impossible on $\Tilde{G}_{\mathbb{C}}$.
\end{example}

To put simply, the motivation of independent \textit{``pseudo-residual''} behind $\GIN$ actually limits the power of non-Gaussianity, with the coefficients vector $\w$ only characterized from variance-covariance matrix (2nd-order). There are actually two cases for $\GIN(\Z,\Y)$ to be violated: 1) though \textit{not all solution} $\w$ makes $\wTY\indep\Z$, there \textit{exists} non-zero $\w$ s.t. $\wTY\indep\Z$, and 2) there naturally exists \textit{no} non-zero $\w$ s.t. $\wTY\indep\Z$. The original $\GIN$ cannot distinguish between these two cases. Hence in this paper, we first aim to distinguish between the two, generalizing $\GIN$ condition to $\EGIN$ condition.
\vspace{-1em}

\section{With Transformed Independent Noise Condition}
\label{sec:approach}
\vspace{-0.5em}
In the above discussion, one can see that the presence of measurement error affects the conditional independence relations among the variables and the independent noise condition. However, we will show in~\cref{thm:equiv_gin_true_observed} (\cref{sec:exploit}) that a specific type of independence conditions are shared between the underlying error-free variables and the measured variables with error. In this section, we will formulate such independence conditions and investigate their graphical implications for error-free variables (i.e., variables generated by the LiNGAM without measurement error). In~\cref{sec:exploit}, we will then extend the results to the measured variables. Please note that \underline{in contrast to other sections}, the notation used in this section, including $\X$, $\Y$, and $\SH$, denotes \textit{error-free variables} generated by the LiNGAM.\looseness=-1
\subsection{Notations}\label{sec:basic_notations}
Let $G$ be a directed acyclic graph with the vertex set $V(G)=[n]\coloneqq\{1,2,\cdots,n\}$ and edge set $E(G)$. A directed path $P = (i_0, i_1, \cdots, i_k)$ in $G$ is a sequence of vertices of $G$ where there is a directed edge from $i_j$ to $i_{j+1}$ for any $0\leq j \leq k-1$. We use notation $i\rightsquigarrow j$ to show that there exists a directed path from vertex $i$ to $j$. Note that a single vertex is also a directed path, i.e., $i\rightsquigarrow i$ holds. Let $\Z \subseteq [n]$ be a subset of vertices. Define ancestors $\Anc(\Z)\coloneqq \{j|\exists i\in\Z, j\rightsquigarrow i\}$. Note that $\Z\subseteq \Anc(\Z)$. Further let $\SH$ be a subset of vertices. We use notation $i \stackrqarrow{\cancel{[\SH]}} j$ to show that there exists a directed path from vertex $i$ to $j$ without passing through $\SH$, i.e., there exists a directed path $P = (i, m_0, \cdots, m_k,j)$ in $G$ s.t. $i,j\not\in\SH$ and $m_l\not\in\SH$ for any $0\leq l \leq k$. Define \textit{ancestors outside $\SH$} accordingly: for two vertex sets $\Y,\SH$, denote ancestors of $\Y$ that have directed paths into $\Y$ without passing through $\SH$ as $\Ancout{\SH}(\Y)\coloneqq \{j|\exists i\in\Y, j\stackrqarrow{\cancel{[\SH]}} i\}$. Note that the graphical definitions here can also be translated to \textit{trek}~\citep{sullivant2010trek} language (see~\cref{app:vertex_cut} for details).

Assume random variables $\X\coloneqq \{X_i\}_{i\in [n]}$ are generated by LiNGAM w.r.t. graph $G$, i.e.,
\begin{equation} \label{Eq:lingam_generating_process}
\X=\A\X+\E=\B\E\text{, with }\B=(\I-\A)^{-1}.
\end{equation}
where $\E= \{E_i\}_{i\in [n]}$ are corresponding mutually independent exogenous noises. $\A$ is the adjacency matrix where entry $\A_{j,i}$ is the linear weight of direct causal effect of variable $X_i$ on $X_j$. $\A_{j,i}\neq 0$ if and only if there exists edge $i\rightarrow j$. $\X$ can also be written directly as a mixture of exogenous noises $\X=\B\E$. If the entry of mixing matrix $\B_{j,i}\neq 0$, then $i\in\Anc(\{j\})$. Note that here and in what follows, we use boldface letters $\A,\B$ to denote matrices, and use boldface letters $\SH,\W,\X,\Y,\Z$ with notation abuse: it can denote vertices set, respective random variables set, or random vector. When we say ``two variables sets $\Z,\Y$'', if not otherwise specified, $\Z,\Y$ need not be disjoint.

\subsection{Independent Linear Transformation Subspace and its Characterization}\label{sec:independent_linear_transformation_subspace}
We first give the definition and characterization of the  \textit{independent linear transformation subspace}.
\begin{definition}[Independent linear transformation subspace]\label{def:independent_omega_subspace}
Let $\Z$ and $\Y$ be two subsets of random variables. Suppose the variables follow the linear non-Gaussian acyclic causal model. Denote:
\begin{equation} \label{Eq:Omega_ZY_definition}\Omega_{\Z;\Y}\coloneqq \{\omega\in \Real^{\vert \Y \vert} \ \vert \ \omega^\intercal \Y \indep \Z\}.\end{equation}\end{definition}\vspace{-0.3em}
By the property of independence, $\Omega_{\Z;\Y}$ is closed under scalar multiplication and addition, and thus is a subspace in $\Real^{\vert \Y \vert}$. In fact, $\Omega_{\Z;\Y}$ can be characterized as a nullspace as follows:
\begin{restatable}[Characterization of $\Omega_{\Z;\Y}$]{theorem}{CHARACTERIZATIONOMEGAZY}\label{thm:characterization_Omega_ZY} 
For two variables subsets $\Z$ and $\Y$, $\OmegaZY$ satisfies:
\begin{equation} \label{Eq:Omega_ZY_characterization_general}
\Omega_{\Z;\Y} = \operatorname{null}(\B_{\Y, \operatorname{nzcol}(\B_{\Z,:})}^\intercal).\end{equation}\vspace{0.2em}
where $\operatorname{null}(\cdot)$ denotes nullspace. $\B_{\Y, \operatorname{nzcol}(\B_{\Z,:})}$ denotes the submatrix of mixing matrix $\B$, with rows indexed by $\Y$ and columns indexed by $\operatorname{nzcol}(\B_{\Z,:})$. $\operatorname{nzcol}(\B_{\Z,:})$ denotes the column indices where the submatrix $\B_{\Z,:}$ has non-zero entries. $\operatorname{nzcol}(\B_{\Z,:})$ actually corresponds to the exogenous noises that constitute $\Z$. Particularly, if assuming ``if $i\rightsquigarrow j$ then $\B_{j,i}\neq 0$'', then, $\operatorname{nzcol}(\B_{\Z,:})=\Anc(\Z)$.
\end{restatable}
Proof of~\cref{thm:characterization_Omega_ZY} is straight-forward by the $\darmois$~\citep{kagan1973characterization}: linear transformation $\wTY\indep\Z$ if and only if $\wTY$ shares no common non-Gaussian exogenous noise components with $\Z$.
\begin{example}[Revisiting examples in~\cref{sec:motivation} from $\OmegaZY$ perspective]\label{eg:revisit_from_Omega_ZY_view}
For illustration, now we revisit the examples in~\cref{sec:motivation} from the perspective of independent linear transformation subspace.\vspace{-0.2em}
\begin{equation} \vspace{-0.2em}
\label{Eq:B_mixing_matrix_examples}
    \begin{tikzpicture}[
every left delimiter/.style={xshift=.5em},
    every right delimiter/.style={xshift=-.5em},
strip/.style = {
    draw=#1,
    line width=1em, opacity=0.2,
    line cap=round ,},]
\matrix (mtrx)  [matrix of math nodes,
                left delimiter={[},
                right delimiter={]},
                inner sep=1.5pt, column sep=7pt, row sep=2pt,
                 ]
{
1 & 0 & 0 & 0 & 0\\
a & 1 & 0 & 0 & 0\\
ab & b & 1 & 0 & 0\\
abc & bc & c & 1 & 0 \\
abcd & bcd & cd & d & 1 \\
};
\fill[red, opacity=0.2, rounded corners]let \p1 = (mtrx-2-1.north east), \p2 = (mtrx-5-1.south east) in (mtrx-5-1.south west) rectangle (\x2,\y1);
\fill[blue, opacity=0.2, rounded corners]let \p1 = (mtrx-3-2.north east), \p2 = (mtrx-5-2.south east) in (mtrx-5-1.south west) rectangle (\x2,\y1);
\fill[green, opacity=0.2, rounded corners]let \p1 = (mtrx-4-1.north east), \p2 = (mtrx-5-3.south east) in (mtrx-5-1.south west) rectangle (\x2,\y1);
\fill[orange, opacity=0.2, rounded corners]let \p1 = (mtrx-1-4.north east), \p2 = (mtrx-3-4.south east) in (mtrx-3-1.south west) rectangle (\x2,\y1);
\end{tikzpicture}\noindent,
\begin{tikzpicture}[
every left delimiter/.style={xshift=.5em},
    every right delimiter/.style={xshift=-.5em},
strip/.style = {
    draw=#1,
    line width=1em, opacity=0.2,
    line cap=round ,},]
\matrix (mtrx)  [matrix of math nodes,
                left delimiter={[},
                right delimiter={]},
                inner sep=1.5pt, column sep=7pt, row sep=2pt,
                 ]
{
1 & 0 & 0 & 0 & 0\\
a+bc & 1 & b & 0 & 0\\
c & 0 & 1 & 0 & 0\\
cd & 0 & d & 1 & 0 \\
cde & 0 & de & e & 1 \\
};
\fill[red, opacity=0.2, rounded corners]let \p1 = (mtrx-2-1.east), \p2 = (mtrx-5-1.south) in (mtrx-2-1.north west) rectangle (\x1,\y2);
\fill[blue, opacity=0.2, rounded corners]let \p1 = (mtrx-5-3.east), \p2 = (mtrx-3-2.north),  \p3 = (mtrx-2-1.west), \p4 = (mtrx-5-1.south) in (\x3,\y4) rectangle (\x1,\y2);
\fill[green, opacity=0.2, rounded corners]let \p1 = (mtrx-5-1.east), \p2 = (mtrx-4-2.north) in (mtrx-5-1.south west) rectangle (\x1,\y2);
\fill[green, opacity=0.2, rounded corners]let \p1 = (mtrx-5-3.east), \p2 = (mtrx-4-2.north) in (mtrx-5-3.south west) rectangle (\x1,\y2);
\fill[orange, opacity=0.2, rounded corners]let \p1 = (mtrx-2-1.west), \p2 = (mtrx-2-1.east), \p3 = (mtrx-3-1.south), \p4 = (mtrx-1-2.north) in (\x1,\y3) rectangle (\x2,\y4);
\fill[orange, opacity=0.2, rounded corners] (mtrx-1-3.north west) rectangle (mtrx-3-4.south east);
\end{tikzpicture}\noindent,
\begin{tikzpicture}[
every left delimiter/.style={xshift=.5em},
    every right delimiter/.style={xshift=-.5em},
strip/.style = {
    draw=#1,
    line width=1em, opacity=0.2,
    line cap=round ,},]
\matrix (mtrx)  [matrix of math nodes,
                left delimiter={[},
                right delimiter={]},
                inner sep=1.5pt, column sep=7pt, row sep=2pt,
                 ]
{
1 & 0 & 0 & 0\\
a & 1 & 0 & 0\\
ad+b & d & 1 & 0\\
a(df+e)+bf+c & df+e & f & 1\\
};
\fill[red, opacity=0.2, rounded corners]let \p1 = (mtrx-2-1.north east), \p2 = (mtrx-4-1.south east) in (mtrx-4-1.south west) rectangle (\x2,\y1);
\fill[green, opacity=0.2, rounded corners]let \p1 = (mtrx-3-2.north east), \p2 = (mtrx-4-2.south east) in (mtrx-4-1.south west) rectangle (\x2,\y1);
\fill[green, opacity=0.2, rounded corners]let \p1 = (mtrx-4-1.south west), \p2 = (mtrx-4-2.south east), \p3 = (mtrx-1-1.south), \p4 = (mtrx-1-1.north) in (\x1, \y3) rectangle (\x2,\y4);
\fill[orange, opacity=0.2, rounded corners]let \p1 = (mtrx-1-1.north), \p2 = (mtrx-3-4.south east) in (mtrx-3-1.south west) rectangle (\x2,\y1);
\end{tikzpicture}
\vspace{-1.4em}
\end{equation}

\vspace{0.3em}

\cref{Eq:B_mixing_matrix_examples} shows the corresponding mixing matrix $\B$ for graph $\Gtilde$ in~\cref{fig:chain,fig:fully_connected,fig:triangle_head_chain}, respectively. Suppose we have access to underlying variables $\Xitilde$ and only focus on $\Gtilde$. Colored blocks denote submatrices of $\B$. \textbf{1)} For the fully connected DAG (\cref{fig:fully_connected}, the right matrix), to identify the root $\smash{\Tilde{X}_1}$, $\smash{\GIN(\Xt_1,\Xt_{2,3,4})}$ is satisfied, corresponding to $\smash{\OmegaZY=\nullspace(\tikz[baseline=1.5ex] {\node[fill, red, opacity=0.2, rounded corners, minimum width=1.7ex, minimum height=2.7ex, anchor=center, shape=rectangle] at (1ex,2ex) {}; } ^\T)}$. For $\smash{(\Z,\Y)\coloneqq (\Xt_2, \Xt_{3,4})}$ or $\smash{(\Xt_4, \Xt_{1,2,3})}$, there exists no non-zero $\w$ s.t. $\wTY\indep\Z$, because the lower $\smash{\tikz[baseline=1.1ex] {\node[fill, green, opacity=0.2, rounded corners, minimum width=2.5ex, minimum height=2.5ex, anchor=center, shape=rectangle] at (1ex,2ex) {}; }}$ part and $\smash{\tikz[baseline=1.1ex] {\node[fill, orange, opacity=0.2, rounded corners, minimum width=4ex, minimum height=2.5ex, anchor=center, shape=rectangle] at (1ex,2ex) {}; }}$ are full row rank. However, if we set $\smash{(\Z,\Y)\coloneqq (\Xt_2, \Xt_{1,3,4})}$, we actually have $\smash{\Xt_2 \indep \frac{cd-be}{d}\Xt_1 + \frac{df+e}{d}\Xt_3 - \Xt_4}$, because the stacked two parts of $\smash{\tikz[baseline=1.1ex] {\node[fill, green, opacity=0.2, rounded corners, minimum width=2.8ex, minimum height=2.8ex, anchor=center, shape=rectangle] at (1ex,2ex) {}; }}$ has rank $2<3$ - though $\GIN$ is still violated because $\smash{\cov(\Z,\Y)}$ has rank $1<2$. \textbf{2)} For the chain structure $\smash{\Tilde{G}_{\mathbb{C}}}$ (\cref{fig:chain}, the left matrix) and chain with triangular head $\smash{\Tilde{G}_{\triangledown \mathbb{C}}}$ (\cref{fig:triangle_head_chain}, the middle matrix), $\GIN$ is satisfied on $\smash{(\Xt_1,\Xt_{2,3,4,5})}$, $\smash{\Xt_2,\Xt_{3,4,5})}$, $\smash{(\Xt_3,\Xt_{4,5})}$, with the $\smash{\B_{\Y,\Anc(\Z)}}$ submatrices being $\smash{\tikz[baseline=1.5ex] {\node[fill, red, opacity=0.2, rounded corners, minimum width=1.5ex, minimum height=3ex, anchor=center, shape=rectangle] at (1ex,2ex) {}; }}$, $\smash{\tikz[baseline=1.5ex] {\node[fill, blue, opacity=0.2, rounded corners, minimum width=2.0ex, minimum height=2.5ex, anchor=center, shape=rectangle] at (1ex,2ex) {}; }}$, $\smash{\tikz[baseline=1.5ex] {\node[fill, green, opacity=0.2, rounded corners, minimum width=3ex, minimum height=2ex, anchor=center, shape=rectangle] at (1ex,2ex) {}; }}$ respectively. Note that though the shape of submatrices are different between $\smash{\Tilde{G}_{\mathbb{C}}}$ and $\smash{\Tilde{G}_{\triangledown \mathbb{C}}}$, their ranks are always equal, and is thus unidentifiable by $\GIN$. However, let $\smash{(\Z,\Y)\coloneqq (\Xt_4, \Xt_{1,2,3})}$, an asymmetry actually exists: in $\smash{\Tilde{G}_{\triangledown \mathbb{C}}}$, $\smash{\Xt_4\indep a\Xt_1-\Xt_2+b\Xt_3}$, because in the right matrix $\smash{\tikz[baseline=1.1ex] {\node[fill, orange, opacity=0.2, rounded corners, minimum width=2.6ex, minimum height=2.6ex, anchor=center, shape=rectangle] at (1ex,2ex) {}; }}$ has rank $2<3$, while in the left matrix $\smash{\tikz[baseline=1.1ex] {\node[fill, orange, opacity=0.2, rounded corners, minimum width=3.ex, minimum height=2.2ex, anchor=center, shape=rectangle] at (1ex,2ex) {}; }}$ is full row rank, so there is no such independence in $\Tilde{G}_{\mathbb{C}}$.
\end{example}

\subsection{Graphical Criteria of Independent Linear Transformation Subspace}\label{sec:graphical_criteria_independent_linear_transformation_subspace}
\cref{sec:independent_linear_transformation_subspace} characterizes $\OmegaZY$ by submatrix of $\B$, which entails information on graph structure and edge weights. The following sections go one step further, explicitly exhibit the graphical criteria of $\OmegaZY$, and investigate how $\OmegaZY$ could help to identify the causal structure. We first give the assumption:

\begin{restatable}[Rank faithfulness]{assumption}{RANKFAITHFULNESS}\label{assum:rank_faithfulness}
Denote by $\mathcal{B}(G)$ the parameter space of mixing matrix $\B$ consistent with the DAG $G$. For any two subsets of variables $\Z,\Y\subseteq\X$, we assume that
\begin{equation}\label{Eq:rank_faithfulness}
    \rank(\B_{\Y,\Anc(\Z)}) = \max_{\B'\in \mathcal{B}(G)} \rank(\B'_{\Y,\Anc(\Z)}).
\end{equation}
\end{restatable}
\vspace{-0.5em}
Roughly speaking, we assume there are no edge parameter couplings to produce coincidental low rank. This is slightly stronger than ``$\operatorname{nzcol}(\B_{\Z,:})=\Anc(\Z)$''. See~\cref{app:assumptions} for elaboration. In other words, the graphical criteria holds on a dense open subset of the edge parameter space. \newcontent{Note that~\cref{assum:rank_faithfulness} is the only other parametric assumption we make besides LiNGAM} throughout the paper, where violation of \cref{assum:rank_faithfulness} is of Lebesgue measure 0, and LiNGAM is testable.

Graphically, we first define \textit{vertex cut}, and then give the graphical criteria based on it:
\begin{definition}[Vertex cut]\label{def:vertex_cut}
Let $\W,\Y,\SH$ be three vertices subsets of $V(G)$ which need not be disjoint. We say that $\SH$ is a \textit{vertex cut} from $\W$ to $\Y$, if and only if there exists no directed paths in $G$ from $\W$ to $\Y$ without passing through $\SH$. With basic notations in~\cref{sec:basic_notations}, the following statements are equivalent:\\
\textbf{1)} $\SH$ is a vertex cut from $\W$ to $\Y$; \textbf{2)} $\forall i\in\W,j\in\Y,$ $i\stackrqarrow{\cancel{[\SH]}} j$ does not hold; \textbf{3)} $\Ancout{\SH}(\Y)\cap\W=\emptyset$; \textbf{4)} $\SH$'s removal from $G$ ensures there is no directed paths from $\W\backslash\SH$ to $\Y\backslash\SH$.
\end{definition}\vspace{-0.3em}
More details on vertex cut are in~\cref{app:vertex_cut}. Note that trivially $\W$ itself and $\Y$ itself are vertex cuts.

\begin{restatable}[Graphical criteria of $\Omega_{\Z;\Y}$]{theorem}{GRAPHICALCRITERIA}\label{thm:graph_criteria_omega_ZY}
Let $\Z,\Y$ be two subsets of variables (vertices), we have:
\begin{equation}\label{Eq:graph_criteria_omega_ZY}
    \vert\Y\vert - \dim(\Omega_{\Z;\Y}) = \min \{\vert\SH\vert \ \vert \ \SH\text{ is a vertex cut from}\Anc(\Z)\text{ to }\Y\}.
\end{equation}
where $\dim(\OmegaZY)$ denotes the dimension of the subspace $\OmegaZY$, i.e., the degree of freedom of $\w$.
\end{restatable}
By~\cref{thm:characterization_Omega_ZY}, $\vert\Y\vert - \dim(\Omega_{\Z;\Y})$ is exactly the rank of $\B_{\Y,\Anc(\Z)}$. \newcontent{\cref{thm:graph_criteria_omega_ZY} can then be proved by a combinatorial interpretation of mixing matrices' determinants.} See~\cref{app:proofs} for details.%

From the graphical view, a vertex cut $\SH$ from $\Anc(\Z)$ to $\Y$ means that the causal effect from $\Z$ and the common causes of $\Z$ and $\Y$, if there is any, must affect $\Y$ through $\SH$ - there is no any bypass. To interpret in \textit{trek-separation}~\citep{sullivant2010trek} language, it is equivalent to ``$(\emptyset,\SH)$ \textit{t-separates} $(\Z,\Y)$''\footnote{See Definition 2.7 in~\citep{sullivant2010trek} for definition of \textit{t-separation}. See~\cref{app:vertex_cut} for details.}.\begin{wrapfigure}[12]{r}{0.2\textwidth}\vspace{-0.5em}
\centering
\vspace{-1em}
\begin{tikzpicture}[thick,>=stealth]
  \node[name=n,shape=ellipse,draw=orange,minimum height=0.8em] {\hspace*{1em}\vspace*{0.8em}};
  \node[name=z,below = of n,xshift=-2.9em,yshift=6ex,shape=ellipse,draw=orange!50!white,fill=orange!20!white] {\hspace*{0.8em}$\Z$\hspace*{0.8em}}; %
  \node[name=nz,below = of n,shift={(-2.5em,8.18ex)},shape=ellipse,draw=orange,dashed,rotate=19] {\hspace*{2.1em}\vphantom{\vrule width 0pt height 3.5ex}\hspace*{2.3em}};
  \node[name=s,below = of n,shift={(-1em,0.8ex)},shape=ellipse,draw=red!50!white,fill=red!20!white] {\hspace*{0.4em}$\SH$\hspace*{0.4em}};
  \node[name=y,below = of s,xshift=0.7em,yshift=4.5ex,shape=ellipse,draw=blue!50!white,fill=blue!20!white] {\hspace*{0.9em}$\Y$\hspace*{0.9em}}; %
  \node[name=ey,left = of y,xshift=2.3em,yshift=0.7ex,shape=ellipse,draw=blue,minimum height=1.0em] {};
  \node[name=yey,below = of s,xshift=-0.3em,yshift=5.3ex,shape=ellipse,draw=blue,dashed,rotate=0] {\hspace*{2.3em}\vphantom{\vrule width 0pt height 2.5ex}\hspace*{2.3em}};
 
  \draw [->,orange] (n.190) to [out=190,in=45] (z);
  \draw [->] (n.300) to [out=-60,in=60] (s.45);
  \draw [->] (z.250) to [out=-70,in=140] (s.170);
  \draw [->] (s.200) to [out=-120,in=120] (y.155);
  \draw [->] (s.350) to [bend left=25] (y.55);
  \draw [->,blue] (ey.south east) to [bend right=25] (y.190); %
  \draw [name=nobypass,->,red,dashed,postaction={decorate,decoration={raise=0.5ex,text along path,text align=center,text={|\tiny\color{red}|No such bypass}}}] (nz.340) to [bend left=20] (y.20);
  \node[right = of s,xshift=-2.6em,yshift=1.5ex] {\red{\xmark}};
  \node[above = of nz,xshift=-1.5em,yshift=-9.5ex] {\rotatebox{25}{ \textcolor{orange}{\scriptsize$\Anc(\Z)$}}};
  \node[below = of yey,xshift=0em,yshift=7ex] {\rotatebox{0}{ \textcolor{blue}{\scriptsize$\Ancout{\SH}(\Y)$}}};
\end{tikzpicture}
\vspace{-1.7em}
\caption{Illustration of vertex cut $\SH$ from $\Anc(\Z)$ to $\Y$.}
\label{fig:vertex_cut}
\vspace{-2.5em}
\end{wrapfigure}
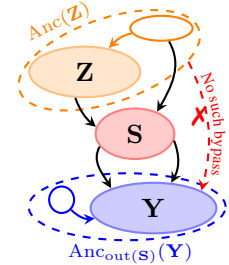

\vspace{-1em}
From the noise view, the noise components that constitute variables $\Z$ are exactly the exogenous noises corresponding to vertices $\Anc(\Z)$. All these noises must contribute to $\Y$ (if any) via $\SH$, and thus $\Y$ can be written as $\Y=L\SH+\mathbf{E'_\Y}$, where $L$ denotes a linear transformation, and $\mathbf{E'_\Y}\indep\Z$. Denote by $\SH^*_{\Z;\Y}$ the \textit{critical vertex cut} from $\Z$ to $\Y$\footnote{See~\cref{app:critical_vertex_cut} for definition. Roughly speaking, \textit{``critical''} means a \textit{smallest} and \textit{last} vertex cut.}, the noise components of $\wTY$ is exactly the exogenous noises corresponding to $\Ancout{\SH^*_{\Z;\Y}}(\Y)$.

\subsection{Formal Definition of \texorpdfstring{$\TIN$}{TEXT}}\label{sec:egin_formal_definition}
With the graphical criteria given in~\cref{sec:graphical_criteria_independent_linear_transformation_subspace}, we could use it for structure inference as long as we have the (dimension of) independent linear transformation subspace $\OmegaZY$. In~\cref{sec:estimation} we will comprehensively discuss methods to estimate $\OmegaZY$, while for now, we could just safely suppose we have $\OmegaZY$: since independence is testable, theoretically one could get $\OmegaZY$ at least by traversing over $\Real^{\vert \Y\vert}$.

\begin{definition}[$\EGIN$ function]\label{def:egin_function}
Let $\mathbf{Z}$ and $\mathbf{Y}$ be two subsets of observed random variables. Suppose variables follow LiNGAM. We define a function $\EGIN$ as follows:
\begin{equation}\label{Eq:EGIN_function_def}
    \EGIN(\Z,\Y)\coloneqq \vert \Y\vert- \operatorname{dim}(\Omega_{\Z;\Y}).
\end{equation}
\end{definition}\vspace{-0.2em}
$\EGIN$ is a function that takes as input two random vectors $\Z,\Y$ and \textit{returns an integer} in range $[0,\vert\Y\vert]$. Note that this is different to $\IN$ or $\GIN$, which returns only a bool value (satisfied or not). $\GIN$ can be viewed as a special case of $\EGIN$, i.e., $\GIN(\Z,\Y)$ is satisfied if and only if $\EGIN(\Z,\Y)=\rank(\cov(\Z,\Y))<\vert \Y \vert$ (where $\rank(\cov(\cdot))$ can be characterized by~\citep{sullivant2010trek}). $\IN$ can also be viewed as a special case of $\EGIN$, i.e., $\operatorname{IN}(\Z,Y_i)$ is satisfied if and only if $\EGIN(\Z,\Z\cup\{Y_i\})=\vert\Z\vert$.

\begin{example}[Review $\EGIN$ graphical criteria on graphs]\label{eg:review_egin_graphical_criteria}For better understanding of $\EGIN$, we demonstrate some representative examples of $\EGIN$ over different graph structures. Results are shown in~\cref{table:egin_examples}. We use the four graph structures in~\cref{fig:graph_examples}. For illustration, assume we have access to latent variables to directly conduct $\EGIN$ over $\Xtilde$. Graphical criteria correspond to $\Gtilde$. For readability, we omit all $\Tilde{ \ }$ notation. We assume there are $5$ vertices in the fully connected DAG (\cref{fig:fully_connected}): consider for example, $\GIN(X_3,X_{1,2,4,5})$ does not hold, since $\rank(\cov(\Z,\Y))$ is only one (restricted by $\Z$ size). However, there exists $\w$ with degree of freedom 1 s.t. $\wTY\indep\Z$ - this corresponds to a 3-variables vertex cut $X_{1,2,3}$ ($\Anc(\Z)$ itself). Consider on \cref{fig:example_d}, $\TIN(X_{1,2},X_{4,5})=1$, corresponding to the 1-variable vertex cut $X_4$. This example shows that vertex cut does not necessarily yield a d-separation pattern (but blocking, instead), since here $X_5\not\indep X_1\vert X_4$.
\end{example}
The graphical criteria for $\EGIN$ over any two sets of variables $\Z,\Y$ are given as above. Interestingly, we find that a special type of $\EGIN$ is particularly simple in form and informative for structure inference:\looseness=-1
\begin{restatable}[Graphical criteria of $\EGIN$ on one-and-others]{lemma}{GRAPHICALCRITERIAONEANDOTHERS}\label{lem:egin_one_and_rest} Assume we have access to all variables $\X$ on a DAG $G$. For each singleton variable $X_i\in\X$, let $\Z\coloneqq \{X_i\}$ and $\Y\coloneqq \X\backslash \{X_i\}$, we have,
\begin{equation}\label{Eq:one_and_rest}
    \TIN(\{X_i\}, \X\backslash \{X_i\}) = \begin{cases}
      \vert \Anc(\{X_i\})\vert     & X_i \text{ is a non-leaf node} \\
      \vert \Anc(\{X_i\})\vert - 1 & X_i \text{ is a leaf node} \\
   \end{cases}
\end{equation}
\end{restatable}
\vspace{-0.5em}
Due to a page limit, here we only give the main criteria of $\EGIN$. See~\cref{app:more_tin_properties} for more properties.\looseness=-1
\vspace{-0.8em}
\begin{table}[t]
  \caption{Examples of $\EGIN$ on different $(\Z,\Y)$ pairs over different graph structures in~\cref{fig:graph_examples}.}
  \label{table:egin_examples}
\vspace{-0.5em}
\scriptsize
  \setlength{\tabcolsep}{3.0pt}
  \renewcommand{\arraystretch}{0.7}
\begin{center}
\resizebox{\textwidth}{!}{
\begin{tabular}{ l *{16}{c} }
\toprule
$(\Z,\Y)$
&
\multicolumn{4}{c}{$(\{X_1, X_2\}, \{X_3, X_4, X_5\})$} &
\multicolumn{4}{c}{$(\{X_1, X_2\}, \{X_4, X_5\})$} &
\multicolumn{4}{c}{$(\{X_3\}, \{X_1, X_2, X_4, X_5\})$} &
\multicolumn{4}{c}{$(\{X_1, X_4\}, \{X_3, X_4, X_5\})$} \\
\cmidrule(lr){2-5} \cmidrule(lr){6-9} \cmidrule(lr){10-13} \cmidrule(lr){14-17}
Graph ID&
(a) & (b) & (c) & (d) &
(a) & (b) & (c) & (d) &
(a) & (b) & (c) & (d) &
(a) & (b) & (c) & (d) \\
\midrule
$\EGIN(\Z,\Y)$ & $1$ & $1$ & $2$ & $1$ & $1$ & $1$ & $2$ & $1$ & $3$ & $2$ & $3$ & $1$ & $2$ & $2$ & $3$ & $2$\\
$\dim(\OmegaZY)$ & $2$ & $2$ & $1$ & $2$ & $1$ & $1$ & $0$ & $1$ & $1$ & $2$ & $1$ & $3$ & $1$ & $1$ & $0$ & $1$\\
$\Anc(\Z)$ & $X_{ 1,2 }$ & $X_{ 1,2,3 }$ & $X_{ 1,2 }$ & $X_{ 1,2 }$ & $X_{ 1,2 }$ & $X_{ 1,2,3 }$ & $X_{ 1,2 }$ & $X_{ 1,2 }$ & $X_{ 1,2,3 }$ & $X_{ 1,3 }$ & $X_{ 1,2,3 }$ & $X_{ 3 }$ & $X_{ 1,2,3,4 }$ & $X_{ 1,3,4 }$ & $X_{ 1,2,3,4 }$ & $X_{ 1,2,3,4 }$\\
$\SH^*_{\Z,\Y}$ & $X_{ 3 }$ & $X_{ 3 }$ & $X_{ 1,2 }$ & $X_{ 4 }$ & $X_{ 4 }$ & $X_{ 4 }$ & $X_{ 4,5 }$ & $X_{ 4 }$ & $X_{ 1,2,4 }$ & $X_{ 1,3 }$ & $X_{ 1,2,3 }$ & $X_{ 3 }$ & $X_{ 3,4 }$ & $X_{ 3,4 }$ & $X_{ 3,4,5 }$ & $X_{ 3,4 }$\\
$\operatorname{A}_{\operatorname{o}(\SH^*)}(\Y)$ & $X_{ 4,5 }$ & $X_{ 4,5 }$ & $X_{ 3,4,5 }$ & $X_{ 3,5 }$ & $X_{ 5 }$ & $X_{ 5 }$ & $\emptyset$ & $X_{ 3,5 }$ & $X_{ 5 }$ & $X_{ 2,4,5 }$ & $X_{ 4,5 }$ & $X_{ 1,2,4,5 }$ & $X_{ 5 }$ & $X_{ 5 }$ & $\emptyset$ & $X_{ 5 }$\\
\bottomrule
\end{tabular}}
\end{center}
\vspace{-2.5em}
\end{table}

\section{\texorpdfstring{$\TIN$}{TEXT} Condition-Based Method for Measurement Error Models}
\label{sec:exploit}
\vspace{-0.8em}
In~\cref{sec:approach} we propose $\EGIN$ condition over general LiNGAM model and give its graphical criteria. In this section, we aim to exploit $\EGIN$ on our specific task of interest: measurement error models.\begin{wrapfigure}[12]{r}{0.20\textwidth}
\centering
\begin{tikzpicture}[->,>=stealth,shorten >=1pt,auto,node distance=0.8cm,
semithick]

\node[state, fill=red!90, inner sep=0.1pt, minimum size=0.5cm] at (2,4) (1) {\scriptsize $\Tilde{X}_1$};
\node[state, fill=red!60, inner sep=0.1pt, minimum size=0.5cm] at (2,3.25) (2) {\scriptsize $\Tilde{X}_2$};
\node[state, fill=red!30, inner sep=0.1pt, minimum size=0.5cm] at (2,2.5) (3) {\scriptsize $\Tilde{X}_3$};
\node[state, fill=red!30, inner sep=0.1pt, minimum size=0.5cm] at (1.2,2) (6) {\scriptsize $\Tilde{X}_6$};
\node[state, fill=red!3, inner sep=0.1pt, minimum size=0.5cm] at (2.8,2) (4) {\scriptsize $\Tilde{X}_4$};
\node[state, fill=red!3, inner sep=0.1pt, minimum size=0.5cm] at (2, 1.7) (5) {\scriptsize $\Tilde{X}_5$};
\node[state, fill=red!60, inner sep=0.1pt, minimum size=0.5cm] at (2.8, 2.75) (7) {\scriptsize $\Tilde{X}_7$};

\path (1) edge node {} (2)
(2) edge node {} (3)
(2) edge node {} (6)
(2) edge node {} (7)
(3) edge node {} (4)
(3) edge node {} (5)
(3) edge node {} (6)
(4) edge node {} (5);
\end{tikzpicture}
\caption{An example of ordered group decomposition.}
\label{fig:ordered_group_decomp_example}
\vspace{-2em}
\end{wrapfigure}
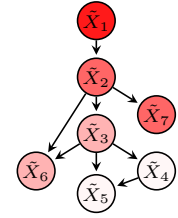
\vspace{-0.8em}
\subsection{Identifiability of Ordered Group Decomposition}\label{sec:ordered_group_decomposition}\vspace{-0.2em}
Under our problem setting where $\Xtilde$ follows LiNGAM, identifiability results can greatly benefit from the non-Gaussianity of data. \citep{zhang2018causal} shows that the \textit{ordered group decomposition} of $\Gtilde$ is identifiable. First review definitions:
\begin{definition}[Pure leaf child]\label{def:pure_leaf_child}
On a DAG $G$, a vertex $j$ is said to be a ``pure leaf child'' of another vertex $i$, iff $j$ is a leaf node with only one parent, $i$.
\end{definition}\vspace{-0.3em}
Particularly, if a variable $\smash{\Xt_j}$ is a pure leaf child of $\smash{\Xt_i}$ in $\smash{\Gtilde}$, then $\smash{\Xt_j}$ and $\smash{\Xt_j}$ are naturally unidentifiable (e.g., $\Xt_2$ and $\Xt_7$ in~\cref{fig:ordered_group_decomp_example}). The reason is that the exogenous noise $\smash{\Tilde{E}_j}$ of leaf $\smash{\Tilde{X}_j}$ does not contribute to any other variables, just like its measurement noise $E_j$. Consequently, $X_i$ and $X_j$ can be viewed as two equally positioned measurements of $\smash{\Tilde{X}_i}$, without any asymmetry.
\begin{definition}[Ordered group decomposition]\label{def:ordered_group_decomposition}
Consider the underlying causal model $\Gtilde$. The ordered group decomposition can be defined by the following procedure: at each step, remove root vertices and their pure leaf children nodes (if any) from graph, and append the removed ones as a new group. Repeat this procedure to remove root vertices from the remaining graph, until the graph is empty.
\end{definition}
This procedure is equivalent to Definition 2 in~\citep{zhang2018causal}\footnote{There is actually trivial difference, depending on how case with multiple roots is considered. See~\cref{app:multiple_subroots_ordered_group_decomp}.}. See~\cref{fig:ordered_group_decomp_example}: roots are removed from graph in order (color dark to light). The ordered group decomposition is $\Xt_1 \rightarrow \Xt_{2,7} \rightarrow \Xt_{3,6} \rightarrow \Xt_{4,5}$.
\subsection{Exploit \texorpdfstring{$\TIN$}{TEXT} to Identify Ordered Group Decomposition}
Our task is to recover $\Gtilde$ over $\Xtilde$ by testing $\TIN$ over $\X$. In notations above, $\Xtilde$ and $\X$ are strictly distinguished to denote unobserved and observed variables. However, actually this can be escaped:
\begin{restatable}[Equivalence of $\EGIN$ over latent and observed variables]{theorem}{EQUIVALENCEOVERLATENTOBSERVED}\label{thm:equiv_gin_true_observed}For two \textbf{disjoint} observed variables subsets $\Z,\Y$, and their respective underlying latent variables subsets $\Tilde{\Z},\Tilde{\Y}$,\vspace{-0.2em}
\begin{equation}\label{Eq:equiv_gin_true_observed}\EGIN(\Z,\Y) = \EGIN(\Tilde{\Z},\Tilde{\Y}).
\end{equation} %
\end{restatable}\vspace{-0.6em}
\cref{thm:equiv_gin_true_observed} can either be proved by graphical criteria, or by showing how the rank of submatrices of $\B$ is preserved among latent and observed variables. Interestingly, recall \cref{prop:rare_dsep}, we show the inequivalence of d-separation and $\IN$ condition held between $\Xtilde$ and $\X$ (\textit{raw independence}). However, $\TIN$ condition, which essentially \textit{finds transformed independence}, holds equivalently on latent and observed variables. With this equivalence, we can conduct $\TIN$ over observed variables $\X$ \textit{just as if} we have access to the latent ones $\Xtilde$. Thus the problem can be restated without measurement error: assuming causal sufficiency, by only using $\TIN$ over disjoint $\Z,\Y$, to what extent is $G$ identifiable?\looseness=-1

Under this equivalent problem, \cref{lem:egin_one_and_rest} can be used directly to identify the ordered group decomposition of $\Gtilde$: for each singleton observed variable $X_i$, test $\TIN$ and assign an order $\texttt{ord}(X_i)\coloneqq \TIN(\{X_i\},\X\backslash\{X_i\})$. Group variables with same $\texttt{ord}$, and then sort the groups by their orders. Obviously, the ordered groups obtained by this procedure is consistent with~\cref{def:ordered_group_decomposition}.

\begin{example}[$\TIN$ on~\cref{fig:ordered_group_decomp_example}]\label{eg:use_tin_to_find_ordered_group_decomposition}
$\texttt{ord}(X_i)_{i=1}^7$ are respectively $1,2,3,4,4,3,2$ (can verify by characterization or graphically), so the group ordering is identified as $\Xt_1 \rightarrow \Xt_{2,7} \rightarrow \Xt_{3,6} \rightarrow \Xt_{4,5}$.
\end{example}

\section{Estimating Linear Independent Transformation Subspace \texorpdfstring{$\OmegaZY$}{TEXT}}
\label{sec:estimation}
\vspace{-0.5em}
In the above sections we safely assume that we can always get $\OmegaZY$, since theoretically independence is testable and $\w$ can be exhaustively traversed. In this section, we give practical methods to estimate $\OmegaZY$. Due to page limit we only give a summary for each. Please see~\cref{app:estimate_omegaZY} for details.

\subsection{Tackling Down to Subsets of \texorpdfstring{$\Y$}{TEXT}}\label{sec:tackle_down_Y_subsets}
\begin{restatable}[$\TIN$ over $\Y$ subsets]{theorem}{TINOVERYSUBSETS}\label{thm:equiv_gin_Y_subsets}For two variables sets $\Z,\Y$, $\EGIN(\Z,\Y)=k$ (assume $k>0$), iff the following two conditions hold: \textbf{1)} $\forall \Y'\subseteq \Y$ with $\vert \Y' \vert = k+1$ (if any), there exists non-zero $\w$ s.t. $\w^\T\Y'\indep\Z$; and \textbf{2)} $\exists \Y'\subseteq \Y$ with $\vert \Y' \vert = k$, there exists no non-zero $\w$ s.t. $\w^\T\Y'\indep\Z$.
\end{restatable}\vspace{-0.5em}
This transforms the task of estimating the \textit{dimension} of $\OmegaZY$ to a simpler one: counting \textit{size} of the subsets $\Y'$. Instead of \textit{all} independence, here we only need to check \textit{existence} of independence.%

\subsection{Constrained Independent Subspace Analysis (ISA)}\label{sec:ISA_for_estimation}Conduct Independent Subspace Analysis (ISA~\citep{theis2006towards}) over variables $\Z$ and $\Y$ in the following form:
\vspace{-0.2em}
\begin{equation}\label{Eq:ISA}
\mathbf{s}= \begin{bmatrix} 
    \mathbf{I} & \mathbf{0} \\
    \mathbf{0} & \mathbf{W}_{\Y\Y}
\end{bmatrix} \begin{bmatrix} 
   \Z \\ \Y
\end{bmatrix},
\end{equation}
\vspace{-1.1em}

where the de-mixing matrix is masked to only update the lower-right $|\Y|\times |\Y|$ block $\mathbf{W}_{\Y\Y}$, with upper-left $|\Z|\times |\Z|$ block fixed as the identity and elsewhere fixed as zero. The independence between $\mathbf{W}_{\Y\Y} \Y$ as a group to $\Z$ is maximized. Since $\mathbf{W}_{\Y\Y}$ is invertible, its rows span the whole $\mathbb{R}^{|\Y|}$, so the maximum number of rows that achieves ${\mathbf{W}_{\Y\Y}}_{i,:}^\intercal\Y \indep \Z$ is exactly the dimension of $\Omega_{\Z;\Y}$.

\subsection{Stacked Cumulants: Ranks Stopped Increasing}\label{sec:stacked_cumulants_ranks_stop}
\begin{definition}[Stacked 2D slices of cumulants]\label{def:stacked_2d_slices_of_cumulants} For two variables sets $\Z,\Y$ and order $k\geq 2$, define:
\begin{equation}\label{Eq:stacked_2d_slices_of_cumulants}
    \Psi_{\Z;\Y}^{(k)} \coloneqq {\begin{bmatrix} 
    {\mathcal{C}_{\Z,\Y}^{{(2)}^\T}} & \cdots & {\mathcal{C}_{\Z,\Y}^{{(k)}^\T}}
    \end{bmatrix}}^\T, \text{ where }\mathcal{C}^{(k)}\text{ is }\text{matrix with }{\mathcal{C}_{i, j}^{(k)}} \coloneqq \operatorname{cum}(\underbrace{X_{i}, \cdots, X_{i}}_{k-1\text{ times}}, X_j).
\end{equation}
\end{definition}\vspace{-1.3em}
$\Psi_{\Z;\Y}^{(k)}$ is a $(k-1)|\Z|\times |\Y|$ matrix that vertically stacks 2D cumulants slices between $\Z,\Y$ with orders from 2 to $k$. When $k=2$ it is $\cov(\Z,\Y)$. Ranks of $\Psi_{\Z;\Y}^{(k)}$ in the sequence $k=2,3,...$ satisfies:
\vspace{-1em}
\begin{restatable}[$\rank(\Psi_{\Z;\Y}^{(k)})$ stopped increasing]{theorem}{RANKSTOPINCREASING}\label{thm:rank_stopped_increasing}
For two variables sets $\Z,\Y$, there exists a finite order $k\geq 2$ s.t. $\operatorname{rank}(\Psi_{\Z;\Y}^{(k+1)}) = \operatorname{rank}(\Psi_{\Z;\Y}^{(k)}).$ Moreover, $\TIN(\Z,\Y)$ equals $\operatorname{rank}(\Psi_{\Z;\Y}^{(k)})$ with this $k$.
\end{restatable}
\vspace{-0.5em}
We will show detailed characterization and graphical criteria of $\Psi_{\Z;\Y}^{(k)}$ in~\cref{app:stacked_cumulants}.
\vspace{0.2em}
\subsection{\texorpdfstring{$\TIN$}{TEXT} in Two Steps: Solve Equations, and then Test for Independence}\label{sec:two_steps_estimation}
Motivation of~\cref{sec:stacked_cumulants_ranks_stop} is that independence yields zero cumulants (not only 2nd-order uncorrelatedness), i.e., $\Omega_{\Z;\Y}\subseteq \operatorname{null}(\Psi_{\Z;\Y}^{(k)})$, for any $k\geq 2$. Hence, we could solve equations introduced at each order $k$ and then check whether \textit{all} solution $\w\in\operatorname{null}(\Psi_{\Z;\Y}^{(k)})$ makes $\wTY\indep\Z$ (similar to $\GIN$ procedure). More generally, $\Omega_{\Z;\Y}\subseteq \operatorname{null}(\cov(f(\Z), \Y))$ for any real-valued function $f(\cdot)$. E.g., solve equations system $\{\cov(\log(\Z^2), \Y)\w = \mathbf{0}; \ \cov(\sin(\Z), \Y)\w = \mathbf{0};\cdotsshort\}$ and test whether $\wTY\indep\Z$ holds.
\vspace{-0.5em}
\begin{figure}[t]
\begin{minipage}{0.01\textwidth}\raggedright
\rotatebox[origin=c]{90}{\tiny{Fully connected DAG \ \ \ \ \ \ \ }}\\
\end{minipage}
\noindent\begin{minipage}{0.98\textwidth}
\includegraphics[width=1\linewidth, ]{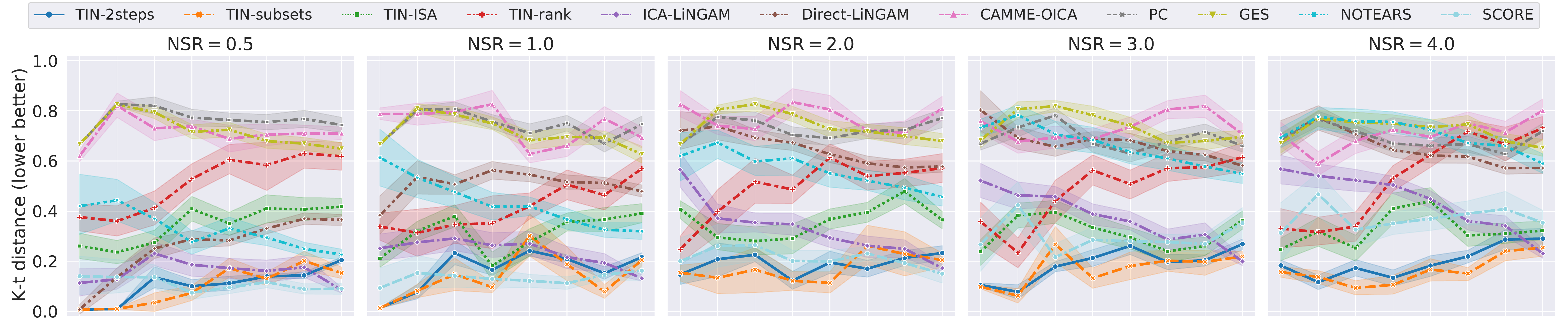}
\end{minipage}
\hfill%
\begin{minipage}{0.01\textwidth}\raggedright
\rotatebox[origin=c]{90}{\tiny{\ \ \ \ \ \ \ \ Chain structure}}\\
\end{minipage}
\noindent\begin{minipage}{0.98\textwidth}
\includegraphics[width=1\linewidth, ]{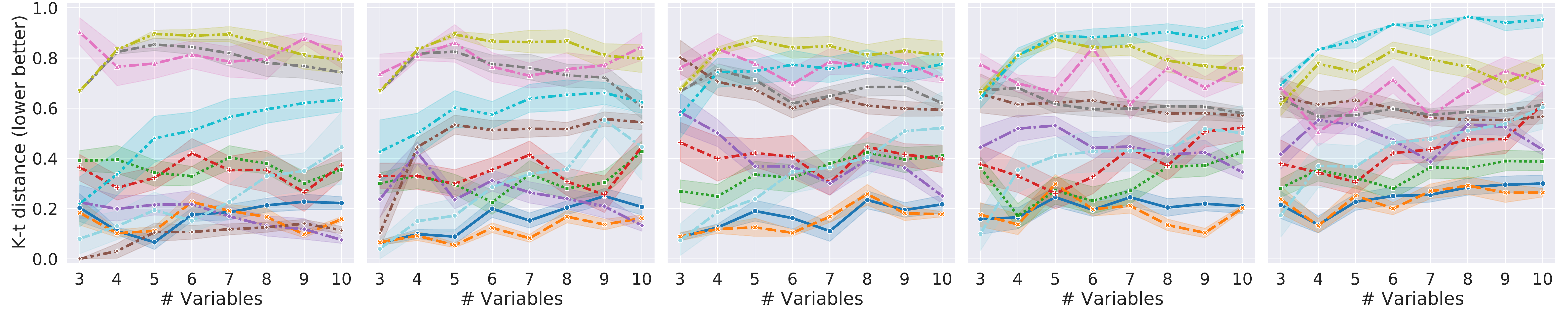}
\end{minipage}
\centering\vspace{-0.5em}
\caption{Distance to truth group ordering (lower better). 4 implementations of $\TIN$ $+$ 7 competitors.}
\label{fig:experiments_fully_chain}
\end{figure}

\vspace{0.5em}
\section{Experimental Results}
\label{sec:experiments}
\vspace{-0.2em}
In this section, we evaluate the performance of $\TIN$ in recovering the ordered group decomposition of $\Gtilde$ under measurement error. Specifically, four implementations of $\TIN$ are evaluated: $\operatorname{TIN-subsets}$(\cref{sec:tackle_down_Y_subsets}), $\operatorname{TIN-ISA}$(\cref{sec:ISA_for_estimation}), $\operatorname{TIN-rank}$(\cref{sec:stacked_cumulants_ranks_stop}), and $\operatorname{TIN-2steps}$(\cref{sec:two_steps_estimation}). We compare our method with PC~\citep{spirtes1991algorithm}, GES~\citep{chickering2002optimal}, ICA-LiNGAM~\citep{shimizu2006linear}, Direct-LiNGAM~\citep{shimizu2011directlingam}, CAMME-OICA~\citep{zhang2018causal}, \newcontent{NOTEARS~\citep{zheng2018dags}, and SCORE~\citep{rolland2022score}}. Experiments are conducted on both synthetic and real-world data.
\vspace{0.1em}
\subsection{Synthetic Data}\label{sec:synthetic_data}
While our proposed method outputs correct ordered group decomposition for any $\Gtilde$ (without assumptions on graph structure), in the following simulation we consider specifically two cases: fully connected DAG (\cref{fig:fully_connected}) and chain structure (\cref{fig:chain}), of which the ordered group decomposition are both $\Xt_1\veryshortarrow\cdotsshort\veryshortarrow\Xt_{n-2},\Xt_{n-1,n}$. We consider $\Gtilde$ with the number of vertices $n=3,\cdotsshort,10$. Edges weights (i.e., the nonzero entries of matrix $\A$) are drawn uniformly from $[-0.9,-0.5]\cup[0.5,0.9]$. Exogenous noises $\Etilde$ are sampled from uniform $\cup [0,1]$ to the power of $c$, $c\sim\cup[5,7]$, and measurement errors are sampled from Gaussian $\mathcal{N}(0,1)$ to the power of $c$, $c\sim\cup[2,4]$. Sample size is $5,000$. Observations are generated by $X_i=\Xt_i+E_i$. To show the effect of measurement error, we simulate with noise-to-signal ratio $\operatorname{NSR}\coloneqq \var(E_i)/\var(\Xt_i)$ in $\{0.5,1,2,3,4\}$. To evaluate the output group ordering, we use \textit{Kendall tau distance}~\citep{kendall1938new} to the ground-truth (in range $[0,1]$, the lower the better). For algorithms returning DAG/PDAG, its ordering is first extracted according to~\cref{def:ordered_group_decomposition}.
\cref{fig:experiments_fully_chain} shows the results. Each column subplot indicates an $\operatorname{NSR}$ setting. The error bar is from 50 random generated instances. We can see that $\operatorname{TIN-subsets}$ and $\operatorname{TIN-2steps}$ are two best methods (ranks first on $68/80$ cases), which steadily stick near x-axis. They both find/check independence by solving equations $\cov(f(\Z),\Y)\w=\mathbf{0}$. $\operatorname{TIN-ISA}$ is slightly inferior (though still ranks 3rd), maybe due to the bias to independence (of $\w$ vectors) by ISA. $\operatorname{TIN-rank}$ fluctuates and performs worst (among $\TIN$s), especially on the fully connected dense case. This might be explained by higher order cumulants' sensitivity to outliers, and unreliable numerical rank tests (we simply use SVD and thresholding). Among the competitors, \newcontent{ICA-LiNGAM and SCORE are the strongest two, which remain relatively stable with $\operatorname{NSR}$ growing larger} (while e.g., Direct-LiNGAM deteriorates rapidly). Interestingly, they both perform much better on fully connected DAGs than on chain structures, and ICA-LiNGAM even performs better on larger graphs than on smaller ones. We will investigate the reason. Consider CAMME-OICA, it relies heavily on mixing matrix parameter initialization, and thus performs generally weak. Weak results by PC, GES might be because of the unfair setting for them: they output CPDAG and in both cases (fully connected and chain) here, naturally no ordering (directions) can be determined. More experimental configuration and results discussions are in~\cref{app:evaluation_details}. \looseness=-1 %
\vspace{0.3em}
\subsection{Real-World Data}\label{sec:real_world_data}
Sachs's~\citep{sachs2005causal} is a real dataset that measures the expression levels of proteins in human cells under various phospholipids. The ground-truth graph structure~\citep{sachs2005causal} contains $17$ edges on $11$ variables (cell types), of which the ordered group decomposition is $\textrm{\{plc,pkc\}}$$\veryshortarrow$$\textrm{\{pip2,pip3,pka,p38,jnk\}}$$\veryshortarrow$$\textrm{\{raf\}}$$\veryshortarrow$ $\textrm{\{mek\}}$$\veryshortarrow$$\textrm{\{erk,akt\}}$. Result given by $\operatorname{TIN-subsets}$ achieves the best distance score $0.33$: $\textrm{\{plc,pkc,p38\}}$$\veryshortarrow$ $\textrm{\{raf,mek,pip2,pka,jnk\}}$$\veryshortarrow$$\textrm{\{pip3\}}$$\veryshortarrow$$\textrm{\{erk,akt\}}$. This score on PC, GES, ICA-LiNGAM are $0.49$, $0.69$, $0.8$ respectively. E.g., PC outputs $\textrm{\{raf,pka\}}$$\veryshortarrow$$\textrm{\{mek\}}$$\veryshortarrow$$\textrm{\{plc,pip2,pip3,pkc,p38,jnk\}}$$\veryshortarrow$$\textrm{\{erk,akt\}}$. \newcontent{Experiments on another dataset, Teacher Burnout~\citep{byrne2013structural}, also show $\TIN$'s good performance. See~\cref{app:experiments_on_teacher_burnout}.}

\section{Conclusion and Discussions}
\label{sec:discussion}
\vspace{-0.5em}
In this work we define the \textbf{T}ransformed \textbf{I}ndependent \textbf{N}oise ($\TIN$) condition based on LiNGAM causal model, which finds and checks for independence between a specific linear transformation of some variables and others. We provide graphical criteria of $\TIN$, which might further improve identifiability of the latent-variable problem. Specifically on causal discovery under measurement error, we exploit $\TIN$ to achieve identifiability of ordered group decomposition.%

We summarize the future work as three fold: \textbf{1)} For the measurement error model, in addition to the special type (one-over-others), $\TIN$ over general $\Z,\Y$ pairs can further improve identifiability. See~\cref{app:more_than_ordered_group_decomp}; \textbf{2)} $\TIN$ now only considers the dimension of $\OmegaZY$, while parameters might also help to recover edges weights. See~\cref{app:more_than_dimension}; and \textbf{3)} Reliable estimation of $\OmegaZY$ can be formulated as an orthogonal research problem. We believe there exists more solutions. See~\cref{app:more_possible_solution_estimate_OmegaZY}.

\begin{ack}
The authors would like to thank Georges Darmois, Viktor Skitovich, Bernt Lindstr\"om, Ira Gessel, G\'erard Viennot, and Seth Sullivant for initializing the beautiful theorems that this work is built upon. Thank Feng Xie, Ruichu Cai, Biwei Huang, Clark Glymour, and Zhifeng Hao for the $\GIN$ condition. Thanks to the anonymous reviewers and Joseph Ramsey, Zeyu Tang, Yujia Zheng, Ignavier Ng, Justin Ding, Mengyao Lu, Haoqin Tu, Qiyu Wu, Muyang Li, Jinkun Cao, and Jinhao Zhu for helpful feedback, proofreading, and discussions. The work was partially supported by the NSF under Project Number A221500S001, by the NSF-Convergence Accelerator Track-D award \#2134901, by NIH-NHLB1 9R01HL159805-05A1, by a grant from Apple Inc., and by a grant from KDDI Research Inc.\looseness=-1 
\end{ack}

\small
\bibliographystyle{plainnat}
\bibliography{references}

\newpage
\appendix

\numberwithin{equation}{section}

\normalsize

\section{Proofs of Main Results}\label{app:proofs}
\subsection{Proof of~\cref{prop:rare_dsep}}
\RAREDSEP*
\begin{proof}[Proof of~\cref{prop:rare_dsep}]\label{prf:prop_rare_dsep}
The whole graph we consider is the graph $\Gtilde$ among latent variables $\Xtilde$ and measurement edges $\Xt_i \rightarrow X_i$. Consider observed variables $X_i, X_j$ and subset $\SH$. Denote by ``$\operatorname{des}$'' the descendants of some vertices on graph. By definition of d-separation, if $X_i \indep_{\mkern-9.5mu d} X_j \vert \SH$, then for every undirected path $p$ linking $X_i$ and $X_j$ (if there is any), $p$ is blocked by $\SH$, i.e., \textit{either 1)} there exists a collider $W$, s.t. $W\not\in\Z$ and $\operatorname{des}(W)\cap \SH = \emptyset$, \textit{or 2)} there exists a non-collider $W$ s.t. $W\in\SH$. Since observed variables are all leaf nodes of their respective latent nodes, for every undirected path $p$ linking $X_i$ and $X_j$, $p$ must be in form of $X_i - \Xt_i - \cdots - \Xt_j - X_j$. Hence, $W$ must be in latent nodes, and only \textit{case 1)} is possible, which means that there exists a collider on every $p$ linking $X_i$ and $X_j$, and thus $X_i$ and $X_j$ is also d-separated by $\emptyset$ (conclusion 1). Specifically on \textit{case 1)}, $W\not\in\SH$ is obvious (since $W\in\Xtilde$ and $\SH\in\X$). And, by $\operatorname{des}(W)\cap \SH = \emptyset$, we have $W\not\in \Tilde{\SH}$ and $\operatorname{des}(W)\cap \Tilde{\SH} = \emptyset$ (easy to show since $\SH\subseteq \operatorname{des}(\Tilde{\SH})$), and thus among latent nodes, there is $\Xt_i \indep_{\mkern-9.5mu d} \Xt_j \vert \Tilde{\SH}$ (conclusion 2). Combining conclusion 1 and 2, let $\SH\coloneqq\emptyset$, we further have marginally $\Xt_i \indep_{\mkern-9.5mu d} \Xt_j$ holds.
\end{proof}
\begin{remark} Roughly speaking, the d-separation patterns among $\Xtilde$ usually do not hold among $\X$ (except for \textit{rare} marginal ones), since the observed variables are all leaf nodes, which are not causes of any other (though the latent variables they intend to measure might be). By a similar proof procedure, we shall have a full version of~\cref{prop:rare_dsep}: for observed variables $X_i, X_j$ and subset $\SH$,
\begin{itemize}[noitemsep,topsep=-3pt]
\item[1.] If $X_i \indep_{\mkern-9.5mu d} X_j \vert \SH$, then among latent variables, marginally $\Xt_i \indep_{\mkern-9.5mu d} \Xt_j$, and $\Xt_i \indep_{\mkern-9.5mu d} \Xt_j \vert \Tilde{\SH}$ holds.
\item[2.] If $X_i\not\indep_{\mkern-9.5mu d} X_j \vert \SH$, which means that there exists a path $p$ from $X_i$ to $X_j$ unblocked by $\mathbf{Z}$. And,
\begin{itemize}[noitemsep,topsep=-1pt]
\item[a)] If on $p$ there is a non-collider $W\in\Tilde{\SH}$, then $\Xt_i \indep_{\mkern-9.5mu d} \Xt_j \vert \Tilde{\SH}$;
\item[b)] Otherwise (on $p$ there is no non-collider $W\in\Tilde{\SH}$), $\Xt_i \not\indep_{\mkern-9.5mu d} \Xt_j \vert \Tilde{\SH}$.
\end{itemize}
\end{itemize}
\end{remark}

\subsection{Proof of~\cref{thm:characterization_Omega_ZY}}
\CHARACTERIZATIONOMEGAZY*
\begin{proof}[Proof of~\cref{thm:characterization_Omega_ZY}]\label{prf:thm_characterization_Omega_ZY}
We write variables in terms of linear combination of exogenous noises, $\X=\B\E$. For variables set $\Z = \B_{\Z,:}\E = \B_{\Z,\operatorname{nzcol}(\B_{\Z,:})}\E_{\operatorname{nzcol}(\B_{\Z,:})}$, where $\operatorname{nzcol}(\B_{\Z,:})$ denotes the column indices where the submatrix $\B_{\Z,:}$ has non-zero entries, i.e., $\Z$ contains and only contains noise terms $\E_{\operatorname{nzcol}(\B_{\Z,:})}$. For a vector $\w \in \mathbb{R}^{\vert \Y \vert}$, $\wTY = \w^\intercal \B_{\Y,:}\E = \w^\intercal \B_{\Y,\operatorname{nzcol}(\B_{\Z,:})}\E_{\operatorname{nzcol}(\B_{\Z,:})} + \w^\intercal\B_{\Y,\sim\operatorname{nzcol}(\B_{\Z,:})}\E_{\sim\operatorname{nzcol}(\B_{\Z,:})}$ (``$\sim$'' denotes complement set). By the $\darmois$~\citep{kagan1973characterization}, $\wTY\indep\Z$ if and only if $\wTY$ shares no common non-Gaussian noise terms with $\Z$, i.e., $\w^\intercal \B_{\Y,\operatorname{nzcol}(\B_{\Z,:})} = 0$. Moreover, if assuming ``if $i\rightsquigarrow j$ then $\B_{j,i}\neq 0$'' (a weaker faithfulness assumption, see~\cref{app:assumptions}), then, $\operatorname{nzcol}(\B_{\Z,:})=\Anc(\Z)$, i.e., variables set $\Z$ contains and only contains exogenous noises w.r.t. its ancestors set.
\end{proof}

\subsection{Proof of~\cref{thm:graph_criteria_omega_ZY}}
\GRAPHICALCRITERIA*
To prove~\cref{thm:graph_criteria_omega_ZY} we mainly use the Lindstr{\"o}m-Gessel-Viennot theorem~\citep{lindstrom1973vector,gessel1985binomial} in algebraic combinatorics, which gives a combinatorial interpretation of the determinants of certain matrices:
\begin{theorem}[Lindstr{\"o}m-Gessel-Viennot theorem~\citep{lindstrom1973vector,gessel1985binomial}]\label{thm:lindstrom_gessel_viennot}
Let $G$ be a directed acyclic graph with vertex set $[n]$. Each directed edge $i\rightarrow j$ is assigned with a weight $e(i,j)$. For each directed path $P$ from vertices $i$ to $j$, let $\operatorname{wt}(P)\coloneqq \Pi_{m\rightarrow l \in P}{e(m,l)}$, the product of the weights of the edges of the path. For any two vertices $i,j$, denote $\mathcal{P}(i,j)$ the set of all directed paths from $i$ to $j$. Write an $n\times n$ matrix $M$, with entries defined as $M_{i,j} = \sum_{P\in \mathcal{P}(i,j)}\operatorname{wt}(P)$, the sum of path weights over all paths from $i$ to $j$. For two subsets $S,T\subseteq [n]$ with $\vert S \vert = \vert T \vert = k$ (letters ``$S$'' means source and ``$T$'' means sink), we have:
\begin{equation}\label{Eq:detM_lindstrom}
    \det(M_{S,T}) = \sum_{\mathbf{P}=(P_1,\ldots,P_k) \colon S \to T} \mathrm{sign}(\sigma(\mathbf{P})) \prod_{i=1}^k \operatorname{wt}(P_i).
\end{equation}
where the sum is taken over all $k$-tuples $\mathbf{P}=(P_1,\ldots,P_k)$ of non-intersecting paths from $S$ to $T$, and $\sigma(\mathbf{P})$ is the sign of the corresponding permutation of elements in $\mathbf{P}$. ``non-intersecting'' means that for any two paths $P_i, P_j \in \mathbf{P}$ with $i\neq j$, $P_i$ and $P_j$ have no two vertices in common (not even endpoints). In particular, $\det(M_{S,T})=0$ if and only if there exists no such $k$-tuple non-intersecting paths, i.e., for every system of $k$ paths from $S$ to $T$, there exists two paths that share a vertex.
\end{theorem}

Based on~\cref{thm:lindstrom_gessel_viennot}, we can readily give proof to~\cref{thm:graph_criteria_omega_ZY}. Note that in our setting where $\X=\A\X+\E=\B\E$ with $\B=(\I-\A)^{-1}$, we know that $\A_{j,i}$ is the $e(i,j)$ above, and $\B$ is exactly $M^\intercal$, with $\B_{j,i}=\sum_{P\in\mathcal{P}(i,j)}{\prod_{k\to l\in P}\A_{l,k}}$, the total causal effect from $i$ to $j$.

\begin{proof}[Proof of~\cref{thm:graph_criteria_omega_ZY}]\label{prf:thm_graph_criteria_omega_ZY}
From~\cref{thm:characterization_Omega_ZY,assum:rank_faithfulness} and the rank-nullity theorem, $\vert \Y \vert - \dim (\OmegaZY)$ is equal to $\rank(\B_{\Y,\Anc(\Z)})$. By the max-flow min-cut theorem (vertex version, known as Menger's theorem)~\citep{dantzig1955max,bondy1976graph,menger1927allgemeinen}, the maximum amount of non-intersecting paths from source to sink is equal to the size of the minimum vertex cut from source to sink. Hence, if the minimum vertex cut from $\Anc(\Z)$ to $\Y$ is of size $k$, then there exists a $k$-tuples of non-intersecting paths from some subset of $\Anc(\Z)$ to some subset of $\Y$, and this is the largest possible non-intersecting paths system from $\Anc(\Z)$ to $\Y$. By~\cref{thm:lindstrom_gessel_viennot} and~\cref{assum:rank_faithfulness} (no parameter coupling to make coincidental low rank), this means that all $(k+1)\times (k+1)$ minors of $\B_{\Y,\Anc(\Z)}$ is zero and at least one $k\times k$ minor of $\B_{\Y,\Anc(\Z)}$ is non-zero. Hence $\rank(\B_{\Y,\Anc(\Z)})=k$.
\end{proof}

Interestingly, we find that our defined vertex cut has connection with trek-separation~\citep{sullivant2010trek}, i.e. ``$\SH$ is a vertex cut from $\Anc(\Z)$ to $\Y$'' is equivalent to ``$(\emptyset, \SH)$ t-separates $(\Z, \Y)$'' (see~\cref{app:vertex_cut}). Trek-separation theorem states that:
\begin{theorem}[Trek-separation for directed graphical models, Theorem 2.8 in~\citep{sullivant2010trek}]\label{thm:trek_separation}
For two vertices sets $\W,\Y$, the variance-covariance matrix $\cov(\W,\Y)$ has rank less than or or equal to $k$ for all covariance matrices consistent with the graph $G$ if and only if there exists subsets $\SH_\W,\SH_\Y \subseteq V(G)$ with $\vert \SH_\W \vert + \vert \SH_\Y \vert \leq k$ such that $(\SH_\W,\SH_\Y)$ t-separates $(\W,\Y)$. Consequently,
\begin{equation}\label{Eq:trek_separation_cov}
    \rank(\cov(\W,\Y)) \leq \min \{ \vert \SH_\W \vert + \vert \SH_\Y \vert \ \vert \ (\SH_\W,\SH_\Y) \text{ t-separates } (\W,\Y)\}
\end{equation}
and equality holds for generic covariance matrices (i.e., no coincidental low rank in variance-covariance matrix) consistent with $G$.
\end{theorem}

We now show that~\cref{thm:graph_criteria_omega_ZY} can also be proved by trek-separation theorem:
\begin{proof}[Proof of~\cref{thm:graph_criteria_omega_ZY} (another version)]\label{prf:thm_graph_criteria_omega_ZY_version2}
From~\cref{thm:characterization_Omega_ZY,assum:rank_faithfulness} and the rank-nullity theorem, $\vert \Y \vert - \dim (\OmegaZY)$ is equal to $\rank(\B_{\Y,\Anc(\Z)})$. Then, what is $\rank(\B_{\Y,\Anc(\Z)})$? Let us consider two variables sets $\Anc(\Z)$ and $\Y$ and their respective variance-covariance matrix. We write $\Anc(\Z)$ as mixed noise components $\Anc(\Z) = \B_{\Anc(\Z),\Anc(\Z)}\E_{\Anc(\Z)}$, where $\B_{\Anc(\Z),\Anc(\Z)}$ is a square matrix which can be simultaneously permuted to lower triangular with diagonals one, and thus is full rank. Then write $\Y = \B_{\Y,:}\E = \B_{\Y,\Anc(\Z)} \E_{\Anc(\Z)} + \B_{\Y,\sim\Anc(\Z)}\E_{\sim\Anc(\Z)}$, where the second part are $\Y$'s noise components that is not shared in $\Anc(\Z)$, so is independent to $\Anc(\Z)$ and can be dropped in calculating covariance. From above, $\cov(\Anc(\Z),\Y) = \B_{\Anc(\Z),\Anc(\Z)} \Phi(\E_{\Anc(\Z)}) \B_{\Y,\Anc(\Z)}^\T$, where $\Phi(\E_{\Anc(\Z)})$ is a diagonal matrix with diagonal entries being variance of exogenous noise terms in $\E_{\Anc(\Z)}$. Since both $\B_{\Anc(\Z),\Anc(\Z)}$ and $\Phi(\E_{\Anc(\Z)})$ are full rank square matrices, $\rank(\B_{\Y,\Anc(\Z)})$ is equal to $\rank (\cov (\Anc(\Z), \Y))$.

According to~\cref{thm:trek_separation} and~\cref{assum:rank_faithfulness}, $\rank (\cov (\W, \Y))$ is equal to $\min \{ \vert \SH_\W \vert + \vert \SH_\Y \vert \vert$ $(\SH_\W,\SH_\Y) \text{ t-separates } (\W,\Y)\}$. Further we obtain a lemma: if $\W=\Anc(\W)$, i.e., ancestors are self-contained in $\W$, then $\rank (\cov (\W, \Y)) = \min \{ \vert \SH \vert \ \vert \ (\emptyset,\SH) \text{ t-separates } (\W,\Y)\}$. This can be proved by that, for self-contained $\W$, for any $(\SH_\W,\SH_\Y)$ that t-separates $(\W,\Y)$, $(\emptyset, \SH_\W\cup \SH_\Y)$ also t-separates $(\W,\Y)$. Another lemma is that, $(\emptyset, \SH)$ t-separates $(\Anc(\Z),\Y)$ if and only if $(\emptyset, \SH)$ t-separates $(\Z,\Y)$ (see~\cref{app:vertex_cut}).

With lemmas above, we immediately have that $\rank(\B_{\Y,\Anc(\Z)})$ is equal to the size of the minimum vertices set $\SH$ s.t. $(\emptyset, \SH)$ t-separates $(\Z,\Y)$, i.e., the size of the minimum vertex cut from $\Anc(\Z)$ to $\Y$.
\end{proof}

\subsection{Proof of~\cref{thm:equiv_gin_true_observed}}
\EQUIVALENCEOVERLATENTOBSERVED*
\begin{proof}[Proof of~\cref{thm:equiv_gin_true_observed}]\label{prf:equiv_gin_true_observed}
\cref{thm:equiv_gin_true_observed} can either be proved by the graphical criteria (where observed variables are all leaf nodes), or by mathematically showing how rank of submatrices of $\B$ preserves among latent and observed variables. Consider the latent $\Xtilde$ in LiNGAM:
\begin{equation}
\Xtilde = \Tilde{\A}\Xtilde + \Tilde{\mathbf{E}}; \ \ \Tilde{\mathbf{X}} = \Tilde{\B}\Tilde{\mathbf{E}}; \ \ \text{variance-covariance matrix }\Tilde{\mathbf{\Sigma}}=\Tilde{\B} \Phi_{\Tilde{\mathbf{E}}} \Tilde{\B}^{\intercal}
\end{equation}

when $X_i = \Xt_i + E_i$, write latent and observed variables together, we have:
\begin{equation}
\begin{aligned}
\left[
    \begin{array}{c}
        \Tilde{\mathbf{X}}\\ \hline \vspace{-.27cm} \\ \mathbf{X}
    \end{array}
\right] &= \A^{'} \cdot \left[
    \begin{array}{c}
        \Tilde{\mathbf{X}}\\ \hline \vspace{-.27cm} \\ \mathbf{X}
    \end{array}
\right] + \left[
    \begin{array}{c}
        \Tilde{\mathbf{E}}\\ \hline \vspace{-.27cm} \\ \mathbf{E}
    \end{array}
\right], \ \ \ \text{where } \A^{'} = \left[
    \begin{array}{c|c}
        \Tilde{\A} & \mathbf{0}\\ \hline \vspace{-.27cm} \\ \mathbf{I} & \mathbf{0}
    \end{array}
\right],\\
\left[
    \begin{array}{c}
        \Tilde{\mathbf{X}}\\ \hline \vspace{-.27cm} \\ \mathbf{X}
    \end{array}
\right] &= \B^{'} \cdot \left[
    \begin{array}{c}
        \Tilde{\mathbf{E}}\\ \hline \vspace{-.27cm} \\ \mathbf{E}
    \end{array}
\right], \ \ \ \text{where } \B^{'} = \left[
    \begin{array}{c|c}
        \Tilde{\B} & \mathbf{0}\\ \hline \vspace{-.27cm} \\ \Tilde{\B} & \mathbf{I}
    \end{array}
\right],\\
\mathbf{\Sigma^{'}} &= \cov\left( \left[
    \begin{array}{c}
        \Tilde{\mathbf{X}}\\ \hline \vspace{-.27cm} \\ \mathbf{X}
    \end{array}
\right] \right) = \left[
    \begin{array}{c|c}
        \Tilde{\mathbf{\Sigma}} & \Tilde{\mathbf{\Sigma}}\\ \hline \vspace{-.27cm} \\ \Tilde{\mathbf{\Sigma}} & \Tilde{\mathbf{\Sigma}} + \Phi_{\mathbf{E}}
    \end{array}
\right], \ \ \ \text{where } \Phi_{\mathbf{E}}=\operatorname{diag}(\var(\mathbf{E})).
\end{aligned}
\end{equation}

More generally, when observations are measured with $X_i = c_i \Xt_i +E_i$, let $\mathbf{C} = \operatorname{diag}([c_1,\cdots,c_n]^\T)$:
\begin{equation}
\A^{'} = \left[
    \begin{array}{c|c}
        \Tilde{\A} & \mathbf{0}\\ \hline \vspace{-.27cm} \\ \mathbf{C} & \mathbf{0}
    \end{array}
\right], \ \ \B^{'} = \left[
    \begin{array}{c|c}
        \Tilde{\B} & \mathbf{0}\\ \hline \vspace{-.27cm} \\ \mathbf{C}\Tilde{\B} & \mathbf{I}
    \end{array}
\right], \ \ \mathbf{\Sigma^{'}} = \left[
    \begin{array}{c|c}
        \Tilde{\mathbf{\Sigma}} & \Tilde{\mathbf{\Sigma}}\mathbf{C}^\T\\ \hline \vspace{-.27cm} \\ \mathbf{C} \Tilde{\mathbf{\Sigma}} & \mathbf{C}\Tilde{\mathbf{\Sigma}}\mathbf{C}^\T + \Phi_{\mathbf{E}}
    \end{array}
\right]
\end{equation}
For two disjoint sets $\Z,\Y$, consider $\B_{\Y,\Anc(\Z)}$, a submatrix in $[\mathbf{C}\Tilde{\B} \ \vert \ \mathbf{I}]$. $\Anc(\Z) = \Anc(\Ztilde) \cup \Z$, where for the second $\Z$ parts, its indexed columns in $\B_{\Y,:}$ must be all zero (since $\Y$ and $\Z$ are disjoint indices in $\mathbf{I}$), and thus can be dropped. For the first $\Anc(\Ztilde)$ part, $\B_{\Y,\Anc(\Ztilde)}$ is $\B_{\Ytilde,\Anc(\Ztilde)}$ with rows scaled by $\mathbf{C}$, and thus the rank holds. Consequently, for two disjoint vertices sets $\Z,\Y$, $\TIN(\Z,\Y) = \TIN(\Ztilde,\Ytilde)$.
\end{proof}

\subsection{Proofs of Other Results}
For other lemmas and theorems in this paper: $\GIN, \IN$ as special cases of $\TIN$ and \cref{lem:egin_one_and_rest} follows directly from the graphical criteria in~\cref{thm:graph_criteria_omega_ZY}. \cref{thm:equiv_gin_Y_subsets} can be proved in a similar way as the proof to~\cref{thm:graph_criteria_omega_ZY}, where the subsets and subdeterminants are considered. For ranks stopped increasing in~\cref{thm:rank_stopped_increasing}, please refer to~\cref{app:stacked_cumulants}.

\section{Using \texorpdfstring{$\GIN$}{TEXT} Condition-Based Algorithm Under 2-Measurements Model}\label{app:gin_two_measurements}
As is illustrated in~\cref{sec:motivation}, when each latent variable $\Tilde{X_i}$ has two pure measurements $X_{i_1}, X_{i_2}$ (by ``pure'' it means that each of $X_{i_1}, X_{i_2}$ has only one latent parent $\Tilde{X_i}$ and no observed parents), graph structure $\Tilde{G}$ over latent variables is fully identifiable by $\GIN$ (a simpler case). This is already a breakthrough comparing to existed methods~\citep{silva2006learning, spirtes2013calculation, kummerfeld2016causal}, which only identify a partial graph.

\newcontent{Here is an illustrating example: consider a simple 2-variables example, $\tilde{X}\rightarrow \tilde{Y}$, with their respective measurements $X_1, X_2$ and $Y_1, Y_2$. One may check the entailed vanishing correlations: $\rho_{X_1,Y_1}\rho_{X_2,Y_2}=\rho_{X_1,Y_2}\rho_{X_2,Y_1}$, $\rho_{X_1,X_2}\rho_{Y_1,Y_2}\neq\rho_{X_1,Y_1}\rho_{X_2,Y_2}$, and $\rho_{X_1,X_2}\rho_{Y_1,Y_2}\neq\rho_{X_1,Y_2}\rho_{X_2,Y_1}$, where $\rho$ denotes correlation coefficient. These (in)equations exhibit no asymmetry between $\tilde{X}$ and $\tilde{Y}$. Indeed, for the inverse direction $\tilde{X}\leftarrow \tilde{Y}$, all the Tetrad constraints among $X_1, X_2, Y_1, Y_2$ hold the same. Therefore, the direction between $\tilde{X}$ and $\tilde{Y}$ is unidentifiable.

However, the $\GIN$ condition can identify an asymmetry: $\GIN$($X_1, Y_{1,2}$) holds, while $\GIN$($Y_1, X_{1,2}$) is violated, and thus the direction $\tilde{X}\rightarrow \tilde{Y}$ is identified. One can see this from the definition of $\GIN$ (\cref{def:gin_definition}).}

Below we give the general algorithm of using GIN to fully identify $\tilde{G}$. \newcontent{Note that here by ``$\GIN$ condition'', it is actually a bit different from the original paper~\citep{xie2020generalized}: it takes into account one more thing than the original definition: the degeneration of $\omega$ (see~\cref{app:more_tin_properties} for details).} Here is the procedure:

Given $2n$ measured variables (where $n$ is the number of vertices in $\Tilde{G}$), let two variables be $\Y$ and the rest $2n-2$ variables be $\Z$, $\GIN(\Z,\Y)$ if and only if these two variables are the two measurements of a same latent variable. Following this, the $2n$ measured variables can first be pairwise clustered, and labeled as $\{X_{i_1}, X_{i_2}\}_{i=1,\cdots,n}$. One may also obtain this pairwise labeling by prior knowledge (e.g., in survey questions design, one already knows which two questions indicate a same latent factor).

Then, find the graph structure $\Tilde{G}$ over $n$ latent variables:

\begin{algorithm}[htb]
	\caption{Identifying graph structure of $\Tilde{G}$ in 2-measurements case}
	\label{alg:one}
	\hspace*{0.02in} {\bf Input:}
	Labeled $2n$ measurements $\X=\{X_{i_1}, X_{i_2}\}_{i=1,\cdots,n}$ and corresponding data samples\\
	\hspace*{0.02in} {\bf Output:}
	Graph structure of $\Tilde{G}$
	\begin{algorithmic}[1]
	\STATE Initialize ordered list $K\coloneqq \emptyset$, remaining indexes $U\coloneqq \{1,\cdots,n\}$, parents dictionary $P\coloneqq \{\}$;
	\STATE Denote a half of measurements $\mathbf{X_1}=\{X_{i_1}\}_{i=1,\cdots,n}$;
	\WHILE{there are more than one remaining index in $U$}
    	\STATE Find one $j\in U$ with $\GIN(\Z,\mathbf{X_1})$, where $\Z\coloneqq \{X_{i_2} | i\in K \cup \{j\}\}$; //pick from another half
    	\STATE Append $j$ to the end of $K$. Let $U\coloneqq U \backslash \{j\}$;
	\ENDWHILE
	\STATE Append the only one remaining index in $U$ to the end of $K$;
	\FOR{vertex index $j$ in causal ordering list $K$}
	\STATE Let $A\coloneqq\{i|i\text{ earlier than }j\text{ in }K\}$, $\Z\coloneqq\{X_{i_1}|i\in A\}$, $\Y\coloneqq\{X_{i_2}|i\in A\cup\{j\}\}$;
	\STATE $\GIN(\Z,\Y)$ must hold, with solution $\omega$. Let $P[j]\coloneqq \{i\in A|\omega\text{ on }X_{i_2}\text{ is non-degenerated}\}$;
	\ENDFOR
	\STATE {\bfseries Return:} Graph structure $\Tilde{G}$ where each vertex $j$ has direct parents $P[j]$
	\end{algorithmic}
\end{algorithm}

\cref{alg:one} follows a similar procedure as Direct-LiNGAM~\citep{shimizu2011directlingam}: Lines 3-5 sorts the vertices by causal ordering, where there is no edge from later ones to earlier ones. Then according to degeneration graphical criteria in~\cref{app:more_tin_properties}, Line 10 identifies the direct parents set of each vertex from its causally earlier vertices set.

Further consider the coefficients. Denote the linear coefficients of latent variable $\Tilde{X_i}$ to two measurements $X_{i_{1,2}}$ as $\alpha_{i_{1,2}}$ respectively. The ratio $\alpha_{i_1}/\alpha_{i_2}$ is accessible when testing $\GIN$ for pairwise clusters. Then, in Line 10, to identify parents set for each vertex $j$, we find from $A$, the vertices earlier than $j$ in ordered list $K$. We write a scaling vector $\mathbf{s}\coloneqq \{\alpha_{j_2}/\alpha_{i_2} | i\in A\}$, and denote the coefficients vector from $A$ to $j$ as $\mathbf{c}$ (zero if no direct edge). Note that here $\omega$ must only have a free degree of one (according to~\cref{thm:graph_criteria_omega_ZY} and~\cref{app:critical_vertex_cut}, critical vertex cut is $A$). So set the value of $\omega$ on $X_{j_2}$ as $-1$, then the value of $\omega$ on other $X_{i_{2}}$s is exactly the point-wise multiplication of $\mathbf{s}$ and $\mathbf{c}$. If we further assume that linear coefficients from latent variables to measurements are all same (e.g., one, $X_{i_{1,2}}=\Tilde{X_i}+E_{i_{1,2}}$), or equivalently, the measurement errors are uni-variance, then the coefficients among $\Tilde{G}$ is also fully identifiable.

\section{Elaboration on Vertex Cut and Graph Definitions}\label{app:vertex_cut}
We first give more detailed definitions to the concepts in~\cref{sec:approach}.

\begin{definition}[Directed paths]\label{def:directed_paths}
A directed path $P = (i_0, i_1, \cdots, i_k)$ in $G$ is a sequence of vertices of $G$ where there is a directed edge from $i_j$ to $i_{j+1}$ for any $0\leq j \leq k-1$. We use notation $i\rightsquigarrow j$ to show that there exists a directed path from vertex $i$ to $j$.
\begin{remark}Note that a single vertex is also a directed path, i.e., $i\rightsquigarrow i$ holds true.\end{remark}
\end{definition}

\begin{definition}[Directed paths without passing through $\SH$]\label{def:directed_paths_nopassing}
Let $\SH$ be a subset of vertices. We use notation $i \stackrqarrow{\cancel{[\SH]}} j$ to show that there exists a directed path from vertex $i$ to $j$ without passing through $\SH$, i.e., there exists a directed path $P = (i, m_0, \cdots, m_k,j)$ in $G$ s.t. $i,j\not\in\SH$ and $m_l\not\in\SH$ for any $0\leq l \leq k$.
\begin{remark}Note that when $\SH$ is empty, $i\rightsquigarrow j$ is equivalent to $i \stackrqarrow{\cancel{[\SH]}} j$.\end{remark}
\end{definition}

\begin{definition}[Ancestors]\label{def:ancestors}
Let $\W$ be a subset of vertices. Ancestors $\Anc(\W)\coloneqq \{j|\exists i\in\W, j\rightsquigarrow i\}$.
\end{definition}
\begin{remark}
Note that $\W\subseteq \Anc(\W)$. Under faithfulness assumption (no parameter coupling), $\Anc(\W)$ means all noise components that $\W$ carries, i.e., writing the corresponding variables set $\{X_i|i\in \W\}$ as linear combination of noises, it contains and only contains exogenous noises from $\{E_i|i\in \Anc(\W)\}$.
\end{remark}

\begin{definition}[Ancestors outside $\SH$]\label{def:ancestors_outside_S}
Let $\W,\SH$ be two subsets of vertices. We denote ancestors of $\W$ that has directed paths into $\W$ without passing through $\SH$ as $\Ancout{\SH}(\W)\coloneqq \{j|\exists i\in\W, j \stackrqarrow{\cancel{[\SH]}} i\}$.
\end{definition}

\begin{remark}
According to definitions above,
\begin{itemize}[noitemsep,topsep=-3pt]
\item[1.] $\Ancout{\emptyset}(\W)=Anc(\W)$. $\Ancout{\W}(\W)=\emptyset$.
\item[2.] $\SH\cap \Ancout{\SH}(\W)=\emptyset$. $\W\backslash\SH\subseteq \Ancout{\SH}(\W)$.
\item[3.] For overlapped $\SH,\W$, $\Ancout{\SH}(\W)=\Ancout{\SH}(\W\backslash \SH)$.
\item[4.] Roughly speaking, $\Ancout{\SH}(\W)$ means all noise components that can contribute to $\W$ without passing $\SH$. With slight notation abuse, we can write variables $\W$ as $\W=A\SH+\E_\W$, where $A\SH$ is a linear transformation to $\SH$, and $\E_\W$ is a linear transformation to exogenous noises set that contains and only contains $\{E_i|i\in \Ancout{\SH}(\W)\}$.
\end{itemize}
\end{remark}

\begin{definition}[Existence of causal effect from $\W_1$ to $\W_2$]\label{def:exist_causal_effect}
Let $\W_1,\W_2$ be two subsets of vertices. We say there exists causal effect from $\W_1$ to $\W_2$ if and only if there exists a directed path $i\rightsquigarrow j$ with $i\in\W_1$ and $j\in\W_2$.
\end{definition}
\begin{remark}According to definitions above,
\begin{itemize}[noitemsep,topsep=-3pt]
\item[1.] Note that if $\W_1$ and $\W_2$ are not disjoint, then there must exist causal effect from $\W_1$ to $\W_2$.
\item[2.] An equivalent definition is that, $\Anc(\W_2)\cap\W_1\neq \emptyset$.
\end{itemize}
\end{remark}

\begin{definition}[Existence of causal effect from $\W_1$ to $\W_2$ without passing through $\SH$]\label{def:exist_causal_effect_without_passing}
Let $\W_1,\W_2,\SH$ be three subsets of vertices. We say there exists causal effect from $\W_1$ to $\W_2$ without passing through $\SH$ if and only if there exists a directed no-passing path $i \stackrqarrow{\cancel{[\SH]}} j$ with $i\in\W_1$ and $j\in\W_2$.
\end{definition}
\begin{remark}According to definitions above,
\begin{itemize}[noitemsep,topsep=-3pt]
\item[1.] An equivalent definition is that, $\Ancout{\SH}(\W_2)\cap\W_1\neq \emptyset$.
\item[2.] There exists no causal effect from $\SH$ to $\W_1$ without passing $\SH$, i.e., $\SH\cap \Ancout{\SH}(\W)=\emptyset$.
\item[3.] This definition shows whether $\SH$ chokes \textbf{all} directed paths from $\W_1$ to $\W_2$.
\item[4.] By~\cref{def:vertex_cut}, the following statements are equivalent: \textbf{1)} there exists no causal effect from $\W_1$ to $\W_2$ without passing through $\SH$; \textbf{2)} $\SH$ is a vertex cut from $\W_1$ to $\W_2$; \textbf{3)} $\forall i\in\W_1,j\in\W_2,$ $i\stackrqarrow{\cancel{[\SH]}} j$ does not hold; \textbf{4)} $\Ancout{\SH}(\W_2)\cap\W_1=\emptyset$; \textbf{5)} $\SH$'s removal from $G$ ensures there is no directed paths from $\W_1\backslash\SH$ to $\W_2\backslash\SH$.
\end{itemize}
\end{remark}

Now we have complete our graphical definitions. Let us also review trek-separation~\citep{sullivant2010trek}.

\begin{definition}[Trek]\label{def:trek}
A \textit{trek} in $G$ from $i$ to $j$ is an ordered pair of directed
paths $(P_1, P_2)$ where $P_1$ has sink $i$, $P_2$ has sink $j$, and both $P_1$ and $P_2$ have the same source $k$. Note that one or both of $P_1$ and $P_2$ may consist of a single vertex, e.g., $((i), (i))$ is a trek from vertex $i$ to $i$.
\end{definition}

\begin{definition}[t-separation]\label{def:t_separation}
Let $\W, \Y, \SH_\W, \SH_\Y$ be four subsets of $V(G)$ which need not be disjoint. We say that the pair $(\SH_\W, \SH_\Y)$ \textit{trek separates} (or \textit{t-separates}) $\W$ from $\Y$ if for every trek $(P_1,P_2)$ from a vertex in $\W$ to a vertex in $\Y$, either $P_1$ contains a vertex in $\SH_\W$ or $P_2$ contains a vertex in $\SH_\Y$.
\end{definition}

The above two definitions are directly from~\citep{sullivant2010trek}. By the ``ancestors'' related definitions introduced above and in~\cref{sec:approach}, we can immediately get an equivalent restatement of t-separation as:

\begin{theorem}[Restatement of t-separation]\label{restate_t_separation}
Let $\W, \Y, \SH_\W, \SH_\Y$ be four subsets of $V(G)$ which need not be disjoint. The pair $(\SH_\W, \SH_\Y)$ t-separates $\W$ from $\Y$, if and only if there exists no causal effect from $\Ancout{\SH_\W}(\W)$ to $\Y$ without passing passing $\SH_\Y$ (see~\cref{def:exist_causal_effect_without_passing}).
\end{theorem}

Note that the above graph condition also has an equivalent restatement: 

\textit{$\cdots$ if and only if there exists no causal effect from $\Ancout{\SH_\Y}(\Y)$ to $\W$ without passing passing $\SH_\W$}.

Both mean that $\Ancout{\SH_\W}(\W)\cap \Ancout{\SH_\Y}(\Y)=\emptyset$, i.e., if some noise components can flow into $\W$ without passing $\SH_\W$, then it cannot also flow into $\Y$ without passing $\SH_\Y$, or vice versa.

\begin{remark}
Further, by definitions above and the rank constraints (in trek-separation theorem~\cref{thm:trek_separation}), we have the followings:
\begin{itemize}[noitemsep,topsep=-3pt]
\item[1.] $\rank(\cov(\W,\Y))\geq |\W\cap\Y|$, since we must have $\W\cap\Y \subseteq \SH_\W\cup\SH_\Y$ if $(\SH_\W, \SH_\Y)$ t-separates $\W$ from $\Y$, otherwise some unblocked vertex is in $\Ancout{\SH_\W}(\W)\cap \Ancout{\SH_\Y}(\Y)$.
\item[2.] $\rank(\cov(\W,\Y))\leq \min(|\W|,|\Y|)$, since $(\W, \emptyset)$ and $(\emptyset, \Y)$ always t-separates $\W$ from $\Y$, i.e., $\Ancout{\W}(\W)=\emptyset$ or $\Ancout{\Y}(\Y)=\emptyset$.
\item[3.] $\rank(\cov(\W,\Y))\leq |\Anc(\W)\cap \Anc(\Y)|$, since $(\Anc(\W)\cap \Anc(\Y), \emptyset)$ and $(\emptyset, \Anc(\W)\cap \Anc(\Y))$ always t-separates $\W$ from $\Y$.
\item[4.] The pair $(\SH_\W, \SH_\Y)$ t-separates $\W$ from $\Y$, if and only if the pair $(\SH_\Y, \SH_\W)$ t-separates $\Y$ from $\W$.
\end{itemize}
\end{remark}

From above we have seen the interpretation of t-separation from the ``ancestors'' language set. Then combining~\cref{def:vertex_cut} and~\cref{restate_t_separation}, we know that the following statements are equivalent: \textbf{1)} $\SH$ is a vertex cut from $\Anc(\Z)$ to $\Y$; \textbf{2)} $(\emptyset,\SH)$ t-separates $(\Z,\Y)$; \textbf{3)} There exists no causal effect from $\Anc(\Z)$ to $\Y$ without passing through $\SH$; \textbf{4)} There exists no causal effect from $\Ancout{\SH}(\Y)$ to $\Z$.

\section{More Properties of \texorpdfstring{$\TIN$}{TEXT} Condition}\label{app:more_tin_properties}
\subsection{Critical Vertex Cut}\label{app:critical_vertex_cut}
From the above~\cref{sec:approach} and~\cref{app:vertex_cut} graphical criteria, we know that $\TIN(\Z,\Y)$ is equal to the size of the minimum vertex cut from $\Anc(\Z)$ to $\Y$.
\begin{remark}\label{remark_y_side_choke_set}
Following~\cref{def:vertex_cut}, we first elaborate more on vertex cut:
\begin{itemize}[noitemsep,topsep=-3pt]
\item[1.] For any $\Z,\Y$, any superset of $\Y$ (including $\Y$) is a vertex cut from $\Anc(\Z)$ to $\Y$.
\item[2.] For any $\Z,\Y$, any superset of $Anc(\Z)$ (including $Anc(\Z)$) is a vertex cut from $\Anc(\Z)$ to $\Y$.
\item[3.] For any vertex cut from $\Anc(\Z)$ to $\Y$, $\Anc(\Z)\cap\Y\subseteq \SH$ (to choke single vertex paths).
\item[4.] Following point 3, for overlapped $\Z,\Y$ in testing $\TIN$ condition, any vertex cut $\SH$ must contains (at least) $\Z\cap\Y$ (the observed/testable intersection) as its subset.
\item[5.] Following point 4, if $\Y \subseteq \Z$, then there exists no non-zero $\w$ s.t. $\wTY \indep \Z$.
\item[6.] Note that though expressed as ``$\SH$ is a vertex cut from $\Anc(\Z)$ to $\Y$'', it never implicitly implies a causal ordering of $\Z\rightarrow\SH\rightarrow\Y$. E.g., in graph $D\leftarrow A\rightarrow C\leftarrow B$, consider $\TIN(\Z=\{A\}, \Y=\{B,C\})=1$ where the minimum vertex cut is $\SH=\{C\}$, not causally earlier than $\Y$; $\TIN(\Z=\{B,C\}, \Y=\{A,D\})=1$ with the minimum vertex cut $\SH = \{A\}$, but $\Z$ is neither causally earlier than $\Y$ nor than $\SH$.
\item[7.] Following point 6, roughly speaking, $\TIN$ tells size of the minimum vertex cut, but not exactly the causal ordering. For the existence of non-zero $\w$ s.t. $\wTY\indep \Z$, there can be some vertices in $\Y$ that are in or causally earlier than $\Z$, i.e. $\Anc(\Z)\cap\Y\neq\emptyset$ - as long as there are not ``too many'' (less than the cardinality of possible $\SH$).
\item[8.] Note that the minimum vertex cut may not be unique. E.g., 1) Consider example in point 6, both $\SH=\{C\}$ and $\SH=\{A\}$ are minimum vertex cuts in $\TIN(\Z=\{A\},\Y=\{B,C\})=1$. 2) Consider a chain structure with $\TIN(\Z=\{X_1\}, \Y=\{X_2,\cdots,X_n\})=1$, both $\SH=\{X_1\}$ and $\SH=\{X_2\}$ are minimum vertex cuts.
\end{itemize}
\end{remark}

Following point 6 of~\cref{remark_y_side_choke_set}, since the minimum vertex cut from $\Anc(\Z)$ to $\Y$ may not be unique in a $\TIN(\Z,\Y)$, to better use the graphical criteria, now we define the \textit{critical vertex cut}:

\begin{definition}[Critical vertex cut]\label{def:critical_gin_separation_set}
Denote $\mathcal{S}(\Z,\Y)$ the collection of all sets $\SH\subseteq V(G)$ s.t. $\SH$ is a minimum vertex cut from $\Anc(\Z)$ to $\Y$ (``minimum'' means that $\vert \SH \vert = \TIN(\Z,\Y)$). For a vertex cut $\SH \in \mathcal{S}(\Z,\Y)$, we say $\SH$ is \textit{critical} if and only if there exists no causal effect from all (other) minimum vertex cuts to $\Y$ without passing through $\SH$, i.e. $\Ancout{\SH}(\Y)\cap \Anc(\bigcup \Scal(\Z,\Y))=\emptyset$.
\end{definition}
\begin{remark}
Roughly speaking, when there are multiple minimum vertex cuts, i.e., these multiple sets can all cut from $\Anc(\Z)$ to $\Y$, then a critical one means a ``last'' one (furthest from $\Z$, deepest to $\Y$): it not only cuts $\Anc(\Z)$ to $\Y$, but also cuts all other vertex cuts to $\Y$. E.g., consider examples in point 8 of~\cref{remark_y_side_choke_set}, 1) $\{C\}$ is critical while $\{A\}$ is not, because $\{C\}$ can cut $\{A\}$ to $\{B,C\}$, but $\{A\}$ cannot cut $\{C\}$ to $\{B,C\}$. 2) $\{X_2\}$ is critical while $\{X_1\}$ is not.
\end{remark}

\begin{theorem}[Uniqueness of critical gin-separation set]\label{uniqueness_of_critical}
For two vertices sets $\Z$ and $\Y$ and their respective $\TIN(\Z,\Y)$, there exists one and only one corresponding critical vertex cut, denoted as $\SH^*_{\Z,\Y}$.
\end{theorem}

\subsection{Noise Components of Linear Transformation \texorpdfstring{$\wTY$}{TEXT}}\label{app:noise_components}

From above~\cref{app:critical_vertex_cut} we defined the critical vertex cut $\SH^*_{\Z,\Y}$ behind a $\TIN(\Z,\Y)$, with special property on it. Now we analyze the linear transformation $\omega^\intercal\Y$:

\begin{theorem}[Noise components of linear transformation $\omega^\intercal\Y$]\label{noise_components}
For two vertices sets $\Z$ and $\Y$ and their respective $\TIN(\Z,\Y)$, for generic choice of $\w$ (i.e., no coincidental noise cancelling by $\w$), the corresponding linear transformation $\omega^\intercal\Y$ contains and only contains exogenous noises introduced by vertices that has directed paths to $\Y$ without passing through the critical vertex cut $\SH^*_{\Z,\Y}$, i.e., $\mathcal{E}(\omega^\intercal\Y)=\{E_i|i\in \Ancout{\SH^*_{\Z,\Y}}(\Y)\}$, where $\mathcal{E}(\cdot)$ denotes the exogenous noises components set that a variable $\cdot$ is constituted of, and $E_i$ is the exogenous noise 
from vertex $i$.
\end{theorem}

\begin{remark}
A vertex cut $\SH$ from $\Anc(\Z)$ to $\Y$ yields that all noise components that $\Z$ carries (i.e., $\Anc(\Z)$) cannot flow into (causal affects / contribute to) $\Y$ without passing through $\SH$, then $\Y$ can be written as $\Y=L\SH+\mathbf{E'_\Y}$, where $L$ denotes a linear transformation, and $\mathbf{E'_\Y}$ denotes noise components that can contribute to $\Y$ without passing through $\SH$ (i.e., $\Ancout{\SH}(\Y)$) - so $\mathbf{E'_\Y}\indep \Z$, but not necessarily $\mathbf{E'_\Y} \indep \SH$.
\end{remark}

Also, we define $\OmegaZY$ as $\{\w \vert \ \wTY \indep\Z\}$, while actually for such $\w$, $\omega^\intercal\Y$ is independent to more variables:

\begin{theorem}[Full version of $\omega^\intercal\Y$ independence]\label{full_independence}
For two vertices sets $\Z$ and $\Y$ and their respective critical vertex cut $\SH^*_{\Z,\Y}$, for any variable $X_i\in\X$ (i.e., respective vertex $i\in V(G)$), $\omega^\intercal\Y \indep X_i$ if and only if there exists no causal effect from $\Ancout{\SH^*_{\Z,\Y}}(\Y)$ to $\{i\}$, i.e., $\Ancout{\SH^*_{\Z,\Y}}(\Y) \cap \Anc(\{i\}) = \emptyset$.
\end{theorem}

\begin{remark}
With~\cref{def:exist_causal_effect} and~\cref{full_independence}, we can immediately get the following:
\begin{itemize}[noitemsep,topsep=-3pt]
\item[1.] $\omega^\intercal\Y\indep\Z$ - it can be derived from~\cref{full_independence},~\cref{def:ancestors} and~\cref{def:critical_gin_separation_set}.
\item[2.] $\omega^\intercal\Y\indep \Anc(\Z)$ - it can be derived from~\cref{full_independence} and~\cref{def:critical_gin_separation_set}.
\item[3.] \cref{full_independence} is straightforward by seeing $\omega^\intercal\Y$ as a linear transformation of its noise sources $\{E_i|i\in \Ancout{\SH^*_{\Z,\Y}}(\Y)\}$. Then any variable is independent to $\omega^\intercal\Y$ if and only if it does not carry noise from these sources (i.e., vertex has no ancestors in $\Ancout{\SH^*_{\Z,\Y}}(\Y)$), by the $\darmois$. With~\cref{full_independence}, after testing on $\TIN(\Z,\Y)$, one can do more independence test over other variables (as long as they are observed/testable), and may get more information about the whole graph structure and the location of critical vertex cut.
\end{itemize}
\end{remark}

Further, we notice that in the independent linear transformation subspace $\OmegaZY$, some indices of $\w$ may be degenerated (i.e., fixed to zero). Consider following examples for an intuition: 1) On a chain structure~\cref{fig:chain} with $\TIN(\{X_2\}, \{X_1,X_3,X_4,\cdots,X_n\})=2$, $\w$ index on $X_1$ must be zero (not include $X_1$ in linear transformation) to make $\wTY\indep\Z$, while in a fully connected DAG~\cref{fig:fully_connected} with also $\TIN(\{X_2\}, \{X_1,X_3,X_4,\cdots,X_n\})=2$, $\w$ is not degenerated on any indices. 2) On a chain structure~\cref{fig:chain} or a chain structure with triangular head~\cref{fig:triangle_head_chain}, $\TIN(\{X_1,X_3\}, \{X_2,X_4,X_5,\cdots,X_n\})=2$ holds, while $\w$ index on $X_2$ must be zero. 3) in~\cref{fig:example_d}, $\TIN(\{X_1\}, \{X_2,X_5\})=1$, while actually $\w$ is degenerated on $X_5$ index, which means that the linear transformation actually does not include $X_5$ and is just trivially $X_2$ independent of $X_1$ (here $\SH^*_{\Z,\Y}$ is just $X_5$).

Now, we would like to first give mathematical characterization for such $\omega$ indices degeneration:

\begin{theorem}\label{mathdegeneration}
Since $\OmegaZY = \operatorname{null}(\B_{\Y,\Anc(\Z)}^\T)$, $\OmegaZY$ degenerates on an index $y\in\Y$ if and only if: remove the corresponding $y$-th column in $\B_{\Y,\Anc(\Z)}^\T$ to get submatrix $\B_{\Y\backslash\{y\},\Anc(\Z)}^\T$, the rank of submatrix is one less than the rank of full matrix $\B_{\Y,\Anc(\Z)}^\T$.
\end{theorem}

Then we give the equivalent graphical criteria for such $\w$ indices degeneration:

We already know that the vertex cut $\SH^*_{\Z,\Y}$ cuts $\Anc(\Z)$ to $\Y$. Moreover, each part of $\SH^*_{\Z,\Y}$ has its ``own indispensable work'' in cutting, so we first define:
\begin{definition}[Local cut scope]\label{def:local_choke_scope}
For each vertex $s\in\SH^*_{\Z,\Y}$, define its local choke scope as $\LC(s)\coloneqq $ $\{y\in\Y | \text{ there exists causal effect from } \Anc(\Z) \text{ to } \{y\} \text{ without passing through } \SH^*_{\Z,\Y} \backslash \{s\} \}$. Furthermore, for each subset $S\subseteq \SH^*_{\Z,\Y}$, define $\LC(S)\coloneqq \{y\in\Y|$ there exists causal effect from $\Anc(\Z) \text{ to } \{y\} \text{ without passing through } \SH^*_{\Z,\Y} \backslash S \}$.
\end{definition}
\begin{remark}
With~\cref{def:local_choke_scope} we have the following:
\begin{itemize}[noitemsep,topsep=-3pt]
\item[1.] $\SH^*_{\Z,\Y}=\emptyset$ if and only if $\Z,\Y$ are marginally independent, i.e., $\B_{\Y,\Anc(\Z)}$ are all zero (no shared noise components).
\item[2.] $\LC(s)$ means the part of $\Y$ that would not be cut/choked, had there been no $s$. In other word, the part of $\Y$ that $s$ has its own indispensable work.
\item[3.] $\LC(S) = \cup_{s\in S}{LC(s)}$. $\LC(S_1\cup S_2) =\LC(S_1) \cup \LC(S_2)$.
\item[4.] $S\subseteq \Anc(\LC(S))$.
\item[5.] For any subset $S$, $|\LC(S)|\geq |S|$ (so $|\LC(s)|\geq 1$ for any vertex $s$).
\item[6.] For any two different vertices $s_1, s_2$, it does not necessarily yield that $\LC(s_1)\cap \LC(s_2) = \emptyset$ - they may work together to cut/choke a part and either is indispensable for this part.
\item[7.] $\LC(\SH^*_{\Z,\Y})$ may not be the whole $\Y$, but a proper subset. The rest $\Y\backslash \LC(\SH^*_{\Z,\Y})$ is exactly part of $\Y$ that is marginally independent to $\Z$ (i.e., no directed paths from $\Anc(\Z)$ to that part).
\end{itemize}
\end{remark}

\begin{theorem}[Graphical criteria for degeneration]\label{graph_criteria_degeneration}
$\OmegaZY$ degenerates on on the indexes subset $Y\subset\Y$ if and only if: there exists a subset $S\subset \SH^*_{\Z,\Y}$ such that its local choke scope $\LC(S)=Y$, and $|Y|=|S|=|\LC(S)|$.
\end{theorem}
\begin{remark}
We already have~\cref{mathdegeneration} for math condition. And for graphical criteria~\cref{graph_criteria_degeneration}:
\begin{itemize}[noitemsep,topsep=-3pt]
\item[1.] A rough interpretation: in general we would expect a smaller $S$ to choke a larger $Y$. However, if for an $S$, its local choke scope $Y$ is of the same size as $S$, then removing $S$ will only affect the same size $Y$ (a feeling that this $S$ is ``wasted''). Then this part $Y$ will be degenerated.
\item[2.] From $Y$ side, it means that this $Y$ requires a same size of separation set $S$ to choke (a feeling that $Y$ is ``too expensive'').
\item[3.] $\TIN(\Z,\Y\backslash Y) = \TIN(\Z,\Y) - |Y|$, with critical vertex cut being $\SH^*_{\Z,\Y}\backslash S$, and with no degeneration.
\item[4.] Note that degeneration does not yield independence, i.e., if $\TIN(\Z,\Y)$ with $\OmegaZY$ degenerated on $Y_k$, it does not necessarily yield that $\omega ^\intercal\Y\indep Y_k$. Because $\omega$ is applied to variables, not noise components. For example, the v-structure $\{A,B\}\rightarrow C$, $\TIN(A,BC)=1$ with $C$ degenerated. But $\omega ^\intercal BC$, which is simply $B$, is not independent to variable $C$.
\item[5.] The inverse direction of 4. is also not sufficient: if for $\TIN(\Z,\Y)$ and some $Y_k\in \Y$, there is also $\omega ^\intercal\Y\indep Y_k$ ($\omega ^\intercal\Y$ independent to not only to $\Z$ but also some part in $\Y$), it still does not necessarily yield that $Y_k$ is degenerated. E.g., in chain structure~\cref{fig:chain}, $\TIN(\{\Xt_1\}, \{ \Xt_2,\cdots,\Xt_n \})=1$ with $\omega ^\intercal\Y$ also independent to $\Xt_2$, but $\Xt_2$ index is not degenerated in $\OmegaZY$.
\item[6.] Note that while as a special case of point 5, in measurement error case, independence in $\Y$ yields degeneration. Because each variable $Y_k\in \Y$ is associated with measurement noise $E_k$ which is only in $Y_k$, not in any other variables (so cannot be cancelled). Then to make $\omega ^\intercal\Y\indep Y_k$, at least $E_k$ must be removed, i.e., $Y_k$ degenerated. E.g., in chain structure~\cref{fig:chain}, $\TIN(\{X_1\}, \{ X_2,\cdots,X_n \})=1$ with no degeneration. So $\omega ^\intercal\Y$ is only independent of $X_1$, not any other in observed variables $\Y = \{X_2,\cdots,X_n\}$ - specifically, $\omega ^\intercal\Y\indep \{\Tilde{X}_1, \Tilde{X}_2, X_1\}$.
\end{itemize}
\end{remark}

\begin{table}[h]
  \caption{Full version of~\cref{table:egin_examples} with more properties on $\TIN$. Examples of $\EGIN$ on different $(\Z,\Y)$ pairs over different graph structures in~\cref{fig:graph_examples}.}
  \label{table:egin_examples_full}
\vspace{-0.5em}
\scriptsize
  \setlength{\tabcolsep}{3.0pt}
  \renewcommand{\arraystretch}{0.7}
\begin{center}
\resizebox{\textwidth}{!}{
\begin{tabular}{ l *{16}{c} }
\toprule
$(\Z,\Y)$
&
\multicolumn{4}{c}{$(\{X_1, X_2\}, \{X_3, X_4, X_5\})$} &
\multicolumn{4}{c}{$(\{X_1, X_2\}, \{X_4, X_5\})$} &
\multicolumn{4}{c}{$(\{X_3\}, \{X_1, X_2, X_4, X_5\})$} &
\multicolumn{4}{c}{$(\{X_1, X_4\}, \{X_3, X_4, X_5\})$} \\
\cmidrule(lr){2-5} \cmidrule(lr){6-9} \cmidrule(lr){10-13} \cmidrule(lr){14-17}
Graph in~\cref{fig:graph_examples}&
(a) & (b) & (c) & (d) &
(a) & (b) & (c) & (d) &
(a) & (b) & (c) & (d) &
(a) & (b) & (c) & (d) \\
\midrule
$\EGIN(\Z,\Y)$ & $1$ & $1$ & $2$ & $1$ & $1$ & $1$ & $2$ & $1$ & $3$ & $2$ & $3$ & $1$ & $2$ & $2$ & $3$ & $2$\\
$\GIN(\Z,\Y)$ & $\true$ & $\true$ & $\true$ & $\true$ & $\true$ & $\true$ & $\false$ & $\true$ & $\false$ & $\false$ & $\false$ & $\true$ & $\true$ & $\true$ & $\false$ & $\true$\\
$\dim(\OmegaZY)$ & $2$ & $2$ & $1$ & $2$ & $1$ & $1$ & $0$ & $1$ & $1$ & $2$ & $1$ & $3$ & $1$ & $1$ & $0$ & $1$\\
$\operatorname{rk}(\Sigma_{\Z\Y})$ & $1$ & $1$ & $2$ & $1$ & $1$ & $1$ & $2$ & $1$ & $1$ & $1$ & $1$ & $1$ & $2$ & $2$ & $2$ & $2$\\
$\Anc(\Z)$ & $X_{ 1,2 }$ & $X_{ 1,2,3 }$ & $X_{ 1,2 }$ & $X_{ 1,2 }$ & $X_{ 1,2 }$ & $X_{ 1,2,3 }$ & $X_{ 1,2 }$ & $X_{ 1,2 }$ & $X_{ 1,2,3 }$ & $X_{ 1,3 }$ & $X_{ 1,2,3 }$ & $X_{ 3 }$ & $X_{ 1,2,3,4 }$ & $X_{ 1,3,4 }$ & $X_{ 1,2,3,4 }$ & $X_{ 1,2,3,4 }$\\
$\SH^*_{\Z,\Y}$ & $X_{ 3 }$ & $X_{ 3 }$ & $X_{ 1,2 }$ & $X_{ 4 }$ & $X_{ 4 }$ & $X_{ 4 }$ & $X_{ 4,5 }$ & $X_{ 4 }$ & $X_{ 1,2,4 }$ & $X_{ 1,3 }$ & $X_{ 1,2,3 }$ & $X_{ 3 }$ & $X_{ 3,4 }$ & $X_{ 3,4 }$ & $X_{ 3,4,5 }$ & $X_{ 3,4 }$\\
$\operatorname{A}_{\operatorname{o}(\SH^*)}(\Y)$ & $X_{ 4,5 }$ & $X_{ 4,5 }$ & $X_{ 3,4,5 }$ & $X_{ 3,5 }$ & $X_{ 5 }$ & $X_{ 5 }$ & $\emptyset$ & $X_{ 3,5 }$ & $X_{ 5 }$ & $X_{ 2,4,5 }$ & $X_{ 4,5 }$ & $X_{ 1,2,4,5 }$ & $X_{ 5 }$ & $X_{ 5 }$ & $\emptyset$ & $X_{ 5 }$\\
$\mathcal{E}(\wTY)$ & $E_{ 4,5 }$ & $E_{ 4,5 }$ & $E_{ 3,4,5 }$ & $E_{ 3,5 }$ & $E_{ 5 }$ & $E_{ 5 }$ & $\emptyset$ & $E_{ 3,5 }$ & $E_{ 5 }$ & $E_{ 2,4,5 }$ & $E_{ 4,5 }$ & $E_{ 1,2,4,5 }$ & $E_{ 5 }$ & $E_{ 5 }$ & $\emptyset$ & $E_{ 5 }$\\
$\wTY\indep$ to & $X_{ 1,2,3 }$ & $X_{ 1,2,3 }$ & $X_{ 1,2 }$ & $X_{ 1,2 }$ & $X_{ 1,2,3,4 }$ & $X_{ 1,2,3,4 }$ & $\text{const}$ & $X_{ 1,2 }$ & $X_{ 1,2,3,4 }$ & $X_{ 1,3 }$ & $X_{ 1,2,3 }$ & $X_{ 3 }$ & $X_{ 1,2,3,4 }$ & $X_{ 1,2,3,4 }$ & $\text{const}$ & $X_{ 1,2,3,4 }$\\
$\w$ degenerate & $\backslash$ & $\backslash$ & $\backslash$ & $\backslash$ & $\backslash$ & $\backslash$ & $\w_{ 4,5 }$ & $\backslash$ & $\w_{ 1,2 }$ & $\backslash$ & $\backslash$ & $\backslash$ & $\w_{ 3 }$ & $\w_{ 3 }$ & $\w_{ 3,4,5 }$ & $\backslash$\\
\bottomrule
\end{tabular}}
\end{center}
\end{table}

\cref{table:egin_examples_full} is a full version of~\cref{table:egin_examples}, where we could use examples to better understand the above properties about $\TIN$ condition: e.g., different cases for $\GIN(\Z,\Y)$ to be violated (see rank of $\B_{\Y,\Anc(\Z)}$ and rank of $\cov(\Z,\Y)$); noise components of $\wTY$ is exactly corresponding to $\Ancout{\SH^*_{\Z,\Y}}(\Y)$; the graphical criteria for some $\w$ indices degeneration, etc.

\subsection{Subsets Implications of the \texorpdfstring{$\TIN$}{TEXT} Condition}\label{app:subset_implications}
In~\cref{sec:estimation} we give~\cref{thm:equiv_gin_Y_subsets} for estimation of $\OmegaZY$, by tackling down $\Y$ to subsets:
\TINOVERYSUBSETS*

\begin{remark}
About how to use this ``big to small'' property, here are some notes:
\begin{itemize}[noitemsep,topsep=-3pt]
\item[1.] Condition 1) can also be ``$|\Y'|\geq k + 1$'' (a weaker/stronger version).
\item[2.] This can be shown by that if a set $\SH$ is a vertex cut from $\Anc(\Z)$ to $\Y$, then $\SH$ is also a vertex cut from any subset of $\Anc(\Z)$ to any subset of $\Y$.
\item[3.] It does not yield that all these $\TIN$ conditions on subsets $\Y'$ has a same rank $k$, and even with a same rank, not necessarily a same critical vertex cut $\SH^*_{\Z,\Y}$. E.g., consider a 3-v-structure $\{A,B,C\}\rightarrow D$, $\TIN(A,BCD)=1$ and $\SH^*_{\Z,\Y}=D$, while $\TIN(A,BC)=0$ ($\Y' = BC$) and $\SH^*_{\Z,\Y}=\emptyset$. $\TIN(A,ABD)=1$ with $\SH^*_{\Z,\Y}=A$, while $\TIN(A,BD)=1$ with $\SH^*_{\Z,\Y}=D$ (though $A$ is still a minimum vertex cut, it is not critical).
\item[4.] Any transformation vector $\omega\in\Omega_{\Z;\Y'}$ is also in $\OmegaZY$, with the other $\Y\backslash\Y'$ indices set to zero.
\item[5.] It does not yield that for $\Y'$ with $|\Y'|\leq k$ there exists no non-zero vector $\w$ to make $\w^\intercal\Y' \indep \Z$ (so in condition 2) it is ``$\exists \Y'\subseteq \Y$'').
\item[6.] \cref{thm:equiv_gin_Y_subsets} can help the estimation of $\OmegaZY$ (existence is easier to check than dimension of all), and can also help the prunning process when we need to test over $\Y$ with size from big to small (to find latent clusters).
\end{itemize}
\end{remark}

Then, with a same $\Y$ but different $\Z$, we also have the following properties:

\begin{lemma}[Subset of whole independence set]\label{subset_Z_indpdt}
For two variables sets $\Z$ and $\Y$ and their respective $\TIN(\Z,\Y)$, denote $\operatorname{Ind}_{\Z,\Y}\coloneqq \{i | \omega^\intercal\Y\indep X_i\}$. From~\cref{full_independence} we have $\operatorname{Ind}_{\Z,\Y} = \{i | \Ancout{\SH^*_{\Z,\Y}}(\Y) \cap \Anc(\{i\}) = \emptyset\}$. Then, $\forall \Z'\subseteq \operatorname{Ind}_{\Z,\Y}$, $\TIN(\Z',\Y)\leq \TIN(\Z,\Y)$. Specifically, if $\Z\subseteq\Z'$, then $\TIN(\Z',\Y) = \TIN(\Z,\Y)$, and moreover, the independent linear transformation subspace is the same: $\OmegaZY = \Omega_{\Z';\Y}$, and the critical vertex cut over all such $\Z'$ is also the same as $\SH^*_{\Z,\Y}$.
\end{lemma}

More properties about subset implications (e.g., combination and expansion of $\Z$ and more independent variables) can be derived from e.g.,~\cref{full_independence}. Another interesting question is, except for pruning in practical algorithms or for easier estimation, how to use these subset implication relationships to help identify the graph structure?

\section{Methods Details for Estimating \texorpdfstring{$\OmegaZY$}{TEXT}}\label{app:estimate_omegaZY}
\subsection{For \texorpdfstring{$\operatorname{TIN-rank}$}{TEXT}: Stacked Cumulants}\label{app:stacked_cumulants}

To estimate the subspace $\Omega_{\Z;\Y}$, we give a method named ``ranks stopped increasing'' in~\cref{sec:stacked_cumulants_ranks_stop} based on cumulants among variables. Now we give more details on this method.

\begin{definition}[Cumulants~\citep{robeva2021multi}]\label{def:cumulants} Define cumulant among $k$ variables $X _ { i _ { 1 } } , \ldots , X _ { i _ { k } }$ as:
\begin{equation}\operatorname { cum } \left( X _ { i _ { 1 } } , \ldots , X _ { i _ { k } } \right) = \sum _ { \left( A _ { 1 } , \ldots , A _ { L } \right) } ( - 1 ) ^ { L - 1 } ( L - 1 ) ! \mathbb { E } \left[ \prod _ { j \in A _ { 1 } } X _ { j } \right] \mathbb { E } \left[ \prod _ { j \in A _ { 2 } } X _ { j } \right] \cdots \mathbb { E } \left[ \prod _ { j \in A _ { L } } X _ { j } \right],\end{equation}
where the sum is taken over all partitions $\left( A _ { 1 } , \ldots , A _ { L } \right)$ of the set $\{i_1,\ldots,i_k\}$.
\end{definition}

\begin{remark} About cumulant defined in~\cref{def:cumulants}:
\begin{itemize}[noitemsep,topsep=-3pt]
\item[1.] Suppose variables are zero-meaned, then sum is taken over all partitions where each $A_i$ has size at least 2. For example, in the following:
\item[2.] $\operatorname { cum }(X_i)=0$.
\item[3.] $\operatorname { cum }(X_{i_1}, X_{i_2}) = \mathbb { E }[X_{i_1} X_{i_2}] = \operatorname{cov}(X_{i_1}, X_{i_2})$.
\item[4.] $\operatorname { cum }(X_{i_1}, X_{i_2}, X_{i_3}) = \mathbb { E }[X_{i_1} X_{i_2} X_{i_3}] = \text{3rd order moment of }(X_{i_1}, X_{i_2}, X_{i_3})$.
\item[5.] $
    \begin{aligned}[t]\operatorname { cum }(X_{i_1}, X_{i_2}, X_{i_3}, X_{i_4}) = \mathbb { E }[X_{i_1} X_{i_2} X_{i_3} X_{i_4}] &- \mathbb { E }[X_{i_1} X_{i_2}]\mathbb{E}[ X_{i_3} X_{i_4}] \\
    &- \mathbb { E }[X_{i_1} X_{i_3}]\mathbb{E}[ X_{i_2} X_{i_4}] \\
    &- \mathbb { E }[X_{i_1} X_{i_4}]\mathbb{E}[ X_{i_2} X_{i_3}].
    \end{aligned} $
\item[6.] As is shown above, the 4-th order cumulant is not equal to the 4-th order momentum. In general, cumulant$\neq$momentum when order $k\geq 4$. We use cumulant, for reason in point 7:
\item[7.] If variables $X _ { i _ { 1 } } , \ldots , X _ { i _ { k } }$ are mutually independent, then $\operatorname { cum } \left( X _ { i _ { 1 } } , \ldots , X _ { i _ { k } } \right)=0$. Note that it is zero cumulant, not zero momentum.
\end{itemize}
\end{remark}

\begin{definition}[Cross cumulant tensor]\label{def:cross_cumulant_tensor} For a random vector $\X=[X_1,\cdots,X_m]^\intercal$, denote its cross cumulant tensor at order $k$ as $\mathcal{T}^{(k)}_\X$, an $\underbrace{m\times \cdots \times m}_{k\text{ times}}$ tensor, where each entry
\begin{equation}{\mathcal{T}^{(k)}_\X}_{i_1, \cdots, i_k} \coloneqq \operatorname{cum}(X_{i_1}, \cdots, X_{i_k}).\end{equation}
\end{definition}

Now suppose these random variables follow a linear acyclic SEM model, with $\X=\A\X+\E$. Because of acyclicity, we could also write $\X=\B\E$, where $\B=(\I-\A)^{-1}$. Then we have the following:

\begin{theorem}[Cross cumulant tensor in linear acyclic SEM]\label{thm:cross_cumulant_tensor_LSEM} $k$-th order cross cumulant tensor equals
\begin{equation}\mathcal{T}^{(k)}_\X = \mathcal{T}^{(k)}_\E \bigcdot \underbrace{\B \bigcdot \cdots \bigcdot \B}_{k\text{ times}},\end{equation}
where $\mathcal{T}^{(k)}_\E$ is the $k$-th order cross cumulant tensor of $\E$, and `$\bigcdot$' denotes the tensor dot, i.e.,
\begin{equation}{\mathcal{T}^{(k)}_\X}_{i_1, \cdots, i_k} = \sum_{j_1, \cdots, j_k}{{\mathcal{T}^{(k)}_\E}_{j_1, \cdots, j_k} \B_{i_1,j_1}\cdots \B_{i_k,j_k}}\end{equation}
Since exogenous noises $\E$ are mutually independent, $\mathcal{T}^{(k)}_\E$ is a diagonal tensor. In this case, the above equation needs not to be summed over all Cartesian product $[m]^k$, but just over each $j\in[m]$.
\end{theorem}

\begin{remark} About cross cumulant tensor in Linear acyclic SEM in~\cref{thm:cross_cumulant_tensor_LSEM}:
\begin{itemize}[noitemsep,topsep=-3pt]
\item[1.] For example, in 2nd order case, $\mathcal{T}^{(2)}_\X$ is the cross covariance matrix $\Sigma \coloneqq \operatorname{cov}(\X,\X)$. We have $\Sigma = \B \Phi \B^\intercal$, where $\Phi$ is a diagonal matrix with entries $\Phi_{i,i}=\operatorname{var}(E_i)$.
\item[2.] Proof to 1: for every two variables $X_i, X_j$, $\operatorname{cov}(X_i,X_j)=\sum_k{\B_{ik}\B_{jk}\operatorname{var}(E_k)}$.
\item[3.] Point 2 means that the covariance between $X_i, X_j$ is contributed by all noise that is contained in both $X_i$ and $X_j$. By `common noise', we mean `confounders', `common ancestors', or the `top-node' of each trek between $(X_i, X_j)$ - and this is the start of the proof to trek-separation.
\item[4.] In general, any order of the cumulant $\operatorname{cum}(X_{i_1}, \cdots, X_{i_k})$ is contributed by the `common noise' that $X_{i_1}, \cdots, X_{i_k}$ all share, i.e., $\bigcap_{l\in [k]} \Anc(X_{i_l})$, the common ancestors.
\end{itemize}
\end{remark}
Since we only care the pairwise relationship between any \textbf{two} subsets $\Z,\Y$, we can take a 2D matrix slice out from each order of cross cumulant tensors:
\begin{definition}[2D slice of cross cumulant tensor]\label{def:2d_cross_cumulant_tensor} For a random vector $\X$ with $k$-th order cross cumulant tensor $\mathcal{T}^{(k)}_\X$, denote its 2D matrix slice of $k$-th order cross cumulant tensor as $\mathcal{C}^{(k)}$, where
\begin{equation}{\mathcal{C}_{i, j}^{(k)}} \coloneqq \operatorname{cum}(\underbrace{X_{i}, \cdots, X_{i}}_{k-1\text{ times}}, X_j)={\mathcal{T}^{(k)}_\X}_{i, \cdots, i, j}.\end{equation}
\end{definition}

\begin{remark} About 2D slice of cross cumulant tensor defined in~\cref{def:2d_cross_cumulant_tensor}:
\begin{itemize}[noitemsep,topsep=-3pt]
\item[1.] For simplicity, here we omit the subscript $_{\X}$ in $\mathcal{C}_{\X}^{(k)}$ and just write as $\mathcal{C}^{(k)}$.
\item[2.] In particular, when $k=2$, $\mathcal{C}^{(2)}$ is the variance covariance matrix $\Sigma_\X$.
\item[3.] $\mathcal{C}^{(k)}$ is $n\times n$ matrix, and is not necessarily symmetric when $k>2$.
\end{itemize}
\end{remark}

Then similar to~\cref{thm:cross_cumulant_tensor_LSEM}, we formulate 2D slice of cross cumulant tensor in linear acyclic SEM:

\begin{theorem}[2D slice of cross cumulant tensor in linear acyclic SEM]\label{thm:2d_cross_cumulant_tensor_LSEM} $\mathcal{C}^{(k)}$ equals
\begin{equation}\label{Eq:2D_slice_cumulant_tensor}\mathcal{C}^{(k)} = \B^{k-1} \cdot \Phi_\E^{(k)} \cdot \B^\intercal,\end{equation}
where $\B^{k-1}$ is the element-wise power (i.e., $\B^{k-1} = \underbrace{\B \circ \cdots \circ \B}_{k-1\text{ times}}$, `$\circ$' denotes element-wise product (Hadamard product), and $\Phi_\E^{(k)}$ is a diagonal matrix with entries ${\Phi_\E^{(k)}}_{i,i}=\operatorname{cum}(\underbrace{E_i, \cdots E_i}_{k\text{ times}})$.

Moreover, for two vertices sets $\Z,\Y$, similar to~\cref{thm:characterization_Omega_ZY}, we have\begin{equation}
\begin{aligned}[t]\mathcal{C}_{\Z,\Y}^{(k)} &= \B_{\Z, :}^{k-1} \cdot \Phi_\E^{(k)} \cdot \B_{\Y, :}^\intercal \\
&= \B_{\Z, \Anc(\Z)}^{k-1} \cdot \Phi_\E^{(k)} \cdot \B_{\Y, \Anc(\Z)}^\intercal,
\end{aligned}\end{equation}
where e.g., $\mathcal{C}_{\Z,\Y}^{(k)}$ denotes the submatrix of $\mathcal{C}^{(k)}$ with rows indexed by $\Z$ and columns indexed by $\Y$.
\end{theorem}

Proof to~\cref{thm:2d_cross_cumulant_tensor_LSEM} is straightforward by plugging~\cref{def:2d_cross_cumulant_tensor} into tensor dot of~\cref{thm:cross_cumulant_tensor_LSEM}.

Since independence yields zero cumulant, we have that for two vertices sets $\Z,\Y$ and $\omega\in\mathbb{R}^{|\Y|}$, if $\omega^\intercal\Y \indep \Z$, then $\mathcal{C}_{\Z,\Y}^{(k)}\omega = 0$. In other words,\begin{equation}\Omega_{\Z;\Y}\subseteq \operatorname{null}(\mathcal{C}_{\Z,\Y}^{(k)}),\text{ for any }k\geq 2.\end{equation}
This can be shown by two ways: one is that $\operatorname{cum}(\Z,\cdots,\Z,\omega^\intercal\Y)=\operatorname{cum}(\Z,\cdots,\Z,\Y)\omega$, another is to use~\cref{Eq:Omega_ZY_characterization_general} we build in~\cref{thm:characterization_Omega_ZY}: $\omega^\intercal\Y\indep\Z \Leftrightarrow \B_{\Y,\Anc(\Z)}^\intercal\omega=0$.

We shall also recap the original $\GIN$ condition: first solve equation by $\operatorname{cov}(\Z,\Y)$, then check whether any solution $\omega$ satisfies $\omega ^\intercal \Y \indep \Z$ (i.e., whether $\operatorname{null}(\operatorname{cov}(\Z,\Y)) = \Omega_{\Z;\Y}$). However, when $\GIN$ is not satisfied (i.e., $\Omega_{\Z;\Y} \subsetneq \operatorname{null}(\operatorname{cov}(\Z,\Y))$), it is not necessarily that $\Omega_{\Z;\Y} = \mathbb{R}_{\mathbf{0}}$ - e.g., the rank may just be limited by the size of $\Z$. This is exactly the motivation why we need to further generalize $\GIN$ to $\TIN$: can we escape from the `unwanted restriction on rank' (e.g., size of $\Z$) and find exactly the $\Omega_{\Z;\Y}$? Fortunately, by above implication from independence to zero cumulant, we could solve equation not only by 2-nd order $\operatorname{cov}(\Z,\Y)$, but more (on any order) $\mathcal{C}_{\Z,\Y}^{(k)}$.

\begin{definition}[Stacked 2D slices of cumulants]\label{def:stacked_2d_slices_of_cumulants_detail} For two vertices sets $\Z,\Y$ and order $k\geq 2$, define:
\begin{equation}
\Psi_{\Z;\Y}^{(k)} \coloneqq \begin{bmatrix} 
    \mathcal{C}_{\Z,\Y}^{(2)}\\
    \vdots\\
    \mathcal{C}_{\Z,\Y}^{(k)}
    \end{bmatrix}\end{equation}
\end{definition}
$\Psi_{\Z;\Y}^{(k)}$ is a $(k-1)|\Z|\times |\Y|$ matrix that vertically stacks 2D cumulants slices between $\Z,\Y$ with order from 2 to $k$. Since independence yields zero cumulant, similarly we have \begin{equation}\Omega_{\Z;\Y}\subset \operatorname{null}(\Psi_{\Z;\Y}^{(k)}),\text{ for any }k\geq 2.\end{equation}

For example, a fully connected DAG with 4 variables $\{X_1,X_2,X_3,X_4\}^\intercal$, the edges parameters are:
\begin{equation}\A= \begin{bmatrix} 
    0 & 0 & 0 & 0\\
    a & 0 & 0 & 0\\
    b & d & 0 & 0\\
    c & e & f & 0
\end{bmatrix};\B= \begin{bmatrix} 
    1 & 0 & 0 & 0\\
    a & 1 & 0 & 0\\
    ad+b & d & 1 & 0\\
    a(df+e)+bf+c & df+e & f & 1
\end{bmatrix}\end{equation}

Denote cumulants of exogenous noises $\varphi_i^{(k)}\coloneqq \operatorname{cum}(\underbrace{E_i,\cdots, E_i}_{k\text{ times}})$.

1) Let $\Z\coloneqq\{X_1\}, \Y\coloneqq\{X_2, X_3, X_4\}$, we have:
\begin{equation}
\begin{aligned}
\Psi_{\Z;\Y}^{(2)}&= \begin{bmatrix} 
    a \varphi_1^{(2)} &  (a d + b)\varphi_1^{(2)} &  (a (d f + e) + b f + c)\varphi_1^{(2)}
\end{bmatrix};\\
\Psi_{\Z;\Y}^{(3)}&= \begin{bmatrix} 
    a \varphi_1^{(2)} &  (a d + b)\varphi_1^{(2)} &  (a (d f + e) + b f + c)\varphi_1^{(2)}\\
    a \varphi_1^{(3)} &  (a d + b)\varphi_1^{(3)} &  (a (d f + e) + b f + c)\varphi_1^{(3)}
\end{bmatrix};\cdots
\end{aligned}
\end{equation}
The independence subspace
\begin{equation}
\Omega_{\Z;\Y} = \operatorname{null}(\B_{\Y, Anc(Z)}^\intercal) = \operatorname{null}\left( \begin{bmatrix} 
    a & ad+b & a(df+e)+bf+c
\end{bmatrix}\right),\text{ dimension=2}.
\end{equation}

Observe that $\operatorname{null}(\Psi_{\Z;\Y}^{(2)}) = \Omega_{\Z;\Y}$, (and also $=\operatorname{null}(\Psi_{\Z;\Y}^{(3)})=\cdots$).

2) Let $\Z\coloneqq\{X_2\}, \Y\coloneqq\{X_1, X_3, X_4\}$, we have:
\begin{equation}
\begin{aligned}
\Psi_{\Z;\Y}^{(2)}&= \begin{bmatrix} 
    a \varphi_1^{(2)} & a (a d + b) \varphi_1^{(2)} + d \varphi_2^{(2)} & a (a (d f + e) + b f + c) \varphi_1^{(2)}  + (d f + e) \varphi_2^{(2)} 
\end{bmatrix};\\
\Psi_{\Z;\Y}^{(3)}&= \begin{bmatrix} 
    a \varphi_1^{(2)} & a (a d + b) \varphi_1^{(2)} + d \varphi_2^{(2)} & a (a (d f + e) + b f + c) \varphi_1^{(2)}  + (d f + e) \varphi_2^{(2)} \\
    a^2 \varphi_1^{(3)} & a^2 (a d + b)\varphi_1^{(3)} + d \varphi_2^{(3)} & a^2 (a (d f + e) + b f + c) \varphi_1^{(3)} + (d f + e) \varphi_2^{(3)}
\end{bmatrix};\cdots
\end{aligned}
\end{equation}
The independence subspace
\begin{equation}
\Omega_{\Z;\Y} = \operatorname{null}(\B_{\Y, Anc(Z)}^\intercal) = \operatorname{null}\left( \begin{bmatrix} 
    1 & ad+b & a(df+e)+bf+c \\
    0 & d & df+e
\end{bmatrix}\right),\text{ dimension=1}.
\end{equation}
Clearly $\operatorname{null}(\Psi_{\Z;\Y}^{(2)})\neq \Omega_{\Z;\Y}$, since the rank of $\Psi_{\Z;\Y}^{(2)}$ is only 1. However, as long as there is no parameter coupling in cumulants, or specifically, $\frac{a \varphi_1^{(3)}}{\varphi_1^{(2)}} \neq \frac{\varphi_2^{(3)}}{\varphi_2^{(2)}}$, then $\operatorname{null}(\Psi_{\Z;\Y}^{(3)}) = \Omega_{\Z;\Y}$ (with the rank increasing to 2). We could verify the solution:
\begin{equation}\begin{aligned}
\omega^\intercal\Y &= \frac{cd-be}{d}E_1 + \frac{df+e}{d}((ad+b)E_1+dE_2+E_3) \\ &\ \ \ \ \ \ \ \ \ \ \ \ \ \ \ \ \ \ \ \ \ \ \ \ - ((a(df+e)+bf+c)E_1 + (df+e)E_2 + fE_3 + E_4) \\
&= \text{contains only }\{E_C, E_D\}\text{, and thus }\omega^\intercal\Y \indep \Z.
\end{aligned}\end{equation}
According to original $\GIN$ definition, there is only $\GIN(\{X_1\},\{X_2,X_3,X_4\})$, and $X_2,X_3,X_4$ cannot be distinguished. However here by using $\TIN$, we could also identify $X_2$.

3) Let $\Z\coloneqq\{X_3\}$ or $\{X_4\}$, $\Y\coloneqq \X\backslash\Z$, there is no non-zero $\omega$ s.t., $\omega^\intercal\Y \indep \Z$. Observe that:
$$\Omega_{\Z;\Y} = \mathbb{R}_\mathbf{0} = \operatorname{null}(\Psi_{\Z;\Y}^{(k)})= \cdots = \operatorname{null}(\Psi_{\Z;\Y}^{(4)}) \subsetneq \operatorname{null}(\Psi_{\Z;\Y}^{(3)}) \subsetneq \operatorname{null}(\Psi_{\Z;\Y}^{(2)}).$$

Above example gives us a motivation to use a sequence of stacked 2D cumulants $\{\Psi_{\Z;\Y}^{(i)}\}_{i=2,3,\cdots}$.

\begin{remark} About this sequence of stacked 2D cumulants, we have:
\begin{itemize}[noitemsep,topsep=-3pt]
\item[1.] $\Psi_{\Z;\Y}^{(i+1)}$ contains $\Psi_{\Z;\Y}^{(i)}$ as some-rows-indexed submatrix, so:
\item[2.] Rank does not drop, i.e., $\operatorname{rank}(\Psi_{\Z;\Y}^{(i+1)})\geq \operatorname{rank}(\Psi_{\Z;\Y}^{(i)})$. 
\item[3.] Nullspaces $\operatorname{null}(\Psi_{\Z;\Y}^{(i+1)})\subseteq \operatorname{null}(\Psi_{\Z;\Y}^{(i)})$.
\item[4.] Independent subspace $\Omega_{\Z;\Y}\subseteq \operatorname{null}(\Psi_{\Z;\Y}^{(i)})$, for any $k\geq 2$.
\item[5.] $\operatorname{rank}(\Psi_{\Z;\Y}^{(i)}) \leq |\Y|-\operatorname{dim}(\Omega_{\Z;\Y})$, for any $k\geq 2$.
\end{itemize}
\end{remark}

Note that in above statements, no assumptions on edge parameters and noise components' cumulants are made, and they are purely by definition. Then, does there exist a finite integer $K\in\mathbb{N}^+$ where the shrinking nullspaces stop hereafter at $\Omega_{\Z;\Y}$, i.e., $\Omega_{\Z;\Y} = \operatorname{null}(\Psi_{\Z;\Y}^{(i)})$, for any $i\geq K$? The answer is yes, under the generic assumptions on edge parameters and noise components' cumulants:
\begin{assumption}[Generic edge parameters and noise components' cumulants]\label{assum:generic} On a LiNGAM instance $\mathbf{L}=\mathcal{G}(G,\B,\E)$ defined by graph structure $G$, edge parameters $\B$ and noise components $\E$, assume that for any two variables sets $\Z,\Y$ and order $k\geq 2$, \begin{equation}\operatorname{rank}(\Psi_{\Z;\Y}^{(k)};\mathbf{L})=\max_{\B',\E'} \{\operatorname{rank}(\Psi_{\Z;\Y}^{(k)};\mathbf{L}') \ |\ \mathbf{L}'= \mathcal{G}(G,\B',\E')\},\end{equation}
where $\B',\E'$ are traversed over the whole edge parameters and noise components space. This is to assume that there is no coincidental low rank parameterized by the LiNGAM instance $\mathbf{L}$. Note that~\cref{assum:generic} is stronger than~\cref{assum:rank_faithfulness} in~\cref{sec:approach}. Here~\cref{assum:generic} assumes not only generic edge parameters, but also noise parameters.
\end{assumption}

Under~\cref{assum:generic} we have the following graphical criteria over stacked 2D cumulants:

\begin{theorem}\label{thm:graph_criteria_cumulant}
For two vertices sets $\Z,\Y$ and order $k\geq 2$, we define a new DAG associated with $G$, denoted as $\hat{G}^{(k)}$, which has $kn$ vertices $\{1,2,\cdots,n\}\cup \{1^{(2)},2^{(2)},\cdots,n^{(2)}\}\cup\cdots\cup\{1^{(k)},2^{(k)},\cdots,n^{(k)}\}$ with edges $i\rightarrow j$ if $i\rightarrow j$ is in $G$, $\{j^{(l)}\rightarrow i^{(l)}\}_{l=2,\cdots,k}$ if $i\rightarrow j$ is in $G$, and $\{i^{(l)}\rightarrow i\}_{l=2,\cdots,k}$ for $i\in [n]$. Define a new vertices set $\Z'\coloneqq \cup \{i^{(2)}, \cdots, i^{(k)}\}_{i\in\Z}$, then we have:
\begin{equation}\operatorname{rank}(\Psi_{\Z;\Y}^{(k)}) = \min\{|\SH| \ | \ \SH \text{ is a vertex cut from } \Z'\text{ to }\Y \text{ on } \hat{G}^{(k)}\}.\end{equation}
\end{theorem}
Note that the trek-separation theorem can be viewed as a special case of~\cref{thm:graph_criteria_cumulant} with $k=2$, where ``$(\SH_\W,\SH_\Y)$ t-separates $(\W,\Y)$'' is equivalent to ``$\SH'_\W\cup \SH_\Y$ vertex cuts $\W'$ to $\Y$''. The proof to~\cref{thm:graph_criteria_cumulant} also basically follow the proof to Theorem 2.8 in~\citep{sullivant2010trek}: using the Lindstr{\"o}m-Gessel-Viennot theorem~\citep{lindstrom1973vector,gessel1985binomial}, the max-flow min-cut theorem (vertex version, known as Menger's theorem)~\citep{dantzig1955max,bondy1976graph,menger1927allgemeinen}, and applying the Cauchy–Binet determinant expansion formula and Schur properties repeatedly on the Hadamard products in~\cref{Eq:2D_slice_cumulant_tensor}.

With the graphical criteria stated in~\cref{thm:graph_criteria_cumulant} and under generic~\cref{assum:generic}, we could have a method to implement $\operatorname{TIN}$ by ranks of stacked cumulants in sequence:
\begin{theorem}[Use ranks' stopped increasing to implement $\operatorname{TIN}$] For two variables sets $\Z,\Y$, there must exists a finite order $k\geq 2$ s.t.\begin{equation}\operatorname{rank}(\Psi_{\Z;\Y}^{(k+1)}) = \operatorname{rank}(\Psi_{\Z;\Y}^{(k)}).\end{equation}
Moreover, this one-step-stop yields an infinite-steps-stop, i.e.,
\begin{equation}\operatorname{rank}(\Psi_{\Z;\Y}^{(l)}) = \operatorname{rank}(\Psi_{\Z;\Y}^{(k)}),\text{ for any }l>k.\end{equation}
and, this stopped-increasing rank equals exactly to $\operatorname{TIN}(\Z,\Y)$, i.e.,
s.t.\begin{equation}\operatorname{rank}(\Psi_{\Z;\Y}^{(k)}) = \operatorname{TIN}(\Z,\Y) = |\Y|-\operatorname{dim}(\Omega_{\Z;\Y}).\end{equation}
\end{theorem}

The original $\GIN$ condition using only covariance matrix can be viewed as a special case, which could be implemented as ``$\operatorname{rank}(\Psi_{\Z;\Y}^{(2)}) = \operatorname{rank}(\Psi_{\Z;\Y}^{(3)})$''.

Note that independence test is not used in this method. We could also use independence tests to test whether $\operatorname{null}(\Psi_{\Z;\Y}^{(k)})$
is equal to $\OmegaZY$, just like the 2-steps method in~\cref{sec:two_steps_estimation}. Independence yields zero cumulants, and also yields independence among functions of variables. Hence in term of solving equations system, $\operatorname{null}(\Psi_{\Z;\Y}^{(k)})$ and $\cov(f(\Z),\Y)\w = \mathbf{0}$ are both correct. However, the latter does not have additional graphical criteria as~\cref{thm:graph_criteria_cumulant}. Empirically, the latter performs better, since higher order cumulants yield higher order exponential, which is sensitive to outliers.

\subsection{For \texorpdfstring{$\operatorname{TIN-2steps}$}{TEXT}, \texorpdfstring{$\operatorname{TIN-subsets}$}{TEXT}, and \texorpdfstring{$\operatorname{TIN-ISA}$}{TEXT}}
Implementation details for these three methods can be referred in~\cref{app:implementation_details}. Specifically, $\operatorname{TIN-ISA}$ directly follows the Independent Subspace Analysis (ISA) from the original paper~\citep{theis2006towards}.

\section{Discussions}\label{app:discussions}
\subsection{Details on Assumptions}\label{app:assumptions}
In this paper, except for the LiNGAM assumption for the causal model, we also give~\cref{assum:rank_faithfulness} in~\cref{sec:approach}:

\RANKFAITHFULNESS*

Roughly speaking, \cref{assum:rank_faithfulness} assumes there are no edge parameter couplings to produce coincidental low rank. Note that violation of \cref{assum:rank_faithfulness} is of Lebesgue measure 0, and LiNGAM is testable. Here we discuss more details on~\cref{assum:rank_faithfulness} by two examples of violation:

\begin{equation}\label{Eq:violation1}
    \begin{tikzpicture}[scale=0.6, line width=0.5pt, inner sep=0.2mm, shorten >=.1pt, shorten <=.1pt]
		\draw (0, 0) node(1)  {{\footnotesize\,$X_1$\,}};
		\draw (2, 0) node(2)  {{\footnotesize\,$X_2$\,}};
		\draw (4, 0) node(3)  {{\footnotesize\,$X_3$\,}};
		\draw[-arcsq] (1) -- (2) node[pos=0.5, above=2pt] {\scriptsize{$a$}}; 
		\draw[-arcsq] (2) -- (3) node[pos=0.5, above=2pt] {\scriptsize{$b$}}; 
		\draw[-arcsq] (1) to [bend right=25] (3);
		\node[below = of 2,xshift=0em,yshift=5.0ex] {\scriptsize{$c$}};
	\end{tikzpicture}
\end{equation}

\underline{\textit{Violation example 1}}: Consider the graph in~\cref{Eq:violation1}, if coincidentally the edge weights $c=-ab$, then the noise components $E_1$ will be cancelled from $X_3$, and marginally $X_1\indep X_3$. In this violation, graphically $\Anc(X_3)=\{X_1,X_2,X_3\}$, but the column indices of $\B_{X_3,:}$ with non-zero entries is just $\{X_2,X_3\}$.

\begin{equation}\label{Eq:violation2}
    \begin{tikzpicture}[anchor=base, baseline=6ex, scale=0.6, line width=0.5pt, inner sep=0.2mm, shorten >=.1pt, shorten <=.1pt]
        \draw (2, 3) node(1)  {{\footnotesize\,$X_1$\,}};
		\draw (2, 0) node(2)  {{\footnotesize\,$X_2$\,}};
		\draw (0, 1.5) node(3)  {{\footnotesize\,$X_3$\,}};
		\draw (4, 1.5) node(4)  {{\footnotesize\,$X_4$\,}};
		\draw[-arcsq] (2) -- (3) node[pos=0.5, below=2pt] {\scriptsize{$a$}}; 
		\draw[-arcsq] (2) -- (4) node[pos=0.5, below=2pt] {\scriptsize{$b$}}; 
		\draw[-arcsq] (1) -- (3) node[pos=0.5, above=2pt] {\scriptsize{$c$}}; 
		\draw[-arcsq] (1) -- (4) node[pos=0.5, above=2pt] {\scriptsize{$d$}}; 
		\draw[-arcsq] (3) -- (4) node[pos=0.5, above=2pt] {\scriptsize{$e$}}; 
	\end{tikzpicture}; \ 
	\A= \begin{bmatrix} 
    0 & 0 & 0 & 0\\
    0 & 0 & 0 & 0\\
    c & a & 0 & 0\\
    d & b & e & 0
\end{bmatrix}; \ \B= \begin{bmatrix} 
    1 & 0 & 0 & 0\\
    0 & 1 & 0 & 0\\
    c & a & 1 & 0\\
    ce+d & ae+b & e & 1
\end{bmatrix}
\end{equation}
\underline{\textit{Violation example 2}}: Consider the graph in~\cref{Eq:violation2}. Let $\Z\coloneqq\{X_1,X_2\}$ and $\Y\coloneqq \{X_3,X_4\}$, by the graphical criteria we have $\TIN(\Z,\Y)=2$, with the critical vertex cut $\SH_{\Z;\Y}^* = \{X_3,X_4\}$. Mathematically, $\B_{\Y;\operatorname{nzcol}(\B_{\Z,:})} = \begin{bmatrix}c & a\\ ce+d & ae+b \end{bmatrix}$, which has rank $2$ under generic parameters choice. However, if $bc=ad$, then coincidentally the rank will drop to $1$, and thus~\cref{assum:rank_faithfulness} is violated. Note that in this violation example, there is no noise cancelling (like violation example 1), i.e., here $\operatorname{nzcol}(\B_{\Z,:}$ is exactly $\Anc(\Z)$, but there is still coincidental low rank by parameter coupling.

Now we further discuss an example where~\cref{assum:rank_faithfulness} is satisfied (and thus is a valid case in this paper), but is not a valid case in the trek-separation paper~\citep{sullivant2010trek} or the $\GIN$ paper~\citep{xie2020generalized}:

\begin{equation}\label{Eq:violation3}
    \begin{tikzpicture}[anchor=base, baseline=6ex, scale=0.6, line width=0.5pt, inner sep=0.2mm, shorten >=.1pt, shorten <=.1pt]
        \draw (0, 2) node(1)  {{\footnotesize\,$X_1$\,}};
		\draw (1, 0.7) node(2)  {{\footnotesize\,$X_2$\,}};
		\draw (2,2) node(3)  {{\footnotesize\,$X_3$\,}};
		\draw (4, 2) node(4)  {{\footnotesize\,$X_4$\,}};
		\draw[-arcsq] (1) -- (2) node[pos=0.4, left=2pt] {\scriptsize{$1$}};
		\draw[-arcsq] (1) -- (3) node[pos=0.5, above=2pt] {\scriptsize{$2$}}; 
		\draw[-arcsq] (2) -- (3) node[pos=0.2, right=2.0pt] {\scriptsize{$-1$}}; 
		\draw[-arcsq] (3) -- (4) node[pos=0.4, above=2pt] {\scriptsize{$1$}}; 
	\end{tikzpicture}; \ 
	\A= \begin{bmatrix} 
    0 & 0 & 0 & 0\\
    1 & 0 & 0 & 0\\
    2 & -1 & 0 & 0\\
    0 & 0 & 1 & 0
\end{bmatrix}; \ \B= \begin{bmatrix} 
    1 & 0 & 0 & 0\\
    1 & 1 & 0 & 0\\
    1 & -1 & 1 & 0\\
    1 & -1 & 1 & 1
\end{bmatrix}
\end{equation}
\underline{\textit{Satisfaction example 3}}: Consider the graph in~\cref{Eq:violation2}. For every pair of $\Z,\Y$, there is no coincidental low rank in $\B_{\Y,\Anc(\Z)}$. Hence, \cref{assum:rank_faithfulness} is satisfied. E.g., let $\Z\coloneqq \{X_2\}, \Y\coloneqq \{X_3,X_4\}$, by the graphical criteria $\TIN(\Z,\Y)=1$ (with $\SH_{\Z;\Y}^* = \{X_3\}$), and $\B_{\Y,\Anc(\Z)}$ is also of rank $1$. However, if we carefully choose noise components' parameters so that the variance of exogenous noise $E_1$ and $E_2$ are equal ($\var(E_1) = \var(E_2)$), then the variance-covariance matrix $\cov(\{X_2\}, \{X_3,X_4\})$ would be $[0 \quad 0]$ (coincidentally dropped to rank $1$). This coincidental low rank is due to noise parameters, and will not affect our proposed method in this paper, because we directly find $\OmegaZY$. However, e.g., in $\GIN$ where $\w$ is characterized by 2nd-order variance-covariance matrix, by solving equation here, any $w\in\Real^2$ is a solution. Then, not every linear combination of $X_3$ and $X_4$ is independent to $X_2$, so $\GIN$ will output `$\GIN(\Z,\Y)$ violated' in this case, though according to the graphical criteria, $\GIN(\Z,\Y)$ is satisfied here.

\subsection{More than Ordered Group Decomposition can be Identified}\label{app:more_than_ordered_group_decomp}
In this paper, we use the $\TIN$ condition to identify the ordered group decomposition of $\Gtilde$ in the measurement error model. Specifically, we only use a special type of $\TIN$, one-and-others (\cref{lem:egin_one_and_rest}). However, actually by using the $\TIN$ condition over more general pairs of $\Z,\Y$, more information of $\Gtilde$ can be identified.

For example, in the chain structure (\cref{fig:chain}) and the fully connected DAG (\cref{fig:fully_connected}), the ordered group decomposition are both $\{\Xt_1\}\rightarrow \{\Xt_2\}\rightarrow \cdots \rightarrow\{\Xt_{n-2}\}\rightarrow\{\Xt_{n-1}, \Xt_n\}$. However, the two can actually be distinguished: In the fully connected DAG, $\operatorname{TIN}(\{X_2\}, \{X_3,\cdots,X_n\})=2$, while in the chain structure, $\operatorname{TIN}(\{X_2\}, \{X_3,\cdots,X_n\})=1$. Even under a same pair of $\Z,\Y$, the $\w$ degeneration may be different. E.g., $\TIN(\{X_2\}, \{X_1,X_3,X_4,\cdots\}) = 2$ in both graphs. However, in the chain structure, $\w$ is degenerated on the index $X_1$ (i.e., the linear combination of $\Y$ cannot include $X_1$. If $\omega_1 X_1 + \omega_3 X_3 + \omega_4 X_4 + \cdots$ is independent to $X_2$, then $\omega_1$ must be zero), while there is no degeneration of $\w$ in the fully connected DAG.

Generally speaking, our final objective is to identify an ``equivalence class'' of $\Gtilde$ w.r.t. the $\TIN$ condition. We have talked about the concept of ``unidentifiable'' in~\cref{sec:motivation}. Here, two graphs (either non-isomorphic or isomorphic but with labelling permutation) are unidentifiable w.r.t. the $\TIN$ condition, if and only if for any two pairs $\Z,\Y$, $\TIN(\Z,\Y)$ are same (with same degeneration).

About ``equivalence class'', we already knew some features that an equivalence class should possess, e.g., a variable is naturally unidentifiable with its pure leaf child in $\tilde{G}$ (see~\cref{def:pure_leaf_child}). Apparently, there are more such features to be discovered. Here are some of the examples:

\begin{example}[Equivalence class for the chain structure]
Consider a chain structure with $5$ nodes $\tilde{X}_1\rightarrow \cdots \rightarrow \tilde{X}_5$, and the following graphs with $5$ nodes:

\begin{enumerate}[noitemsep,topsep=-3pt]
\item $5$ edges: $\tilde{X}_1\rightarrow \cdots \rightarrow \tilde{X}_5$, with an additional $\tilde{X}_3\rightarrow\tilde{X}_5$.
\item $5$ edges: $\tilde{X}_1\rightarrow \cdots \rightarrow \tilde{X}_5$, with an additional $\tilde{X}_2\rightarrow\tilde{X}_4$.
\item $5$ edges: $\tilde{X}_1\rightarrow \cdots \rightarrow \tilde{X}_5$, with an additional $\tilde{X}_1\rightarrow\tilde{X}_3$.
\item $6$ edges: $\tilde{X}_1\rightarrow \cdots \rightarrow \tilde{X}_5$, with additional $\tilde{X}_1\rightarrow\tilde{X}_3$ and $\tilde{X}_3\rightarrow\tilde{X}_5$.
\end{enumerate}

For these five non-isomorphic graphs, with two equivalent permutations of each (swap the labeling of $\tilde{X}_4$ and $\tilde{X}_5$) - these $10$ graphs form an equivalence class. One might be curious: what if a graph with one more edge, i.e.,

\begin{enumerate}[noitemsep,topsep=-3pt]
\item[5.] $7$ edges: edges: $\tilde{X}_1\rightarrow \cdots \rightarrow \tilde{X}_5$, with additional $\tilde{X}_1\rightarrow\tilde{X}_3$, $\tilde{X}_2\rightarrow\tilde{X}_4$ and $\tilde{X}_3\rightarrow\tilde{X}_5$.
\end{enumerate}

However, this graph is no longer in the equivalence class. For example, $\operatorname{TIN}(X_3, X_{4,5})=1$ for the chain structure (and its equivalence class), while $\operatorname{TIN}(X_3, X_{4,5})=2$ for this graph.
\end{example}

\begin{example}[An equivalence class with one unique graph] Consider a $\tilde{G}$ with $5$ nodes and $7$ edges: $\tilde{X}_1\rightarrow \{\tilde{X}_2, \tilde{X}_3, \tilde{X}_4\}$, $\tilde{X}_2\rightarrow\{\tilde{X}_3, \tilde{X}_5\}$, and $\tilde{X}_3\rightarrow\{\tilde{X}_4, \tilde{X}_5\}$: surprisingly, its equivalence class contains only one graph, itself. I.e., by $\TIN$ conditions this structure should be uniquely recovered.
\end{example}

With the equivalence class, the identifiability result could be improved, and constrained O-ICA may be further applied to identify a final graph. It would be an interesting future work to characterize the ``equivalence class'' w.r.t. $\TIN$, and then design an algorithm to identify it.

\subsection{More than Dimension of \texorpdfstring{$\OmegaZY$}{TEXT}: Parameters}\label{app:more_than_dimension}
Currently we only care about the \textit{dimension} of the independent subspace $\Omega_{\Z;\Y}$, but not the exact \textit{parameters}. If we have obtained exactly the $\Omega_{\Z;\Y}$, we could write its basis matrix $M_{\Omega_{\Z;\Y}}$ in shape $|\Y| \times \operatorname{dim}(\Omega_{\Z;\Y})$, with each column vector being a basis. Then, the subspace spanned by row vectors of $\B_{\Y, \Anc(\Z)}$, which reflects edge parameters, is exactly the left nullspace of $M_{\Omega_{\Z;\Y}}$.

The degeneration of $\w$ we discussed in~\cref{app:more_than_dimension,graph_criteria_degeneration} is actually a special case of recovering information from $\OmegaZY$ parameters. For edge parameters, it means that rank of $\B_{\Y, \Anc(\Z)}$ will drop one if deleting the respective degenerated columns in $\Y$. More general exploitation of $\OmegaZY$ parameters is an interesting future work.

\subsection{More Possible Solutions for Estimation of \texorpdfstring{$\OmegaZY$}{TEXT}}\label{app:more_possible_solution_estimate_OmegaZY}
In~\cref{sec:estimation} we propose four methods to estimate $\OmegaZY$: tackling down to subsets of $\Y$ (\cref{sec:tackle_down_Y_subsets}), constrained independent subspace analysis (ISA) (\cref{sec:ISA_for_estimation}), stacked cumulants’ ranks stopped increasing (\cref{sec:stacked_cumulants_ranks_stop}), and $\TIN$ in two steps: solve equations, and then test for independence (\cref{sec:two_steps_estimation}). Generally, reliable estimation of $\OmegaZY$ can be formulated as an orthogonal research problem, and we believe that there exists more solutions.

For example, if we only care about the dimension of $\Omega_{\Z;\Y}$, the following heuristic method might help. The intuition is that, uniformly sample infinite many random points on the surface of a unit sphere (centered on origin point) at $\mathbb{R}^n$, denote $d^{(k)}$ the average distance from these points to a subspace in $\mathbb{R}^n$ with dimension $k$ ($0\leq k \leq n$). Then this average distance is monotonic over $k$: $d^{(k_1)}<d^{(k_2)}$ if and only if $k_1>k_2$. For example, on an 2D circle, $d^{(0)}=1$ (to center; radius), $d^{(1)}=2/\pi$ (to diameter), and $d^{(2)}=0$ (already on 2D); on a 3D sphere surface, $d^{(0)}=1$ (to center; radius), $d^{(1)}=\pi/4$ (to diameter), $d^{(2)}=1/2$ (to diameter plane), and $d^{(3)}=0$ (already on 3D).

If we assume the independence tests return a bool value (independent or not), then this method will not help, because generally, the measure of $\OmegaZY$ relative to $\Real^{|\Y|}$ is always zero. However, if we assume that, for a unit vector $\omega\in\mathbb{R}^{|\Y|}$, there exists a monotonic relationship between the independence strength of $\operatorname{Ind}(\omega^\intercal\Y; \Z)$ (e.g., mutual information) and the distance to the subspace $\operatorname{dist}(\omega; \Omega_{\Z;\Y})$, then we could have a non-parametric method to recover $\Gtilde$: for each variable $X_i$, uniformly sample many $\{\omega_l\}_{l=1,\cdots}$ from $\mathbb{R}^{(n-1)}$ and calculate the average independence $\operatorname{avg}_{l}{\operatorname{Ind}(\omega_l^\intercal [\X\backslash X_i]; X_i)}$, then sort $X_i$ by their respective average independence (i.e., dimensions of there respective $\Omega_{\Z;\Y}$) to get an estimation of the group ordering.

\subsection{\newcontent{What if Causal Sufficiency is Not Satisfied in \texorpdfstring{$\Gtilde$}{TEXT}?}}\label{app:causal_sufficiency_violated}

In this paper we assumed causal sufficiency relative to $\Xtilde$. Though it is reasonable to assume causal sufficiency in this context (which, to the best of our knowledge, is indeed a common assumption in the current literature of causal discovery with measurement error), this assumption itself, is a strong one and is not testable. Thus, it would be interesting to investigate the case where causal sufficiency is violated (in a sense of ``latents of latents''): Will $\TIN$-based method still output a correct ordering? If not, by which correction rules or algorithm relaxations can the identifiability be still partially preserved? We leave this as an interesting future work. For now, we try to provide some hints from examples (where we still use the~\cref{lem:egin_one_and_rest}-based method in this paper):

\begin{example}[A still (partially) identifiable case]
Consider a chain structure $\tilde{X}_1\rightarrow \tilde{X}_2 \rightarrow \cdots \rightarrow \tilde{X}_n$ (or similarly, a fully connected DAG) with a common hidden confounder $\tilde{L}$ pointing to them all: $\tilde{L} \rightarrow \{\tilde{X}_i\}_{i=1}^n$. If $\tilde{L}$ is not measured and only measurements $\mathbf{X}=\{X_1,\cdots,X_n\}$ are available, we have now: $\texttt{ord}(X_1)=\operatorname{TIN}(X_1, \mathbf{X}\backslash X_1)=2$, $\texttt{ord}(X_2)=3$, $\cdots$, $\texttt{ord}(X_{n-3})=n-2$, and $\texttt{ord}(X_{n-2})=\texttt{ord}(X_{n-1})=\texttt{ord}(X_{n})=n-1$. We shall see that: 1) The causal ordering of all but the last 3 variables is identifiable. While without $\tilde{L}$ (our previous result), this identifiability result is all but the last 2 variables (see~\cref{eg:review_egin_graphical_criteria}), and 2) the existence of hidden (root) confounder(s) will also be reported, since there is no root (with $\texttt{ord}=1$) found across measurements $\mathbf{X}$.
\end{example}

\begin{example}[A no-longer identifiable case]
Consider a simple fork $\tilde{X}_2\leftarrow \tilde{X}_1 \rightarrow \tilde{X}_3$, with a hidden confounder $\tilde{L}$: $\tilde{L}\rightarrow \tilde{X}_1$ and $\tilde{L}\rightarrow \tilde{X}_2$. Then, $\texttt{ord}(X_1)=\operatorname{TIN}(X_1, \mathbf{X}\backslash X_1)=2$, $\texttt{ord}(X_2)=1$, and $\texttt{ord}(X_3)=2$. Sorting by $\texttt{ord}$, we have the group decomposition as $\{\tilde{X}_2\}\rightarrow \{\tilde{X}_1, \tilde{X}_3\}$, while this is incorrect: there exists directed edge(s) from later groups to earlier groups, $\tilde{X}_1 \rightarrow \tilde{X}_2$.
\end{example}

\subsection{\newcontent{What if some Measurements are Caused by Multiple Latent Variables?}}
In this paper, we consider the measurement error model, where each measurement is caused by only one latent variable. For $\GIN$, it can generally handle the cases where measurements are caused by multiple latent variables, as long as each latent variable has enough pure indicators. Interestingly however, we find that this may also be relaxed for our case (where there are not enough pure indicators), and our $\TIN$-based method may still work (in identifying the correct group ordering). See below for some simple examples:

Consider a $3$-nodes chain structure $\tilde{A}\rightarrow \tilde{B} \rightarrow \tilde{C}$, and their respective measurements $A,B,C$. We have the ordered group decomposition $\{\tilde{A}\} \rightarrow \{\tilde{B}, \tilde{C}\}$, with $\texttt{ord}$ being $1$ and $2$. Then, what if we add an edge from a latent variable to another measured variable? There are $3\times 2=6$ ways of adding an edge. Surprisingly, among these $6$ ways, there are $5$ which preserves exactly the same $\TIN$ results over $A,B,C$. The only one difference is by adding $\tilde{C}\rightarrow A$, where $\TIN(A,BC)=2$, instead of $1$. It would be interesting to generalize this observation: What if more nodes? What if more edges?

\section{Implementation and Evaluation}\label{app:sec_eval_all}
\vspace{-1em}
\subsection{Implementation Details}\label{app:implementation_details}
In this section we provide the information required to reproduce our results reported in the main text. We also commit to making our implementations of $\TIN$ public.
\vspace{-0.5em}
\paragraph{Simulation setup} In simulation we consider specifically two cases: fully connected DAG (\cref{fig:fully_connected}) and chain structure (\cref{fig:chain}), of which the ordered group decomposition are both $\Xt_1\veryshortarrow\cdotsshort\veryshortarrow\Xt_{n-2},\Xt_{n-1,n}$. We consider $\Gtilde$ with the number of vertices $n=3,\cdotsshort,10$. Edges weights (i.e., the nonzero entries of matrix $\A$) are drawn uniformly from $[-0.9,-0.5]\cup[0.5,0.9]$. Exogenous noises $\Etilde$ are sampled from uniform $\cup [0,1]$ to the power of $c$, $c\sim\cup[5,7]$, and measurement errors are sampled from Gaussian $\mathcal{N}(0,1)$ to the power of $c$, $c\sim\cup[2,4]$. Sample size is $5,000$. Observations are generated by $X_i=\Xt_i+E_i$. To show the effect of measurement error, we simulate with noise-to-signal ratio $\operatorname{NSR}\coloneqq \var(E_i)/\var(\Xt_i)$ in $\{0.5,1,2,3,4\}$. On each configuration (under a graph type, measurement noise scaling, and the number of vertices), $50$ random graphs are generated for repeated experiments.
\vspace{-0.5em}
\paragraph{PC} We use the implementation from the \texttt{causal-learn} package\footnote{https://github.com/cmu-phil/causal-learn/blob/main/causallearn/search/ConstraintBased/PC.py}. Kernel-based conditional independence test~\citep{zhang2012kernel} is used. For speed consideration, datasets are downsampled to $1,000$ on PC runs. The significance level \texttt{alpha} is set to $0.05$. to~\cref{def:ordered_group_decomposition}.

\paragraph{GES} We use the implementation from the \texttt{causal-learn} package\footnote{https://github.com/cmu-phil/causal-learn/blob/main/causallearn/search/ScoreBased/GES.py}. The score used is \texttt{local-BIC-score}~\citep{schwarz1978estimating}.

\paragraph{Direct-LiNGAM and ICA-LiNGAM} We use the implementation from the \texttt{lingam} package\footnote{https://github.com/cdt15/lingam}. Note that for Direct-LiNGAM, actually the method used is based on pairwise likelihood ratios~\citep{hyvarinen2013pairwise}.

\paragraph{CAMME-OICA} We use the implementation from \texttt{LFOICA}\footnote{https://github.com/dingchenwei/Likelihood-free\_OICA} (Likelihood-Free Overcomplete ICA). It estimates the mixing matrix by first transforming random noise into components, and then mimic the mixing procedure from components to noise with MMD score as a metric.

Below we give details on implementations of $\TIN$. Specifically,

\paragraph{Independence test} We use the \texttt{HSIC} (Hilbert-Schmidt independence criterion) test~\citep{gretton2007kernel} with the implementation from \texttt{lingam} package~\footnote{https://github.com/cdt15/lingam/blob/master/lingam/hsic.py}. The kernel width is set to $0.1$ times the standard deviation of the data samples. The significance level \texttt{alpha} of p-value is set to $0.05$. Note that when the noise-to-signal ratio is large (e.g. $>3$), usually observed variables are already `independent enough' (i.e., with p-value given by HSIC test on raw data samples already $>0.05$). In this case, we use the difference of $\frac{1000*\text{severity}}{\text{sample size}}$ between $\Z;\Y$ and $\Z;\wTY$ to show how much independence is \textit{`gained'} by linear transformation. The threshold for this criterion is set to $0.5$.
 
\paragraph{TIN-ISA} We implement the constrained ISA where the de-mixing matrix is masked to only update the lower-right $|\Y|\times |\Y|$ block $\mathbf{W}_{\Y\Y}$, with upper-left $|\Z|\times |\Z|$ block fixed as the identity and elsewhere fixed as zero. We follow~\citep{pham1997blind} for the estimation of conditional score function. Independence between $\Z$ and $\wTY$ for each row of $\mathbf{W}_{\Y\Y}$ is then tested by HSIC test, as is described aboce.

\paragraph{TIN-rank} Numerical rank of a 2D matrix is calculated by SVD (singular value decomposition), with tolerance $\epsilon$ set to $0.005$. Singular values below threshold $T$ are considered zero, where $T=\epsilon*\operatorname{max}(S)*\max(M,N)$. $S$ is all singular values, and $M,N$ are shape of the 2D matrix. According to $\cref{thm:rank_stopped_increasing}$, we use the rank where stacked 2D slices of cumulants stops increasing rank as the output of $\TIN$.

\paragraph{TIN-2steps} To solve euqations system $\{\cov(f(\Z),\Y)\w = \mathbf{0}\}$, functions $f$ contain: $\Z$, $\Z^2$, $\Z^3$, $\vert \Z \vert$, $e^{\Z}$, $\log(|\Z|)$, $\sin(\Z)$, $\cos(\Z)$, $\operatorname{sigmoid}(\Z)$, $\tanh(\Z)$. Nullspace is calculated by SVD, while we do not set a hard threshold of singular value to determine its space (like TIN-rank). Instead, we test HSIC between $\Z$ and $\wTY$ for each $\w$ in the $|\Y| \times |\Y|$ unitary matrix $\mathbf{V}$, and count the number of independence achieved.

\paragraph{TIN-subsets} The core to find the existence of transformed independence is similar to TIN-2steps. Then, for the part of traversing over $\Y$'s subsets, ``all $\Y'$ ...'' and ``exists a $\Y'$ ...'' are characterized by $90\%$ and $10\%$ percentile of the independence statistics (e.g., p-value of HSIC test) respectively.

\paragraph{Noise synthesis} Edges weights (i.e., the nonzero entries of matrix $\A$) are drawn uniformly from $[-0.9,-0.5]\cup[0.5,0.9]$. Exogenous noises $\Etilde$ of the latent variables are sampled from uniform $\cup [0,1]$ to the power of $c$, $c\sim\cup[5,7]$, and measurement errors are sampled from Gaussian $\mathcal{N}(0,1)$ to the power of $c$, $c\sim\cup[2,4]$. Sample size is $5,000$. Below we show a synthetic dataset with $\Gtilde$ being a fully connected DAG with $n=7$, and the noise-to-signal ratio being $3$ (\cref{fig:plot_panels_correlation,fig:plot_noise_data}):

\begin{figure}[h]
\centering
\includegraphics[width=0.7\linewidth, ]{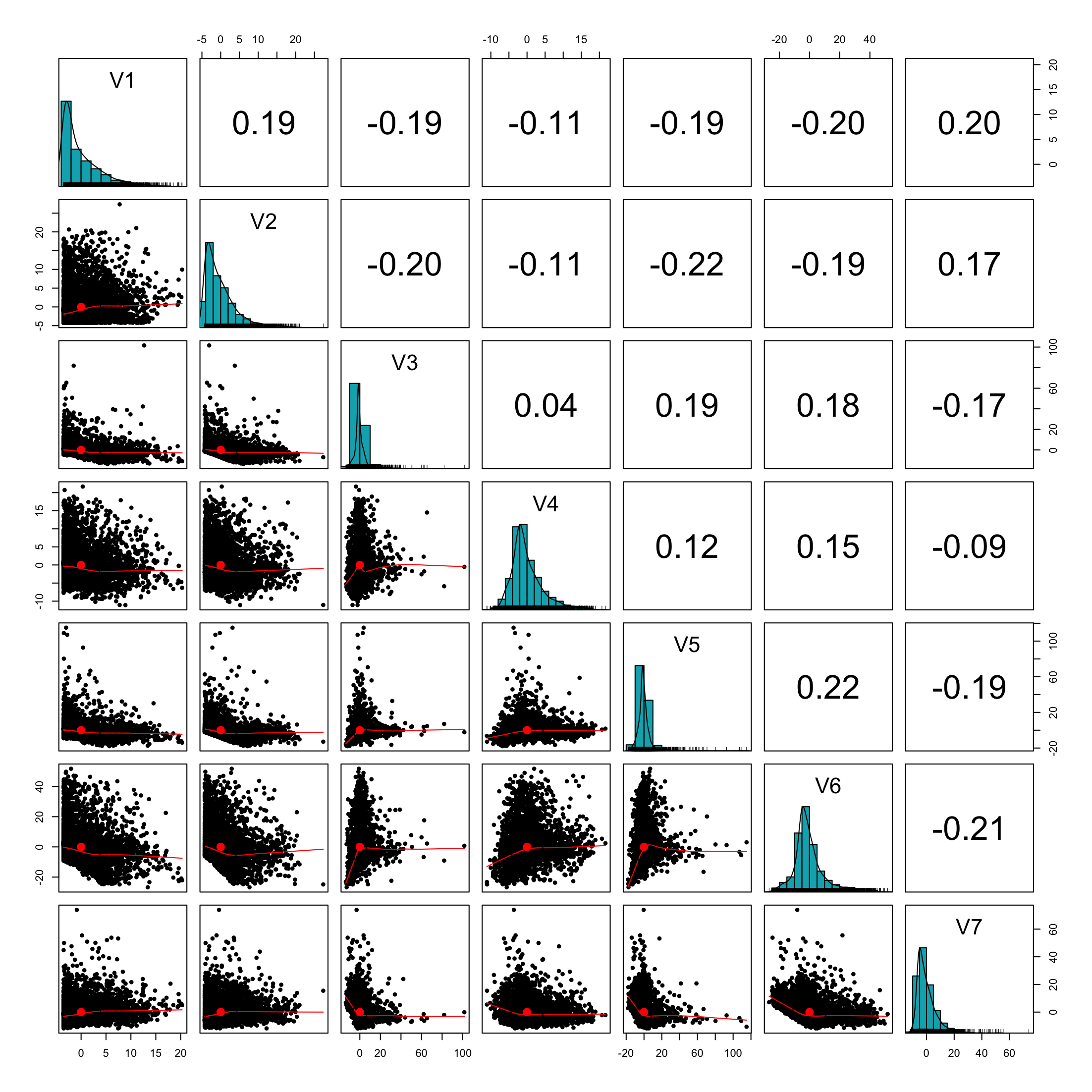}
\caption{The scatter plot matrix for seven observed variables in $\X$. From the first column we could see that, though $\Xt_1$ is a root variable in $\Gtilde$, regressing neither of $\{X_2,\cdots, X_7\}$ on $X_1$ will make the regression residual independent to the regressor $X_1$, due to the presence of measurement error.}
\label{fig:plot_panels_correlation}
\end{figure}

\begin{figure}[h]
\centering
\includegraphics[width=1.\linewidth, ]{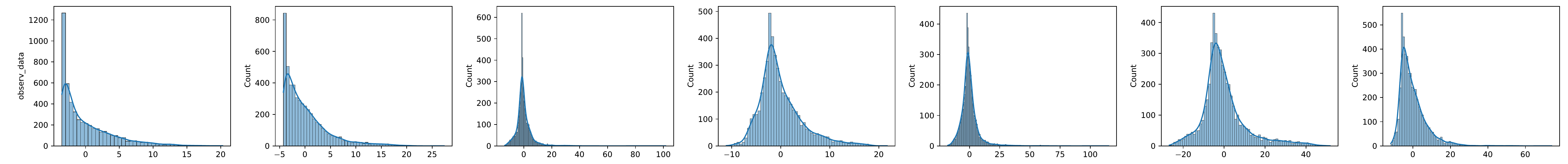}
\includegraphics[width=1.\linewidth, ]{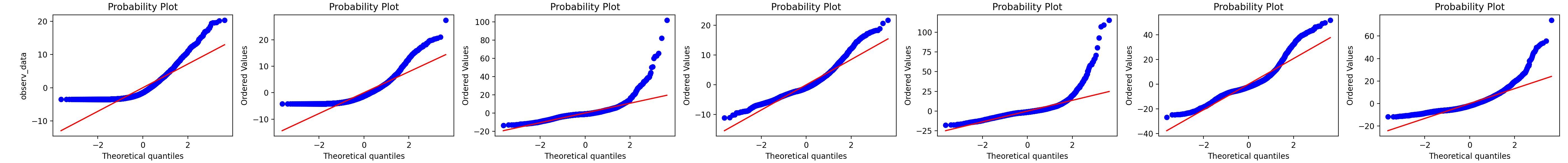}
\caption{The histogram plot and Q-Q plot to Gaussian distribution of seven observed variables in $\X$. We could see the non-Gaussianity of data.}
\label{fig:plot_noise_data}
\end{figure}

\subsection{Evaluation Details}\label{app:evaluation_details}
To evaluate the output group ordering, we use \textit{Kendall tau distance}~\citep{kendall1938new} to the ground-truth (in range $[0,1]$, the lower the better). Kendall tau distance counts the number of pairwise disagreements between two orderings. Specifically, for two variables $i,j$ and a grouped ordering $\tau$, we define:

\begin{equation}\label{Eq:cmp_func}
    \texttt{cmp}(i,j,\tau) = \begin{cases}
      1     & \text{ in } \tau, i \text{ is in an earlier group than } j\\
      0     & \text{ in } \tau, i \text{ is in a same group with } j\\
      -1    & \text{ in } \tau, i \text{ is in a later group than } j\\
   \end{cases}
\end{equation}

Then, for the $n$ variables $[n]$ and two ordered group decompositions $\tau_1 , \tau_2$ on them, the Kendall tau distance is defined as:
\begin{equation}\label{Eq:kt_dist}
    \texttt{ktdist}(\tau_1, \tau_2) = \frac{2}{n(n-1)} \sum_{i,j\in[n] , \ i < j} {\operatorname{sign} (\texttt{cmp}(i,j,\tau_1) \neq \texttt{cmp}(i,j,\tau_2))}
\end{equation}

For algorithms returning DAG/PDAG, its ordering is first extracted according to~\cref{def:ordered_group_decomposition}. Specifically, for PDAG, a vertex's ancestors is defined as all vertices that has mixed paths (no directed edges backward) to it.

For intuition, here we give some typical examples: if the true ordering is $\{\{X_1\}, \{X_2\}, \{X_3\}, \{X_4, X_5\}\}$, the following group orderings have the respective distances:
\begin{enumerate}
    \item $\{\{X_5\}, \{X_4\}, \{X_3\}, \{X_2, X_1\}\}$: $1.0$;
    \item $\{\{X_5\}, \{X_4, X_3, X_2, X_1\}\}$: $1.0$;
    \item $\{\{X_1, X_2, X_3, X_4, X_5\}\}$: $0.9$;
    \item $\{\{X_2, X_4\}, \{X_1, X_3, X_5\}\}$: $0.8$;
    \item $\{\{X_1, X_4\}, \{X_2, X_3, X_5\}\}$: $0.7$;
    \item $\{\{X_1, X_2, X_4\}, \{X_3, X_5\}\}$: $0.6$;
    \item $\{\{X_2\}, \{X_1, X_4\}, \{X_3, X_5\}\}$: $0.5$;
    \item $\{\{X_1, X_3\}, \{X_2, X_4, X_5\}\}$: $0.4$;
    \item $\{\{X_1, X_2\}, \{X_3, X_4, X_5\}\}$: $0.3$;
    \item $\{\{X_1, X_2\}, \{X_3\}, \{X_4, X_5\}\}$: $0.1$;
\end{enumerate}

\section{Detailed Elaboration on Examples}\label{app:examples}
To explain the $\TIN$ condition's definition, characterization, and graphical criteria, in this section we provide some step-by-step derivation of typical examples.

\subsection{Motivation Examples}
We first give an example of $\GIN$ to show our motivation (\cref{sec:motivation}):
\begin{example}[$\GIN$ on chain structure~\cref{fig:chain}]\label{app:eg_gin_on_chain_details}
Let $\Z\coloneqq\{X_1\}, \Y\coloneqq\{X_2,X_3,X_4,X_5\}$. Calculate $\cov(X_1,X_2)=\cov(\Tilde{X}_1 + E_1,\Tilde{X}_2 + E_2) = \cov(\Tilde{X}_1, \Tilde{X}_2) = \cov(\Tilde{X}_1, a\Tilde{X}_1 + \Tilde{E}_2) = a\var(\Tilde{X}_1)$. Similarly, we get the covariance matrix
\begin{equation} \label{Eq:chain_1_2345_cov}
\cov(\Z,\Y) = \begin{bmatrix} 
    a\var(\Tilde{X}_1) & ab\var(\Tilde{X}_1) & abc\var(\Tilde{X}_1) & abcd\var(\Tilde{X}_1)
\end{bmatrix}
\end{equation}
Solve the linear homogeneous equations $\cov(\Z,\Y)\w=\mathbf{0}$, we have $\w = [-bx-bcy-bcdz,x,y,z]^\T$, $x,y,z\in\Real$. Plug $\w$ into $\wTY$:
\begin{equation} \label{Eq:chain_gin_1_2345_wTY}
\begin{aligned}
\omega^\intercal\Y &= (-\cancel{bx}-\cancel{bcy}-\cancel{bcdz})(\cancel{\Tilde{X}_2}+E_2) + x(b\cancel{\Tilde{X}_2}+\Tilde{E}_3+E_3) \\
& \ \ \ \ \ + y(c(b\cancel{\Tilde{X}_2}+\Tilde{E}_3)+\Tilde{E}_4+E_4) + z (d(c(b\cancel{\Tilde{X}_2}+\Tilde{E}_3)+\Tilde{E}_4)+\Tilde{E}_5+E_5) \\
&= \text{does not contain }\{\Tilde{E}_1, \Tilde{E}_2, E_1\}\text{, thus }\omega^\intercal\Y \indep X_1\text{, by the }\darmois\text{~\citep{kagan1973characterization}}.
\end{aligned}\end{equation}
By above we have $\GIN(\{X_1\},\{X_2,\cdots,X_n\})$ satisfied. 
\end{example}

\subsection{Examples of Ordered Group Decomposition}
In~\cref{def:ordered_group_decomposition} we define the ordered group decomposition of a graph. Actually there is slight difference between~\cref{def:ordered_group_decomposition} and Definition 2 in~\citep{zhang2018causal}. According to~\cref{def:ordered_group_decomposition}, when the graph has only one subroots at each step, the two definitions are the same. However, when there are multiple subroots at some step, \cref{def:ordered_group_decomposition} takes all such subroots (and their pure leaf children) as a new group, while in~\citep{zhang2018causal}, each group has one and only one non-leaf node, and thus only one subroot (and its pure leaf children) is taken (and removed from graph) as a new group, which yields multiple ordered group decompositions. Here we only return one ordered group decomposition mainly for simplicity. To obtain the multiple ordered group decompositions defined in~\citep{zhang2018causal} from our result, we could do the following: for each pair of $X_i, X_j$ in a new group with multiple variables, test $\TIN$ with $\Z$ being $X_i$ and $\Y$ being $X_j$ and all variables in the previous groups, to see whether $X_i, X_j$ are in a same \textit{cluster}. Below we give an example to show this slight difference between two definitions:

\begin{example}[Ordered group decomposition with multiple subroots]\label{app:multiple_subroots_ordered_group_decomp}
Consider the graph $D\leftarrow A\rightarrow C\leftarrow B$. The ordered group decomposition we defined in~\cref{def:ordered_group_decomposition} is $\{A,B,D\}\rightarrow\{C\}$, while Definition 2 in~\citep{zhang2018causal} will give two ordered group decompositions: $\{A,D\}\rightarrow\{B,C\}$, and $\{B\}\rightarrow\{A,C,D\}$.
\end{example}

\subsection{A Concrete Example of Using $\TIN$ on Fully Connected DAG}\label{app:example_fullyC}
Consider a fully connected DAG (\cref{fig:fully_connected}) with 4 variables $\{X_1,X_2,X_3,X_4\}^\intercal$ (suppose we have access directly to the measurement-error-free variables and directly test $\TIN$ over them), the edges parameters are:
\begin{equation}
    \A= \begin{bmatrix} 
    0 & 0 & 0 & 0\\
    a & 0 & 0 & 0\\
    b & d & 0 & 0\\
    c & e & f & 0
\end{bmatrix};\B= \begin{bmatrix} 
    1 & 0 & 0 & 0\\
    a & 1 & 0 & 0\\
    ad+b & d & 1 & 0\\
    a(df+e)+bf+c & df+e & f & 1
\end{bmatrix}\end{equation}

\textbf{1)} Let $\Z\coloneqq\{X_1\}, \Y\coloneqq\{X_2, X_3, X_4\}$, we have the independent subspace:
\begin{equation}
\Omega_{\Z;\Y} = \operatorname{null}(\B_{\Y, \Anc(Z)}^\intercal) = \operatorname{null}\left( \begin{bmatrix} 
    a & ad+b & a(df+e)+bf+c
\end{bmatrix}\right),\text{ dimension=2}.
\end{equation}
Two basis of $\Omega_{\Z;\Y}$ are:\begin{equation}
\begin{aligned}
\omega_1 &= \begin{bmatrix} -\frac{b+ad}{a} & 1 & 0\end{bmatrix}^\intercal, \\ \omega_2 &= \begin{bmatrix} -\frac{c+ae+bf+adf}{a} & 0 & 1\end{bmatrix}^\intercal
\end{aligned}
\end{equation}
For any $\omega = k_1\omega_1 + k_2\omega_2, k_1,k_2\in \mathbb{R}$, $\omega^\intercal \Y$ does not contain noise $E_1$, so $\omega^\intercal \Y \indep \Z$. We have $\operatorname{TIN}(\Z,\Y) = |\Y| - \operatorname{dim}({\Omega_{\Z;\Y}}) = 3-2=1$. Graphically, the minimum vertex cut from $\Anc(\{X_1\}) = \{X_1\}$ to $\{X_2, X_3, X_4\}$ is $\{X_1\}$, with size 1. And, $|\Anc(X_1)|=1$.\\

\textbf{2)} Let $\Z\coloneqq\{X_2\}, \Y\coloneqq\{X_1, X_3, X_4\}$, we have the independence subspace:
\begin{equation}
\Omega_{\Z;\Y} = \operatorname{null}(\B_{\Y, \Anc(Z)}^\intercal) = \operatorname{null}\left( \begin{bmatrix} 
    1 & ad+b & a(df+e)+bf+c \\
    0 & d & df+e
\end{bmatrix}\right),\text{ dimension=1}.
\end{equation}
One basis of $\Omega_{\Z;\Y}$ is:\begin{equation}\begin{aligned}
\omega_1 = \begin{bmatrix} \frac{cd-be}{d} & \frac{df+e}{d} & -1\end{bmatrix}^\intercal
\end{aligned}\end{equation}
For any $\omega = k_1\omega_1, k_1\in \mathbb{R}$, $\omega^\intercal \Y$ does not contain noise $E_1, E_2$, so $\omega^\intercal \Y \indep \Z$. We have $\operatorname{TIN}(\Z,\Y) = |\Y| - \operatorname{dim}({\Omega_{\Z;\Y}}) = 3-1=2$. Graphically, the minimum vertex cut from $\Anc(\{X_2\}) = \{X_1, X_2\}$ to $\{X_1, X_3, X_4\}$ is $\{X_1, X_2\}$, with size 2. And, $|\Anc(X_2)|=2$.\\

\textbf{3)} Let $\Z\coloneqq\{X_3\}, \Y\coloneqq\{X_1, X_2, X_4\}$, we have the independence subspace:
\begin{equation}
\Omega_{\Z;\Y} = \operatorname{null}(\B_{\Y, \Anc(Z)}^\intercal) = \operatorname{null}\left( \begin{bmatrix} 
    1 & a & a(df+e)+bf+c \\
    0 & 1 & df+e \\
    0 & 0 & f
\end{bmatrix}\right),\text{ dimension=0}.
\end{equation}
$B_{\Y, \Anc(Z)}^\intercal$ is full column rank, so that there exists no non-zero $\omega$ s.t. $\omega^\intercal\Y\indep\Z$, i.e., $\Omega_{\Z;\Y}$ contains only origin point $\mathbb{R}_\mathbf{0}$, with dimension 0.  We have $\operatorname{TIN}(\Z,\Y) = |\Y| - \operatorname{dim}({\Omega_{\Z;\Y}}) = 3-0=3$. Graphically, the minimum vertex cut from $\Anc(\{X_3\}) = \{X_1, X_2, X_3\}$ to $\{X_1, X_2, X_4\}$ is $\{X_1, X_2, X_3\}$ or $\{X_1, X_2, X_4\}$, with size 3. And, $|\Anc(X_3)|=3$.\\

\textbf{4)} Let $\Z\coloneqq\{X_4\}, \Y\coloneqq\{X_1, X_2, X_3\}$, we have the independence subspace:
\begin{equation}
\Omega_{\Z;\Y} = \operatorname{null}(\B_{\Y, \Anc(Z)}^\intercal) = \operatorname{null}\left( \begin{bmatrix} 
    1 & a & ad+b \\
    0 & 1 & d \\
    0 & 0 & 1 \\
    0 & 0 & 0
\end{bmatrix}\right),\text{ dimension=0}.
\end{equation}
$B_{\Y, \Anc(Z)}^\intercal$ is full column rank, so that there exists no non-zero $\omega$ s.t. $\omega^\intercal\Y\indep\Z$, i.e., $\Omega_{\Z;\Y}$ contains only origin point $\mathbb{R}_\mathbf{0}$, with dimension 0.  We have $\operatorname{TIN}(\Z,\Y) = |\Y| - \operatorname{dim}({\Omega_{\Z;\Y}}) = 3-0=3$. Graphically, the minimum vertex cut from $\Anc(\{X_4\}) = \{X_1, X_2, X_3, X_4\}$ to $\{X_1, X_2, X_3\}$ is $\{X_1, X_2, X_3\}$, with size 3. And, $|\Anc(X_4)|=4$. Since $X_4$ is a leaf node, $\operatorname{TIN}(\Z,\Y)=4-1=3$.

By above, we obtain the ordered group decomposition $\{\{X_1\}, \{X_2\}, \{X_3, X_4\}\}$.

\subsection{A Concrete Example of Using $\TIN$ on the Chain Structure}
To align with~\cref{app:example_fullyC}, here we consider a chain structure with 4 variables $\{X_1,X_2,X_3,X_4\}^\intercal$. The edges parameters are:
\begin{equation}\A= \begin{bmatrix} 
    0 & 0 & 0 & 0\\
    a & 0 & 0 & 0\\
    0 & b & 0 & 0\\
    0 & 0 & c & 0
\end{bmatrix};\B= \begin{bmatrix} 
    1 & 0 & 0 & 0\\
    a & 1 & 0 & 0\\
    ab & b & 1 & 0\\
    abc & bc & c & 1
\end{bmatrix}\end{equation}

\textbf{1)} Let $\Z\coloneqq\{X_1\}, \Y\coloneqq\{X_2, X_3, X_4\}$, we have the independent subspace:
\begin{equation}
\Omega_{\Z;\Y} = \operatorname{null}(\B_{\Y, \Anc(Z)}^\intercal) = \operatorname{null}\left( \begin{bmatrix} 
    a & ab & abc
\end{bmatrix}\right),\text{ dimension=2}.
\end{equation}
Two basis of $\Omega_{\Z;\Y}$ are:\begin{equation}
\begin{aligned}[t]
\omega_1 &= \begin{bmatrix} b & -1 & 0\end{bmatrix}^\intercal, \\ \omega_2 &= \begin{bmatrix} bc & 0 & 1\end{bmatrix}^\intercal
\end{aligned}
\end{equation}
We have $\operatorname{TIN}(\Z,\Y) = |\Y| - \operatorname{dim}({\Omega_{\Z;\Y}}) = 3-2=1$. Analysis is similar to that of~\cref{app:example_fullyC}.\\

\textbf{2)} Let $\Z\coloneqq\{X_2\}, \Y\coloneqq\{X_1, X_3, X_4\}$, we have the independence subspace:
\begin{equation}
\Omega_{\Z;\Y} = \operatorname{null}(\B_{\Y, \Anc(Z)}^\intercal) = \operatorname{null}\left( \begin{bmatrix} 
    1 & ab & abc \\
    0 & b & bc
\end{bmatrix}\right),\text{ dimension=1}.
\end{equation}
One basis of $\Omega_{\Z;\Y}$ is:
\begin{equation}
\begin{aligned}
\omega_1 &= \begin{bmatrix} 0 & c & -1\end{bmatrix}^\intercal
\end{aligned}
\end{equation}
We have $\operatorname{TIN}(\Z,\Y) = |\Y| - \operatorname{dim}({\Omega_{\Z;\Y}}) = 3-1=2$. Analysis is similar to that of~\cref{app:example_fullyC}.\\

\textbf{3)} Let $\Z\coloneqq\{X_3\}, \Y\coloneqq\{X_1, X_2, X_4\}$, we have the independence subspace:
\begin{equation}
\Omega_{\Z;\Y} = \operatorname{null}(\B_{\Y, \Anc(Z)}^\intercal) = \operatorname{null}\left( \begin{bmatrix} 
    1 & a & abc \\
    0 & 1 & bc \\
    0 & 0 & c
\end{bmatrix}\right),\text{ dimension=0}.
\end{equation}
We have $\operatorname{TIN}(\Z,\Y) = |\Y| - \operatorname{dim}({\Omega_{\Z;\Y}}) = 3-0=3$. Analysis is similar to that of~\cref{app:example_fullyC}.\\

\textbf{4)} Let $\Z\coloneqq\{X_4\}, \Y\coloneqq\{X_1, X_2, X_3\}$, we have the independence subspace:
\begin{equation}
\Omega_{\Z;\Y} = \operatorname{null}(\B_{\Y, \Anc(Z)}^\intercal) = \operatorname{null}\left( \begin{bmatrix} 
    1 & a & ab \\
    0 & 1 & b \\
    0 & 0 & 1 \\
    0 & 0 & 0
\end{bmatrix}\right),\text{ dimension=0}.
\end{equation}
We have $\operatorname{TIN}(\Z,\Y) = |\Y| - \operatorname{dim}({\Omega_{\Z;\Y}}) = 3-0=3$. Analysis is similar to that of~\cref{app:example_fullyC}.\\

By above, we get the group ordering $\{\{X_1\}, \{X_2\}, \{X_3, X_4\}\}$. Recall the chain structure with triangular head example, we could distinguish it from the chain structure with ordered group decomposition $\{\{X_1\}, \{X_2, X_3, X_4\}\}$.

\subsection{\newcontent{Experiments on another Real-world Dataset: Teacher Burnout}}\label{app:experiments_on_teacher_burnout}
Except for Sach's dataset discussed in~\cref{sec:real_world_data}, we also conduct experiments on another real-world dataset, Teacher Burnout~\citep{byrne2013structural}. It is from a sociology survey conducted by Barbara Byrne to investigate the influence on the three facets (emotional exhaustion, depersonalization, and personal accomplishment) of full-time elementary teachers' burnout from factors including: organizational (role ambiguity, role conflict, work overload, classroom climate, decision making, superior support, peer support) and personality (self-esteem, external locus of control) variables. Please see chapter six of~\citep{byrne2013structural} for more details about the dataset (Page 161), and the structure (Page 191).

While in the raw dataset, each (latent/target) variable has more than one measurements/indicators, in this experiment we pick only one measurement for each to demonstrate the measurement error situation. Specifically, we pick ten variables (according to the ten latent variables in Figure 6.10 of~\citep{byrne2013structural}): RA$_1$, RC$_1$, CC$_1$, DM$_1$, SS$_1$, SE$_1$, ELC$_1$, EE$_1$, DP$_1$, and PA$_1$. Though for a thorough study of the dataset, one could try other combinations of measurements, e.g., RA$_2$, RC$_3$, ..., in this experiment we only study one combination as above for illustration.

According to Figure 6.10 of~\citep{byrne2013structural}, the ordered group decomposition of the ground-truth underlying causal graph is $\{$RA, RC, CC, DM, SS$\}\veryshortarrow$ $\{$SE, ELC$\}\veryshortarrow$ $\{$EE$\}\veryshortarrow$ $\{$DP, PA$\}$. Result given by $\operatorname{TIN-subsets}$ is $\{$DM, SE$\}\veryshortarrow$ $\{$CC, SS$\}\veryshortarrow$ $\{$RA, ELC$\}\veryshortarrow$ $\{$RC, EE, DP$\}\veryshortarrow$ $\{$PA$\}$ (with the one-over-others $\TIN$ being $5,6,7,8,9$ respectively). This is similar to Byrne's conclusion (the true ordering) according to the domain knowledge. For example, 1) the three facets of burnout (emotional exhaustion, depersonalization, and personal accomplishment) are caused by other factors and are at the end of the ordered groups, 2) decision making, classroom climates and superior support are root causes (in the first two groups), and 3) self-esteem and role conflict influences external locus of control. Interestingly, some of the ordering inconsistent with the ground-truth might also be reasonable to some extent. For example, 1) self-esteem is among the first group (though should be in the second), maybe because it is ``root-like'': it is only caused by two root causes and causes another four variables, 2) decision making and superior support are in the first and second groups respectively (though should both be in the first, as two root causes for self-esteem), maybe because there exists difference in their impact on others, and 3) role ambiguity is in the third group (though should be in the first), maybe because that though it is a root, it has only one child, personal accomplishment, which is a leaf node in the graph; the same may applies to role conflict: though being a root, it is even among the second to last group, which is also echoed by other methods.

Here is an overview of the distance scores and ordered groups returned by all methods:
\begin{enumerate}[noitemsep,topsep=-3pt]
\item $\operatorname{TIN-2steps}$: 0.49, $\{$CC, SE, ELC$\}\veryshortarrow$ $\{$RA, DM, SS$\}\veryshortarrow$ $\{$EE, PA$\}\veryshortarrow$ $\{$RC, DP$\}$.
\item $\operatorname{TIN-subsets}$: 0.47, $\{$DM, SE$\}\veryshortarrow$ $\{$CC, SS$\}\veryshortarrow$ $\{$RA, ELC$\}\veryshortarrow$ $\{$RC, EE, DP$\}\veryshortarrow$ $\{$PA$\}$.
\item $\operatorname{TIN-ISA}$: 0.56, $\{$RA, DM, SE$\}\veryshortarrow$ $\{$SS, ELC$\}\veryshortarrow$ $\{$RC, CC, EE, DP, PA$\}$.
\item $\operatorname{TIN-rank}$: 0.60, $\{$CC, SE$\}\veryshortarrow$ $\{$RA, RC, DM, ELC, DP$\}\veryshortarrow$ $\{$SS, EE, PA$\}$.
\item ICA-LiNGAM: 0.56, $\{$CC$\}\veryshortarrow$ $\{$SE$\}\veryshortarrow$ $\{$ELC$\}\veryshortarrow$ $\{$RA$\}\veryshortarrow$ $\{$SS$\}\veryshortarrow$ $\{$EE$\}\veryshortarrow$ $\{$RC, PA$\}\veryshortarrow$ $\{$DM, DP$\}$.
\item Direct-LiNGAM: 0.73, $\{$DP$\}\veryshortarrow$ $\{$SE$\}\veryshortarrow$ $\{$SS$\}\veryshortarrow$ $\{$RA$\}\veryshortarrow$ $\{$RC, PA$\}\veryshortarrow$ $\{$CC$\}\veryshortarrow$ $\{$EE$\}\veryshortarrow$ $\{$DM, ELC$\}$.
\item CAMME-OICA: 0.78, $\{$CC, SS, EE, DP, PA$\}\veryshortarrow$ $\{$RA, RC, DM, SE, ELC$\}$.
\item PC: 0.73, $\{$RA, RC, CC, DM, SS, SE, ELC, EE, DP, PA$\}$.
\item GES: 0.73, $\{$CC, DM, EE, DP$\}\veryshortarrow$ $\{$RA$\}\veryshortarrow$ $\{$SS, SE, PA$\}\veryshortarrow$ $\{$RC, ELC$\}$.
\item NOTEARS: 0.82, $\{$CC, SE, ELC, PA$\}\veryshortarrow$ $\{$RA, EE, DP$\}\veryshortarrow$ $\{$RC$\}\veryshortarrow$ $\{$DM, SS$\}$.
\item SCORE: 0.76, $\{$SE$\}\veryshortarrow$ $\{$SS$\}\veryshortarrow$ $\{$RA, CC, DM, ELC, EE, DP, PA$\}\veryshortarrow$ $\{$RC$\}$.
\end{enumerate}

\end{document}